\documentclass[10pt,journal]{IEEEtran}
\usepackage{amsmath,amssymb,amsfonts}
\usepackage{amsthm}
\usepackage{algorithmic}
\usepackage{array}
\usepackage[caption=false,font=footnotesize]{subfig}
\usepackage{textcomp}
\usepackage{url}
\usepackage{verbatim}
\usepackage{graphicx}
\usepackage{cite}
\usepackage{booktabs}
\usepackage{nicefrac}
\usepackage{microtype}
\usepackage{xcolor}
\usepackage{mathtools}
\usepackage{multirow}
\usepackage{colortbl}
\usepackage[hidelinks]{hyperref}
\usepackage[ruled,vlined,linesnumbered]{algorithm2e}
\graphicspath{{figures/}}
\newtheorem*{theoremfinal}{Theorem}
\newtheorem*{corollaryfinal}{Corollary}
\newtheorem*{remarkfinal}{Remark}
\newtheorem*{propositionfinal}{Proposition}

\title{Preserving Geometric Integrity in Graph Prompting\\ via Measure-Constrained Optimal Transport}

\author{Xiangyu~Wang,
        Shuo~Wang,
        Ruiyi~Fang,
        and~Zhao~Kang%
\thanks{X. Wang and S. Wang contributed equally to this work.}%
\thanks{X. Wang and Z. Kang are with the School of Computer Science and Engineering, University of Electronic Science and Technology of China, Chengdu 611731, China (e-mail: xy.texwang.cs@gmail.com; zkang@uestc.edu.cn).}%
\thanks{S. Wang is with the School of Computer Science and Engineering, University of Electronic Science and Technology of China, Chengdu 611731, China, and also with the Institute for Interdisciplinary Information Sciences, Tsinghua University, Beijing 100084, China (e-mail: runner21st@gmail.com).}%
\thanks{R. Fang is with the Department of Computer Science, Western University, Canada (e-mail: rfang32@uwo.ca).}%
\thanks{Z. Kang is the corresponding author.}}

\begin{document}

\pagestyle{plain}

\maketitle

\begin{abstract}

Graph prompt learning enables parameter-efficient adaptation of frozen
Graph Neural Networks to downstream tasks through lightweight prompt
parameters. As routing becomes increasingly node-adaptive, however,
independently optimized local decisions can collectively concentrate assignment
mass on a small subset of a finite shared prompt bank, even when individual
node--prompt matches remain locally meaningful. We propose
\textbf{MINT (Measure-INtegrity Transport)}, an entropically regularized
optimal transport framework that formulates node-to-prompt adaptation as
a globally coupled allocation problem. The transport cost favors local
geometric compatibility, while a prescribed prompt-side marginal explicitly
controls graph-wide prompt utilization. We further derive an exact variance
decomposition that separates prompt-side geometric variance into retained
prompt-update variation and within-node barycentric dispersion, together with
a conditional stability bound for the frozen-encoder forward map. Across
standard citation networks and additional heterophilic graphs, MINT remains
competitive in few-shot adaptation. Controlled and end-to-end experiments
further distinguish the roles of routing and topology: fixed-marginal routing
controls graph-wide prompt utilization and has measurable end-to-end effects
on citation networks, while topology augmentation provides a complementary,
graph-dependent mechanism for addressing structural mismatch. Code is available
at \url{https://github.com/Ga1axy0051/MINT}.

\end{abstract}


\section{Introduction}

Graph Neural Networks (GNNs) learn representations from
node attributes and graph connectivity and have become a standard foundation
for learning on graph-structured data
~\cite{kipf2017semi,hamilton2017inductive,velickovic2018graph}.
Graph pre-training and prompt learning further enable frozen GNN backbones to
adapt to downstream tasks through a small number of task-specific parameters
~\cite{sun2022gppt,liu2023graphprompt,fang2023gpf}. Existing studies have
progressively enriched graph prompting through more expressive prompt
parameterizations, instance-conditioned prompts, and topology-aware adaptation.
As routing becomes increasingly node- and instance-adaptive, however, a
distinct challenge emerges: local node-to-prompt decisions collectively
determine how a finite shared prompt bank is utilized across the graph.

\begin{figure*}[!t]
    \centering
    \subfloat[\textbf{Evolution of $G_p$.}\label{fig:motivation_gini}]{%
        \includegraphics[width=0.4\textwidth]{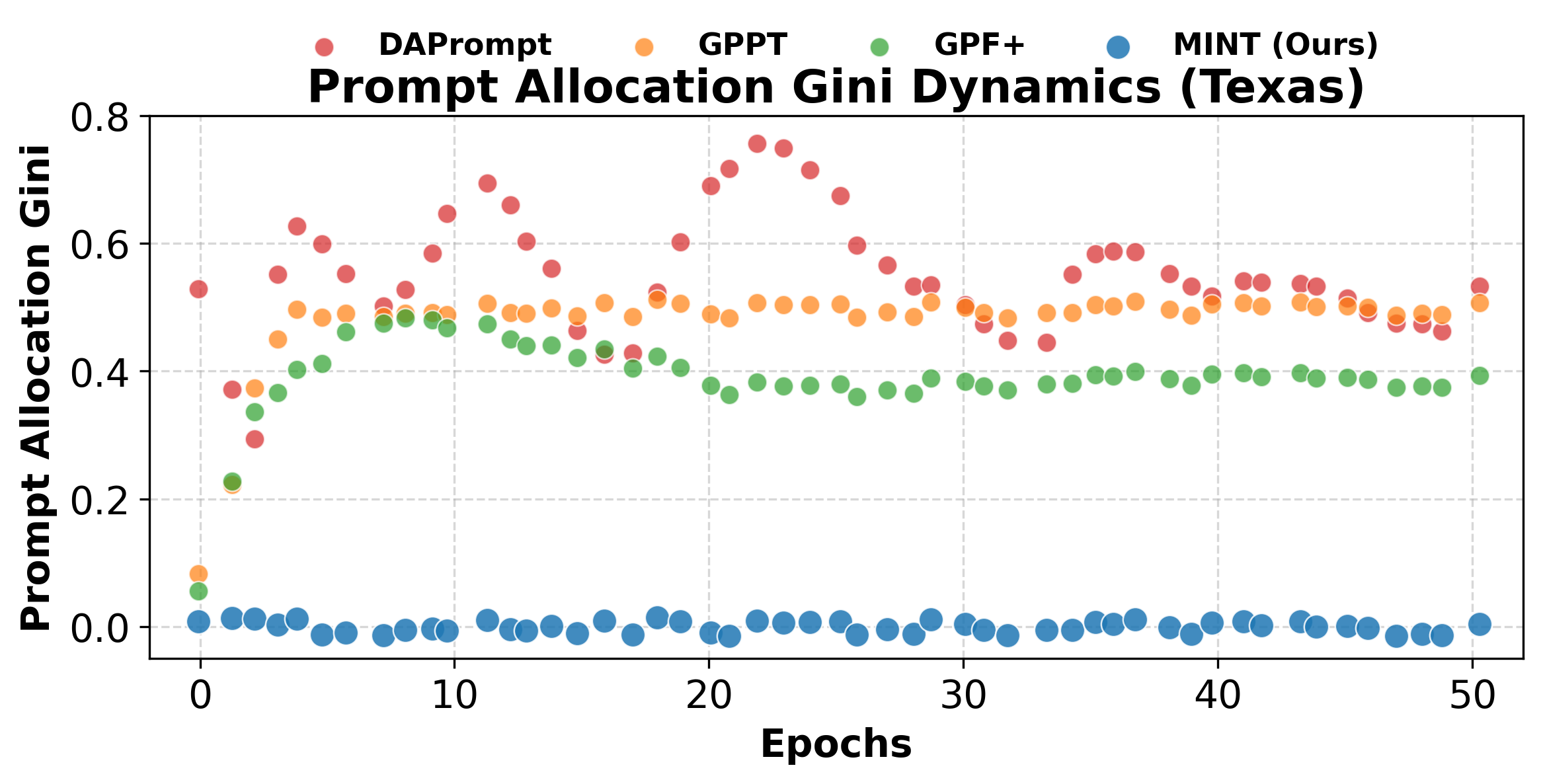}%
    }
    \hfil 
    \subfloat[\textbf{Evolution of Dirichlet Energy.}\label{fig:motivation_energy}]{%
        \includegraphics[width=0.4\textwidth]{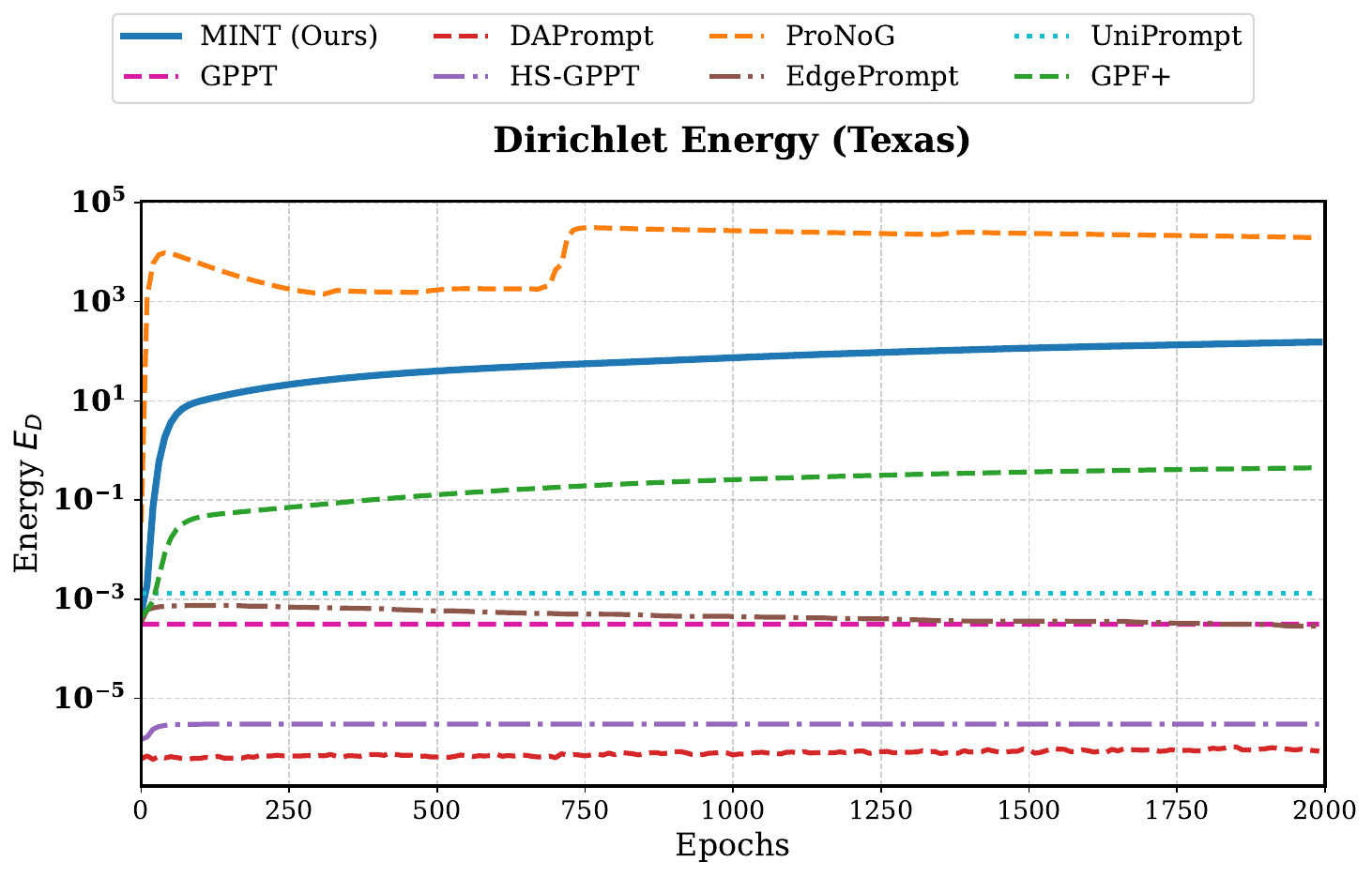}%
    }
    \caption{\textbf{Allocation and representation diagnostics for graph prompting.} (a) The Prompt Allocation Gini Coefficient ($G_p$) measures allocation concentration across the prompt bank. (b) Feature Dirichlet Energy ($E_D$), shown on a logarithmic axis, measures graph-relative representation geometry. Prompt-update variation serves as the intermediate quantity analyzed in Section~\ref{subsec:theory_laws}, while Feature Dirichlet Energy is retained as a downstream graph-relative representation diagnostic.}
    \label{fig:motivation_dynamics}
\end{figure*}

A useful geometric analogy is \textbf{hubness} in high-dimensional
representation spaces
~\cite{radovanovic2010hubness,feldbauer2019comprehensive},
where a small subset of points occurs disproportionately often in the
nearest-neighbor lists of other points. Analogously, node-to-prompt routing can
assign a disproportionate fraction of the aggregate mass to a small subset of
prompts even when the corresponding local matches are individually meaningful.
We quantify this graph-wide allocation concentration using the Prompt
Allocation Gini Coefficient ($G_p$), where $G_p=0$ denotes uniform aggregate
utilization. As shown by the optimization trajectories in
Fig.~\ref{fig:motivation_gini}, substantial allocation concentration can arise
under representative local routing strategies.

This graph-wide allocation profile matters because the routing matrix directly
determines the prompt-induced feature update supplied to the frozen model.
When aggregate allocation becomes highly concentrated, the variation retained
in this update can be reduced. The downstream effect is not determined by
allocation alone: final representations and predictions also depend on the
original node features, prompt geometry, graph propagation, and optimization.
Accordingly, the central object is not a generic notion of representation
collapse, but the barycentric prompt-induced update that enters the actual
frozen-GNN forward path. In Section~IV-B, we formalize this relation through
an exact allocation identity, an exact variance decomposition, and a
conditional forward-stability result for the frozen encoder. This issue is also
distinct from the optimization difficulty introduced by discrete routing: the
present problem concerns graph-wide utilization under otherwise valid local
assignments.

These observations motivate considering graph prompting at both local and
graph-wide scales. Node--prompt assignments should remain locally compatible,
while aggregate use of the shared prompt bank should be explicitly controlled.
Achieving both objectives requires moving beyond independently normalized
local routing toward a globally coupled allocation formulation.

To this end, we propose \textbf{Measure-INtegrity Transport (MINT)}, which
formulates node-to-prompt adaptation as a globally coupled, fixed-marginal
allocation problem through entropically regularized optimal transport (OT)
~\cite{peyre2019computational,cuturi2013sinkhorn}. MINT optimizes a transport
plan between nodes and prompts: the transport cost favors local geometric
compatibility, whereas a prescribed prompt-side marginal specifies graph-wide
prompt utilization. Entropic regularization yields a differentiable routing
mechanism implemented with finite Sinkhorn iterations. Separately, MINT
includes a topology-augmentation branch for graph-space structural mismatch.
The routing and topology mechanisms therefore act on different computational
objects: fixed-marginal OT controls prompt-space allocation, whereas topology
augmentation modifies node-to-node propagation.

We evaluate MINT across homophilic citation networks and additional
heterophilic graphs using multiple pretrained GNN backbones. Beyond aggregate
few-shot accuracy, we analyze the mechanism at several levels: graph-wide
prompt utilization, prompt-induced feature variation, frozen-forward behavior,
and the separate contributions of routing and topology augmentation. The
results show that MINT remains competitive across the evaluated few-shot
settings, that fixed-marginal routing has measurable end-to-end effects on
standard citation networks, and that topology augmentation plays a
complementary but graph-dependent role.

Our main contributions are summarized as follows:

\begin{itemize}

\item We identify graph-wide utilization of a finite shared prompt bank as a
distinct aspect of graph prompt learning: locally meaningful node-to-prompt
routing can still produce substantial aggregate concentration.

\item We introduce MINT, which couples node-to-prompt allocation through a
prescribed prompt-side marginal while favoring local geometric compatibility
through entropically regularized OT. A separate topology-augmentation branch
addresses graph-space structural mismatch.

\item We derive an exact variance decomposition linking the graph-wide
allocation profile, prompt geometry, and barycentric mixing to prompt-update
variation, together with a conditional bound on the frozen-encoder deviation
under severe single-prompt concentration. We further evaluate few-shot
adaptation and the complementary effects of routing and topology augmentation
across datasets, pretrained backbones, and graph regimes.

\end{itemize}


\section{Related Work}
\label{sec:related_work}

\subsection{Graph Pre-training and Few-Shot Adaptation}
Graph pre-training extracts transferable, task-agnostic information from
unlabeled graph data. Early contrastive self-supervised learning (SSL) methods
such as DGI \cite{velickovic2019deep} and InfoGraph \cite{sun2020infograph}
maximized mutual information, while subsequent methods (e.g.,
GraphCL \cite{you2020graphcl}, GCC \cite{qiu2020gcc},
MVGRL \cite{hassani2020mvgrl}, GCA \cite{zhu2021gca}) used augmented views.
BGRL \cite{thakoor2022bgrl} and CCA-SSG \cite{zhang2021cca}
removed negative sampling. Because heuristic augmentations such as edge
dropping may alter important topology, generative pre-training---including
Masked Graph Autoencoders (GAEs) such as GraphMAE \cite{hou2022graphmae},
GraphMAE2 \cite{hou2023graphmae2}, MaskGAE \cite{li2023maskgae},
and HGMAE \cite{tian2023hgmae}---instead learns priors by reconstructing
masked graph information.

Adapting these priors to data-scarce downstream tasks has also been studied
through graph meta-learning methods, including Meta-GNN
\cite{zhou2019metagnn}, G-Meta \cite{huang2020gmeta}, and
GPN \cite{ding2020gpn}. Parameter-efficient graph prompting offers another
route by adapting frozen pretrained models to downstream tasks with lightweight
parameters.
Unified reviews organize graph SSL objectives and clarify how pretext design, augmentation, and downstream transfer interact across contrastive, predictive, and generative paradigms~\cite{xie2023selfsupervised}.
Recent graph foundation models pursue shared interfaces across graphs and tasks through text-space evaluation, joint graph--language modeling, text-free multi-domain pre-training, and transferable text-attributed representations~\cite{chen2024textspace,kong2025gofa,yu2025samgpt,zhu2025graphclip}.
Unsupervised, source-free, and invariant-learning formulations further address shifts in graph attributes, structure, and environments under label scarcity~\cite{yin2025dream,luo2024gala,wu2022handling}.
Multi-domain transfer has consequently been studied through topology alignment, attribute-driven adaptation, and structure-centric geometric bases for few-shot transfer across heterogeneous feature spaces~\cite{wang2025mdgfm,fang2025attribute,he2026scgfm}.

Evidence from heterophily-focused designs, datasets, and metrics shows that the utility of graph propagation varies with the graph regime rather than following from a single homophily score~\cite{zhu2020beyond,lim2021large,luan2022revisiting}.
Accordingly, topology effects should be interpreted as graph-regime dependent; robust structure learning under heterophily and curvature-stratified evaluation offer complementary structural and geometric perspectives without identifying homophily with curvature~\cite{shen2025heterophily,wang2026postgcn}.

\subsection{Graph Prompt Learning}

Graph prompting bridges the pretext--downstream gap with lightweight parameters that steer frozen GNNs~\cite{sun2022gppt, zi2024prog}. GPF adds a shared learnable prompt feature to node inputs, whereas GPF+ generates node-dependent prompts from learnable basis prompts~\cite{fang2023gpf}. GraphPrompt uses a task-specific learnable prompt vector in its prompt-assisted readout~\cite{liu2023graphprompt}, while All-in-One constructs graph prompts with learnable prompt tokens and prompt structure~\cite{sun2023allinone}. Other multi-prompt and instance-dependent designs increase local
flexibility~\cite{fu2025edgeprompt,yu2025pronog}; additional frameworks
enrich prompt conditioning or task generality~\cite{huang2023prodigy,
liu2024ofa,yu2024hgprompt}. Topology- and frequency-aware approaches further
address structural mismatch~\cite{jiang2026daprompt,luo2025hsgppt}.

Sparse graph prompting, adaptive prompt pools for continual test-time training, and graph prompt clustering extend parameter-efficient adaptation along complementary axes~\cite{jiang2026gsp,cai2026continual,chen2025gpc}.
Graph foundation models also study cross-domain transfer through multi-domain pre-training, text-attributed transfer, and topology alignment~\cite{yu2025samgpt,zhu2025graphclip,wang2025mdgfm}.
Topology-routed mixed-curvature experts provide a recent dynamic-graph prompting perspective~\cite{wang2026curvprompt}; this approach does not impose MINT's fixed prompt-side marginal.

Graph structure learning separately studies topology refinement and sparsification for robustness, including regularized reconstruction, learned edge removal, and statistically tested similarities~\cite{jin2020prognn,zheng2020neuralsparse,wang2022multiple}.
Feature-derived structural views, heterophily-aware refinement, and information-aware multiplex fusion further adapt graph topology~\cite{fang2022spgrl,shen2025heterophily,shen2024infomgf}, which remains distinct from MINT's prompt-space allocation mechanism.

These developments primarily improve prompt content, specificity, or structural adaptation. Graph-wide utilization has received substantially less attention. In particular, row-wise attention determines a valid mixture for every node but leaves aggregate assignment mass uncontrolled. MINT instead couples node-to-prompt adaptation globally through a prescribed prompt-side marginal that controls the allocation profile.

\subsection{Hubness Mitigation in Representation Learning}
Classical hubness occurs when a small subset of points appears disproportionately often in the $k$-nearest-neighbor lists of other points in high-dimensional spaces~\cite{radovanovic2010hubness}. In broader alignment tasks, heuristic techniques such as Mutual Proximity~\cite{schnitzer2012local} and Cross-Domain Similarity Local Scaling (CSLS)~\cite{conneau2017word} mitigate hubness by adjusting similarity scores using local neighborhood densities.

These post-hoc neighborhood corrections address a different setting from end-to-end prompt allocation. They adjust local scores, whereas MINT uses the prescribed prompt-side marginal of a transport coupling to encode graph-wide prompt utilization directly while preserving differentiability through entropic regularization.

\subsection{Optimal Transport in Graph Learning}
Optimal Transport (OT) provides a mathematical framework for measuring spatial discrepancies between probability distributions \cite{peyre2019computational}. In graph learning, Gromov-Wasserstein (GW) and related distances have been applied to graph matching, barycenter computation, and cross-domain alignment \cite{peyre2017gromov, vayer2019optimal, xu2019gwl, chen2020got}, including the Wasserstein Weisfeiler-Lehman (WWL) kernel \cite{togninalli2019wwl} and GOT \cite{petric2019got}. Recent studies have also formulated generalized global pooling as regularized OT problems across sample indices and feature dimensions \cite{xu2023regularized}. Prior graph-learning applications use OT for graph embeddings, distances to learnable templates, and feature aggregation~\cite{kolouri2021wegl,vincentcuaz2022template,mialon2021trainable}, while unbalanced and partial OT relax mass or matching assumptions~\cite{sejourne2021unbalanced,ratnayaka2024learning}.
GW alignment has also been used to map heterogeneous graphs to shared learnable geometric bases for cross-domain transfer~\cite{he2026scgfm}.

Prior graph-learning uses of OT commonly emphasize graph comparison, alignment, or pooling. MINT uses fixed-marginal entropic transport~\cite{cuturi2013sinkhorn} to couple node-to-prompt assignments within the graph-prompting forward path: the transport cost favors locally compatible node--prompt assignments, while the prescribed prompt-side marginal specifies graph-wide prompt utilization. This allocation role distinguishes MINT's use of OT from the graph comparison, alignment, and pooling settings above.

Balanced prototype assignment predates our formulation: SwAV uses Sinkhorn-based equipartition for online prototype assignment in self-supervised visual representation learning~\cite{caron2020swav}. Sinkhorn scaling, equipartition, and prototype assignment are established ingredients; MINT instead uses a prescribed prompt-side marginal to couple graph-wide prompt allocation with the barycentric prompt-induced feature update.

\section{Preliminaries}
\label{sec:preliminaries}
In this section, we introduce the prompt-space notation and two diagnostics: allocation concentration and graph-relative representation geometry.

\subsection{Prompt Space}
\label{subsec:pre_notations_prompt}
Let $G = (\mathcal{V}, \mathcal{E})$ be an undirected graph with $N = |\mathcal{V}|$ nodes and edge set $\mathcal{E}$, where the topological structure is represented by the adjacency matrix $A \in \mathbb{R}^{N \times N}$. Node features are denoted by $X \in \mathbb{R}^{N \times d}$ for an embedding dimension $d$. We define the augmented adjacency matrix as $\tilde{A} = A + I$, with $\tilde{D}$ as its corresponding degree matrix. The normalized graph Laplacian is given by $L = I - \tilde{D}^{-1/2} \tilde{A} \tilde{D}^{-1/2}$.

For downstream adaptation, we consider a node classification task over a discrete label space $\mathcal{Y}$. The corresponding ground-truth one-hot label matrix is defined as $Y \in \{0, 1\}^{N \times |\mathcal{Y}|}$. 
Let $\mathcal{V}_{tr} \subset \mathcal{V}$ denote the subset of labeled training nodes. 
We represent the prompt space by a bank of $M$ learnable prompt vectors, collected in the prompt matrix $P=[p_1,\ldots,p_M]^\top\in\mathbb{R}^{M\times d}$. Let $F_\theta(U,G')$ denote the representation map of the frozen pretrained encoder on node features $U$, where $G'$ denotes the weighted node-to-node propagation graph supplied to the encoder, and let $h_\phi$ denote the separately trained linear classifier.

To define node-to-prompt transport, let $\mu = \frac{1}{N}\mathbf{1}_N$ and $\nu = \frac{1}{M}\mathbf{1}_M$ denote the uniform source (nodes) and target (prompts) marginal distributions, respectively, where $\mathbf{1}$ is an all-ones vector. The prompt-side marginal $\nu_j$ specifies the aggregate transport mass assigned to prompt vector $p_j$. The feasible set of transport couplings with prescribed marginals is $\mathcal{U}(\mu, \nu) = \{ \Pi \in \mathbb{R}_{+}^{N \times M} \mid \Pi \mathbf{1}_M = \mu, \Pi^\top \mathbf{1}_N = \nu \}$.

\subsection{Prompt Allocation Concentration}
\label{subsec:pre_prompt_gini}
Classical hubness describes points that occur disproportionately often in the $k$-nearest-neighbor lists of other points in high-dimensional spaces. In node--prompt routing, an analogous allocation concentration can place disproportionate aggregate assignment mass on a small subset of prompts. To quantify this disparity, we define the Prompt Allocation Gini Coefficient ($G_p$).
Let $w_{ij}\ge0$ be generic node-to-prompt assignment weights with $\sum_{i,j}w_{ij}>0$, and define the raw mass assigned to prompt $j$ by $a_j:=\sum_{i=1}^Nw_{ij}$.
We use the normalized allocation profile $\rho_j=a_j/\sum_{\ell=1}^M a_\ell$, so $\sum_j\rho_j=1$. Given the non-decreasing ordered sequence $\{a_{(1)}, a_{(2)}, \dots, a_{(M)}\}$, the coefficient is formulated as:
\begin{equation*}
G_{p} = \frac{\sum_{j=1}^M (2j - M - 1) a_{(j)}}{M \sum_{j=1}^M a_{(j)}}
\end{equation*}
For fixed finite $M$, $0\le G_p\le(M-1)/M$: uniform allocation gives $G_p=0$, while complete concentration on one prompt gives $G_p=(M-1)/M$. Because $G_p$ is invariant to a common positive scaling, row-stochastic routing weights $w_{ij}=R_{ij}$ and the corresponding unit-mass MINT coupling induce the same normalized profile. For row-stochastic $R$, $\sum_{i,j}R_{ij}=N$ and hence $\rho_j=\frac{1}{N}\sum_iR_{ij}$.

\subsection{Representation Smoothness and Geometry Diagnostics}
\label{subsec:pre_manifold_metrics}

GNN representations depend on both node features and graph structure. We therefore use the Dirichlet energy of node representations, denoted $E_D$~\cite{zhou2021dirichlet}, as a descriptive measure of graph-relative representation variation:
\begin{equation*}
E_D(X) = \frac{1}{|\mathcal{V}|} \text{tr}(X^\top L X)
\end{equation*}
Low energy indicates that adjacent-node representations vary little; high energy indicates stronger graph-relative variation. Its scale is determined jointly by feature normalization, the graph, and the representation layer. We report $G_p$ for aggregate routing and $E_D$ for graph-relative representation geometry as complementary diagnostics.

\section{Motivation and Theoretical Framework}
\label{sec:motivation_theory}

\subsection{Allocation Concentration and Prompt-Update Geometry}
\label{subsec:empirical_obs}

Figure~\ref{fig:motivation_dynamics}(a) shows that independently normalized routing can yield concentrated graph-wide prompt use, whereas a prescribed prompt-side marginal controls aggregate assignment mass. Figure~\ref{fig:motivation_dynamics}(b) reports graph-relative representation geometry as a downstream diagnostic. The prompt-induced feature update $\Delta X_P=RP$ directly connects routing and prompt geometry; its relationship to representations and predictions further depends on the frozen encoder and optimization. Section~\ref{subsec:mechanism_analysis} evaluates these relationships empirically.

MINT controls prompt allocation through fixed transport marginals, while topology augmentation separately targets graph-space structural mismatch, whose effect varies with graph regime.

\subsection{Prompt-Update Geometry and Forward Stability}
\label{subsec:theory_laws}

Let $R\in\mathbb{R}_{+}^{N\times M}$ be a row-stochastic node-to-prompt routing matrix, $R\mathbf{1}_M=\mathbf{1}_N$, and let $P\in\mathbb{R}^{M\times d}$ be the prompt matrix. The prompt-induced feature update and adapted input used by the frozen model are
\begin{equation}
\Delta X_P=RP, \qquad X_{\mathrm{adapted}}=X+\alpha\Delta X_P,
\label{eq:actual_forward_message}
\end{equation}
followed by the frozen encoder $F_\theta(X_{\mathrm{adapted}},G')$; predictions use $\hat Y=h_\phi(F_\theta(X_{\mathrm{adapted}},G'))$. Thus, $R$ is a node-to-prompt routing matrix, distinct from $G'$, the weighted node-to-node propagation graph supplied to the encoder.

For any such $R$, the normalized graph-wide allocation profile from Section~\ref{subsec:pre_prompt_gini} is $\rho_j=\frac{1}{N}\sum_{i=1}^N R_{ij}$. Under MINT, $R=N\Pi^\star$ and $\Pi^\star\in\mathcal{U}(\mu,\nu)$, giving the exact identity
\begin{equation}
 \rho_j=\frac{1}{N}\sum_iR_{ij}
 =\sum_i\pi^\star_{ij}=\nu_j.
 \label{eq:normalized_capacity_identity}
\end{equation}
For a prompt-index subset $S\subseteq\{1,\ldots,M\}$, define $\rho(S)=\sum_{j\in S}\rho_j$. Under MINT, $\rho(S)=\sum_{j\in S}\nu_j$, which becomes $|S|/M$ for uniform $\nu$. This identity specifies graph-wide prompt utilization, while the realized prompt-induced update additionally depends on prompt geometry and row-wise routing.

Let $\Delta x_{P,i}$ denote row $i$ of $\Delta X_P$, let $\overline{\Delta x}_P=\frac{1}{N}\sum_i\Delta x_{P,i}$, and let
\begin{equation}
\begin{aligned}
 V_{\Delta X_P}
 &=\frac{1}{N}\|H_N\Delta X_P\|_F^2\\
 &=\frac{1}{N}\sum_i
 \|\Delta x_{P,i}-\overline{\Delta x}_P\|_2^2,\\
 H_N&=I-\frac{1}{N}\mathbf{1}_N\mathbf{1}_N^\top.
\end{aligned}
 \label{eq:prompt_update_variance}
\end{equation}
The empirical mechanism study uses the same node-centered variance, computed from the squared Frobenius norm across features and normalized by $N$.

Row stochasticity and the definition of $\rho$ give the mean identity
\begin{equation}
 \overline{\Delta x}_P
 =\frac{1}{N}\sum_i\sum_jR_{ij}p_j
 =\sum_j\rho_jp_j.
 \label{eq:prompt_update_mean_identity}
\end{equation}
We define the prompt-side geometric variance carried by the prompt matrix under the graph-wide allocation profile as
\begin{equation}
 V_P(\rho):=\sum_{j=1}^M\rho_j
 \|p_j-\overline{\Delta x}_P\|_2^2.
 \label{eq:prompt_side_variance}
\end{equation}

\begin{theoremfinal}[Variance Decomposition]
\begin{equation}
\begin{aligned}
 V_P(\rho)=V_{\Delta X_P}
 &+\frac{1}{N}\sum_{i=1}^N\sum_{j=1}^M R_{ij}\,
 \|p_j-\Delta x_{P,i}\|_2^2,
\end{aligned}
 \label{eq:variance_decomposition}
\end{equation}
and hence $V_{\Delta X_P}\le V_P(\rho)$.
\end{theoremfinal}

The decomposition separates allocation-induced weighting from the prompt geometry and node-wise barycentric mixing that determine the realized update. MINT couples these quantities through the prescribed prompt-side marginal and the node--prompt transport cost. Accordingly, the association between allocation Gini and $V_{\Delta X_P}$ may vary as the prompt matrix and row-wise mixtures co-adapt.

\begin{figure*}[!t]
    \centering
    \includegraphics[width=0.95\textwidth]{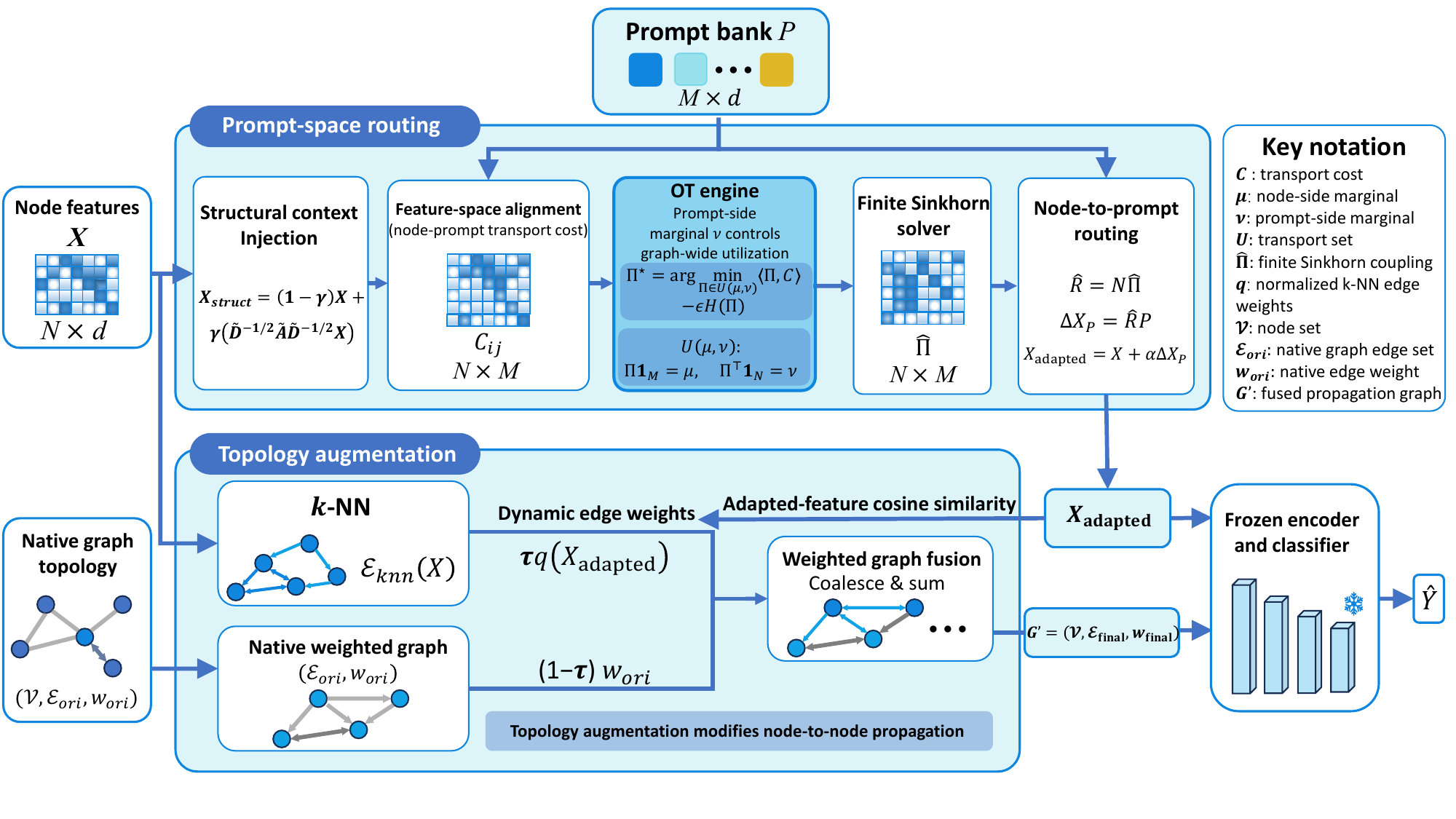}
    \caption{Overview of the MINT framework.
Fixed-marginal optimal transport globally couples node-to-prompt routing for prompt-induced feature adaptation, while topology augmentation dynamically reweights a fixed feature-derived $k$-NN graph and fuses it with the native graph.
The adapted node features and fused propagation graph are jointly supplied to the frozen encoder for downstream prediction.}
    \label{fig:framework}
\end{figure*}

For nonempty $S$, write $P_S=\{p_j:j\in S\}$ and define
$\operatorname{diam}(P_S)=\max_{j,k\in S}\|p_j-p_k\|_2$ and
$\operatorname{diam}(P)=\max_{j,k}\|p_j-p_k\|_2$.

\begin{corollaryfinal}[Subset Concentration]
For every nonempty $S\subseteq\{1,\ldots,M\}$,
\begin{equation}
 V_{\Delta X_P}
 \le \rho(S)\operatorname{diam}(P_S)^2
 +[1-\rho(S)]\operatorname{diam}(P)^2.
 \label{eq:subset_variance_bound}
\end{equation}
\end{corollaryfinal}
For $S=\{j^\star\}$, $\operatorname{diam}(P_S)=0$, so
\begin{equation}
 V_{\Delta X_P}\le(1-\rho_{j^\star})\operatorname{diam}(P)^2.
 \label{eq:single_prompt_variance_bound}
\end{equation}

\begin{remarkfinal}[Additive Cost Potentials]
If $C'_{ij}=C_{ij}+\xi_i+\zeta_j$, then every $\Pi\in\mathcal{U}(\mu,\nu)$ satisfies
$\langle\Pi,C'\rangle=\langle\Pi,C\rangle+\sum_i\mu_i\xi_i+\sum_j\nu_j\zeta_j$.
The added term is independent of the feasible coupling, so additive row and column potentials leave the fixed-marginal optimizer unchanged. This classical OT property is used to interpret graph-wide prompt-attractiveness bias.
\end{remarkfinal}

For any prompt $j^\star$, convexity of the squared norm gives the common-prompt update proximity bound
\begin{equation}
\begin{aligned}
 \frac{1}{N}\|\Delta X_P-\mathbf{1}_Np_{j^\star}^\top\|_F^2
 &\le\sum_j\rho_j\|p_j-p_{j^\star}\|_2^2\\
 &\le(1-\rho_{j^\star})\operatorname{diam}(P)^2.
\end{aligned}
 \label{eq:common_prompt_proximity_bound}
\end{equation}

\begin{propositionfinal}[Forward Stability]
Fix $G'$ and the frozen parameters $\theta$. Suppose the map $U\mapsto F_\theta(U,G')$ is $L_F$-Lipschitz in Frobenius norm on a convex neighborhood containing both $X+\alpha\Delta X_P$ and $X+\alpha\mathbf{1}_Np_{j^\star}^\top$. Then
\begin{equation}
\begin{aligned}
 &\frac{1}{\sqrt N}\big\|
 F_\theta(X+\alpha\Delta X_P,G')\\
 &\quad-F_\theta(X+\alpha\mathbf{1}_Np_{j^\star}^\top,G')\big\|_F\\
 &\le L_F|\alpha|\operatorname{diam}(P)
 \sqrt{1-\rho_{j^\star}}.
\end{aligned}
 \label{eq:forward_stability_bound}
\end{equation}
\end{propositionfinal}
The proposition bounds the forward-map deviation associated with single-prompt concentration when the weighted propagation graph and frozen encoder are fixed; the complete method is evaluated in Section~\ref{sec:experiments}.

\section{Methodology}
\label{sec:methodology}

Motivated by the analysis in Section~\ref{sec:motivation_theory}, \textbf{MINT} couples node-to-prompt decisions through a prescribed prompt-side marginal, while topology augmentation separately addresses graph-space structural mismatch, as illustrated in Fig.~\ref{fig:framework}.

\subsection{Structural Context Injection}
\label{subsec:struct_aggr}
To inject local structural context before prompt routing, we perform structural pre-aggregation on the native graph. Using the graph notation from Section~\ref{subsec:pre_notations_prompt}, we compute the structural feature $X_{struct}$ with a learnable scalar $\gamma$, initialized at $0.5$:
\begin{equation}
    X_{struct} = (1 - \gamma) X + \gamma (\tilde{D}^{-1/2} \tilde{A} \tilde{D}^{-1/2} X).
\label{eq:x_struct}
\end{equation}
The learnable scalar $\gamma$ adjusts the relative contribution of original features and normalized neighborhood aggregation~\cite{wu2019simplifying, gasteiger2019predict}.

\subsection{OT Engine: Fixed-Marginal Node--Prompt Transport}
\label{subsec:ot_engine}
Let $\bar x_i^{\mathrm{struct}}$ and $\bar p_j$ denote row-$L_2$-normalized structural features and prompt vectors. MINT computes
\begin{equation}
 C^{\mathrm{raw}}_{ij}=\|\bar x_i^{\mathrm{struct}}-\bar p_j\|_2^2,
 \qquad
 C=\frac{C^{\mathrm{raw}}}{\max_{a,b}C^{\mathrm{raw}}_{ab}+10^{-8}}.
 \label{eq:implemented_cost}
\end{equation}
For normalized rows, $C^{\mathrm{raw}}_{ij}=2(1-\cos(\bar x_i^{\mathrm{struct}},\bar p_j))$; the matrix-wise maximum rescaling is retained in the implemented cost. MINT solves the entropically regularized OT objective to obtain the formulation-level optimal transport plan $\Pi^\star$:
\begin{equation}
    \Pi^\star = \arg\min_{\Pi \in \mathcal{U}(\mu, \nu)} \sum_{i,j} \pi_{ij} C_{ij} - \epsilon \mathcal{H}(\Pi),
    \label{eq:ot_objective}
\end{equation}
where $\epsilon > 0$ controls the smoothness of the assignment, and $\mathcal{H}(\Pi) = -\sum_{i,j} \pi_{ij} \log \pi_{ij}$ is the Shannon entropy of the transport plan. With $K=\exp(-C/\epsilon)$, the implementation initializes $u=v=\mathbf{1}$ and performs 20 differentiable $u$-then-$v$ Sinkhorn scaling updates with $10^{-8}$ denominator safeguards before forming $\widehat\Pi=\operatorname{diag}(u)K\operatorname{diag}(v)$\cite{cuturi2013sinkhorn}. At the formulation level, aggregate balance follows from the prescribed transport marginals; the finite Sinkhorn solver realizes them to numerical tolerance.

We instantiate the prescribed prompt-side marginal as $\nu=\frac{1}{M}\mathbf{1}_M$, a symmetric equal-capacity prior over the prompt bank. The prior is defined over prompt utilization rather than class frequencies; settings with intrinsically non-uniform prompt demand may instead motivate non-uniform marginals. At the formulation level, defining the node-to-prompt routing matrix $R=N\Pi^\star$ gives $R\mathbf{1}_M=\mathbf{1}_N$ and, by~\eqref{eq:normalized_capacity_identity}, $\rho_j=\frac{1}{N}\sum_iR_{ij}=\nu_j$. Under the usual fixed-marginal conditions, the converged entropic transport solution depends smoothly on nondegenerate costs. This supports gradient-based optimization while retaining the allocation identity.

Feature adaptation uses the prompt-induced feature update. The formulation-level node-to-prompt routing matrix is $R=N\Pi^\star$. The implementation instead defines $\widehat R=N\widehat\Pi$, which numerically realizes the prescribed routing marginals, and computes
\begin{equation}
    \Delta X_P=\widehat R P=N\widehat\Pi P, \qquad X_{\mathrm{adapted}}=X+\alpha\Delta X_P,
    \label{eq:feat_adapt}
\end{equation}
where $\alpha$ is a learned scalar initialized at $0.01$. The prescribed prompt-side marginal fixes aggregate utilization; the exact decomposition separates retained prompt-update variation from within-node barycentric dispersion, while the frozen encoder and classifier determine downstream representation and prediction effects.

\subsection{Topology Augmentation}
\label{subsec:edge_fusion}
The transport branch operates in node--prompt space, whereas topology augmentation modifies the node-to-node propagation operator. For each experimental run, MINT constructs a directed exact $k$-NN index set $\mathcal{E}_{\mathrm{knn}}(X)$ once from the raw input features using cosine similarity, with self-neighbors excluded and no explicit symmetrization. For every indexed edge, each forward pass computes
\begin{equation}
\begin{aligned}
s_{ij}&=\operatorname{ReLU}\!\left(\cos(x_i^{\mathrm{adapted}},x_j^{\mathrm{adapted}})\right),\\
q_{ij}(X_{\mathrm{adapted}})&=
\frac{s_{ij}}{\max_{(a,b)\in\mathcal{E}_{\mathrm{knn}}}s_{ab}+10^{-8}}.
\end{aligned}
\label{eq:dynamic_knn_weight}
\end{equation}
With a fixed hyperparameter $\tau$, the fused propagation graph is
\begin{equation}
\begin{aligned}
(\mathcal{E}_{\mathrm{final}},w_{\mathrm{final}})=\operatorname{coalesce}\big(&
(\mathcal{E}_{\mathrm{ori}},(1-\tau)w_{\mathrm{ori}})\cup{}\\
&(\mathcal{E}_{\mathrm{knn}}(X),\tau q(X_{\mathrm{adapted}}))\big),
\end{aligned}
\label{eq:weighted_graph_fusion}
\end{equation}
where overlapping edges are merged by summing their weights and no post-fusion normalization is applied. We denote the resulting weighted propagation graph by $G'=(\mathcal{V},\mathcal{E}_{\mathrm{final}},w_{\mathrm{final}})$. Structural pre-aggregation in~\eqref{eq:x_struct} uses the native graph, whereas downstream frozen-encoder propagation uses $G'$.

\textbf{Adapting to Discrete Architectural Constraints:} Prompt-utilization control and graph topology address different objects. For backbones that support topology replacement but do not consume external scalar edge weights, the corresponding experiments use topology-compatible variants. Original MINT retains the weighted directed fusion in~\eqref{eq:weighted_graph_fusion}, with fixed indices and dynamically recomputed scalar weights.

\subsection{Joint Optimization and Inference}
\label{subsec:optimization_inference}

The complete forward pass and optimization procedure are summarized in Algorithm~\ref{alg:mint}. The framework differentiates through the finite Sinkhorn computation and the adapted-feature scalar edge weights. With frozen encoder parameters $\theta$, the predictions are $\hat{Y}=h_\phi(F_\theta(X_{\mathrm{adapted}},G'))$. The classification loss $\mathcal{L}_{cls}$ is cross-entropy over $\mathcal{V}_{tr}$, and $\mathcal{L}_{OT}=\sum_{i,j}\widehat\pi_{ij}C_{ij}$ is the expected transport cost:
\begin{equation}
    \mathcal{L}_{total} = \mathcal{L}_{cls} + \beta \mathcal{L}_{OT}.
\label{eq:total_loss}
\end{equation}
The coefficient $\beta$ weights geometric compatibility in the training objective. Aggregate balance follows from the feasible transport marginals; downstream representations are determined by the complete forward and optimization paths.

At transductive evaluation, the selected prompt and classifier states are restored; transport and adapted-feature $k$-NN weights are recomputed before the frozen-encoder forward pass on the full graph.

\begin{algorithm}[!t]
\caption{Forward Pass and Optimization of MINT}
\label{alg:mint}
\LinesNumbered
\KwIn{Graph $\mathcal{G}=(\mathcal{V}, \mathcal{E}_{ori}, X)$ with native weights $w_{ori}$, pretrained encoder $F_{\theta}$, prompt bank size $M$, hyperparameters $\epsilon, \beta, k, \tau$.}
\KwOut{Adapted features $X_{\mathrm{adapted}}$, predictions $\hat{Y}$, optimized parameters $\{P,\gamma,\alpha,\phi\}$.}

\tcp{Initialization}
Initialize $P$ with Xavier uniform, $\alpha=0.01$, and $\gamma=0.5$; initialize classifier $h_\phi$\;
Construct directed $k$-NN indices $\mathcal{E}_{knn}(X)$ once from raw $X$ using cosine similarity, excluding self-neighbors\;
Freeze pretrained encoder parameters $\theta$\;

\For{each optimization epoch}{
    \tcp{1. Structural Context Injection (Eq.~\ref{eq:x_struct})}
    Compute $X_{struct} = (1 - \gamma) X + \gamma (\tilde{D}^{-1/2} \tilde{A} \tilde{D}^{-1/2} X)$\;
    
    \tcp{2. Fixed-Marginal OT Engine (Eq.~\ref{eq:ot_objective})}
    Row-normalize $X_{struct}$ and $P$; compute $C^{raw}$ and $C$ by Eq.~\ref{eq:implemented_cost}\;
    Set $K=\exp(-C/\epsilon)$, $\mu=\frac{1}{N}\mathbf{1}_N$, $\nu=\frac{1}{M}\mathbf{1}_M$, and $u=v=\mathbf{1}$\;
    \For{$iter = 1$ \KwTo $20$}{
        \tcp{$\oslash$ denotes element-wise division}
        $u = \mu \oslash (K v + 10^{-8})$\;
        $v = \nu \oslash (K^\top u + 10^{-8})$\;
    }
    Form the finite coupling $\widehat\Pi = \operatorname{diag}(u)K\operatorname{diag}(v)$\;
    
    \tcp{3. Prompt-Induced Feature Adaptation (Eq.~\ref{eq:feat_adapt})}
    Set $\widehat R=N\widehat\Pi$; compute $\Delta X_P=\widehat R P$ and $X_{\mathrm{adapted}}=X+\alpha\Delta X_P$\;

    \tcp{4. Dynamically Weighted Topology Augmentation}
    Compute $q(X_{\mathrm{adapted}})$ on fixed $\mathcal{E}_{knn}(X)$ by Eq.~\ref{eq:dynamic_knn_weight}\;
    Coalesce $(\mathcal{E}_{ori},(1-\tau)w_{ori})$ and $(\mathcal{E}_{knn}(X),\tau q)$, summing overlaps\;
    
    \tcp{5. Downstream Inference \& Joint Optimization}
    Get predictions $\hat{Y}=h_\phi(F_\theta(X_{\mathrm{adapted}},G'))$\;
    Compute $\mathcal{L}_{cls} = \text{CrossEntropy}(\hat{Y}, Y)$ over $\mathcal{V}_{tr}$\;
    Compute transport-cost term $\mathcal{L}_{OT} = \sum \widehat\pi_{ij} C_{ij}$\;
    $\mathcal{L}_{total} = \mathcal{L}_{cls} + \beta \mathcal{L}_{OT}$ \tcp*[r]{Eq.~\ref{eq:total_loss}}
    Update $\{P,\gamma,\alpha,\phi\}$ via backpropagation\;
}
\Return $X_{\mathrm{adapted}}$, $\hat{Y}$, $\{P,\gamma,\alpha,\phi\}$
\end{algorithm}

\subsection{Complexity Analysis}
\label{subsec:complexity}

MINT uses a node--prompt transport problem with $M\ll N$, separating its bipartite transport cost from graph construction.

\textbf{Time Complexity:} Native structural aggregation costs $\mathcal{O}(|\mathcal{E}_{ori}|d)$. Constructing the transport cost and kernel and applying the prompt-induced feature update costs $\mathcal{O}(NMd)$. The $L_{\mathrm{SK}}$ Sinkhorn scaling updates cost $\mathcal{O}(L_{\mathrm{SK}}NM)$, where $L_{\mathrm{SK}}=20$ denotes the number of Sinkhorn iterations, giving $\mathcal{O}(NMd+L_{\mathrm{SK}}NM)$ for the bipartite OT component. For each experimental run, exact cosine $k$-NN construction is performed once with $\Theta(N^2d)$ work, separate from the reported per-epoch training cost. Dynamic weights on the fixed $Nk$ indices add $\mathcal{O}(Nkd)$ work per forward pass.

\textbf{Space Complexity:} MINT freezes the pretrained encoder and optimizes $\{P,\gamma,\alpha,\phi\}$. The cost, kernel, and finite coupling require $\mathcal{O}(NM)$ primary transport storage, in addition to features/prompts, Sinkhorn scaling vectors, graph storage, and autograd intermediates. Batched exact cosine search can bound temporary score storage while retaining the quadratic work of exact search; minibatch transport is a possible extension for graphs whose full $N\times M$ coupling exceeds device memory.

\section{Experiments}
\label{sec:experiments}
\subsection{Experimental Setup}
\label{subsec:exp_setup}

We evaluate MINT on nine real-world node classification datasets. Cora, CiteSeer, and PubMed are standard citation networks~\cite{yang2016cora}; Cornell, Texas, Wisconsin, Chameleon, Squirrel, and Actor~\cite{pei2020cornell} broaden the evaluation to heterophilic regimes. To assess compatibility across pretrained backbones, we use three self-supervised GNN backbones: GraphMAE~\cite{hou2022graphmae}, GraphMAE2~\cite{hou2023graphmae2}, and MaskGAE~\cite{li2023maskgae}. We compare MINT against 13 baselines, including two standard tuning paradigms (Fine-tune and Linear Probe) and $11$ graph prompting baselines (GPF~\cite{fang2023gpf}, GPF+~\cite{fang2023gpf}, GPPT~\cite{sun2022gppt}, GraphPrompt~\cite{liu2023graphprompt}, EdgePrompt~\cite{fu2025edgeprompt}, EdgePrompt+~\cite{fu2025edgeprompt}, All-in-One~\cite{sun2023allinone}, HS-GPPT~\cite{luo2025hsgppt}, ProNoG~\cite{yu2025pronog}, UniPrompt~\cite{huang2025uniprompt}, and DAPrompt~\cite{jiang2026daprompt}).

\begin{table}[!t]
  \centering
  \caption{Statistics of Real-World Datasets}
  \label{tab:dataset_stats}
  \setlength{\tabcolsep}{5pt}
  \begin{tabular}{lrrrrr}
    \toprule
    Dataset & \#Nodes & \#Edges & \#Features & \#Classes & \#Homophily \\
    \midrule
    Cora      & 2,708  & 5,278   & 1,433 & 7 & 0.81 \\
    CiteSeer  & 3,327  & 4,552   & 3,703 & 6 & 0.74 \\
    PubMed    & 19,717 & 44,324  & 500   & 3 & 0.80 \\
    Cornell   & 183    & 298     & 1,703 & 5 & 0.31 \\
    Texas     & 183    & 325     & 1,703 & 5 & 0.11 \\
    Wisconsin & 251    & 515     & 1,703 & 5 & 0.20 \\
    Chameleon & 2,277  & 36,101  & 2,277 & 5 & 0.24 \\
    Squirrel  & 5,201  & 217,073 & 2,089 & 5 & 0.22 \\
    Actor     & 7,600  & 30,019  & 932   & 5 & 0.22 \\
    \bottomrule
  \end{tabular}
\end{table}

We report results over 30 trials; complete adaptation, selection, seed, perturbation, ablation, checkpoint, and hardware details are provided in Appendix F.

\begin{table*}[!t]
  \centering
  \caption{Mean accuracy $\pm$ standard deviation under 1-shot node classification. The best results across all methods are highlighted in \textbf{bold}, and the second-best results are \underline{underlined}.}
  \label{tab:1shot_main_results}
  \resizebox{\textwidth}{!}{
    \begin{tabular}{c l ccc cccccc}
      \toprule
      \multirow{2}{*}{\textbf{Pretrain}} & \multirow{2}{*}{\textbf{Methods}} & \multicolumn{3}{c}{\textbf{Homophilic Graphs}} & \multicolumn{6}{c}{\textbf{Heterophilic Graphs}} \\
      \cmidrule(lr){3-5} \cmidrule(lr){6-11}
      & & \textbf{Cora} & \textbf{CiteSeer} & \textbf{PubMed} & \textbf{Cornell} & \textbf{Texas} & \textbf{Wisconsin} & \textbf{Chameleon} & \textbf{Squirrel} & \textbf{Actor} \\
      \midrule

      \multirow{11}{*}{\textbf{GraphMAE}} 
      & Fine-tune    & 52.57$\pm$7.44 & 42.02$\pm$7.53 & \underline{57.63$\pm$6.73} & 32.30$\pm$10.98 & 33.57$\pm$17.20 & 36.99$\pm$10.54 & \underline{24.56$\pm$3.49} & 21.55$\pm$2.77 & 20.78$\pm$2.26 \\
      & Linear Probe & 50.85$\pm$6.55 & 45.48$\pm$8.47 & 56.63$\pm$7.18 & 28.12$\pm$8.18  & 33.20$\pm$19.79 & 31.46$\pm$11.45 & 23.76$\pm$3.34 & 20.28$\pm$0.85 & 19.79$\pm$2.90 \\
      & GPF          & 51.40$\pm$8.89 & 43.49$\pm$9.45 & 54.04$\pm$6.77 & 25.74$\pm$9.60  & 41.22$\pm$20.54 & 32.91$\pm$11.52 & 22.94$\pm$3.16 & 20.51$\pm$0.96 & 20.04$\pm$3.18 \\
      & GPF+         & 51.66$\pm$7.81 & 45.83$\pm$7.97 & 55.83$\pm$7.28 & 30.08$\pm$14.47 & 33.68$\pm$24.17 & 33.10$\pm$11.27 & 23.83$\pm$3.02 & 20.36$\pm$0.80 & 19.45$\pm$2.35 \\
      & GPPT         & 52.43$\pm$6.16 & 46.21$\pm$7.71 & 56.16$\pm$7.67 & 28.82$\pm$7.60  & 31.67$\pm$21.40 & 30.53$\pm$8.92  & 23.86$\pm$2.72 & 20.46$\pm$0.92 & 20.59$\pm$3.23 \\
      & GraphPrompt  & 51.84$\pm$8.36 & 44.28$\pm$8.83 & 56.38$\pm$6.99 & 31.46$\pm$11.56 & 44.71$\pm$20.93 & 36.08$\pm$12.75 & 22.90$\pm$3.33 & 20.48$\pm$0.70 & 19.26$\pm$3.17 \\
      & EdgePrompt   & 51.14$\pm$7.90 & 44.26$\pm$8.43 & 56.15$\pm$6.79 & 26.61$\pm$9.34  & 32.96$\pm$20.52 & 30.37$\pm$12.05 & 23.28$\pm$2.85 & 20.43$\pm$0.68 & 19.76$\pm$2.95 \\
      & EdgePrompt+  & 51.02$\pm$7.44 & 44.85$\pm$8.09 & 56.44$\pm$8.21 & 31.79$\pm$15.41 & 37.20$\pm$24.61 & 31.76$\pm$12.39 & 23.56$\pm$3.04 & 20.30$\pm$0.46 & 20.29$\pm$2.68 \\
      & All-in-One   & 34.33$\pm$6.60 & 22.13$\pm$3.91 & 41.49$\pm$5.74 & 33.08$\pm$17.41 & 36.88$\pm$22.55 & 31.57$\pm$16.25 & 22.50$\pm$2.96 & \textbf{23.42$\pm$3.19} & 20.31$\pm$2.22 \\
      & HS-GPPT      & \underline{52.72$\pm$8.72} & 37.65$\pm$7.50 & 52.79$\pm$8.39 & 30.17$\pm$10.59 & 36.93$\pm$21.21 & 34.03$\pm$9.46 & 23.20$\pm$2.50 & \underline{22.12$\pm$2.08} & 20.57$\pm$3.00 \\
      & ProNoG       & 44.95$\pm$7.65 & 36.37$\pm$7.73 & 53.78$\pm$9.79 & 28.60$\pm$9.98 & 33.17$\pm$19.12 & 31.21$\pm$15.03 & 24.18$\pm$3.24 & 20.26$\pm$1.36 & 20.42$\pm$3.73 \\
      & UniPrompt    & 47.75$\pm$9.27 & \underline{51.18$\pm$7.04} & 53.03$\pm$7.85 & \underline{49.33$\pm$12.69} & \underline{50.98$\pm$17.28} & \underline{61.76$\pm$13.00} & 23.55$\pm$3.21 & 20.72$\pm$1.60 & \underline{20.84$\pm$3.22} \\
      & DAPrompt     & 36.97$\pm$7.48 & 31.53$\pm$5.51 & 47.09$\pm$7.41 & 30.70$\pm$11.17 & 36.19$\pm$17.69 & 37.09$\pm$8.85 & 22.59$\pm$3.65 & 20.71$\pm$1.36 & 20.33$\pm$2.80 \\
      \cmidrule{2-11}
      & \cellcolor{gray!15}\textbf{MINT (Ours)} & \cellcolor{gray!15}\textbf{53.46$\pm$7.36} & \cellcolor{gray!15}\textbf{52.08$\pm$6.62} & \cellcolor{gray!15}\textbf{58.18$\pm$7.20} & \cellcolor{gray!15}\textbf{50.73$\pm$13.82} & \cellcolor{gray!15}\textbf{53.39$\pm$20.14} & \cellcolor{gray!15}\textbf{66.22$\pm$14.85} & \cellcolor{gray!15}\textbf{25.14$\pm$2.93} & \cellcolor{gray!15}21.83$\pm$2.00 & \cellcolor{gray!15}\textbf{22.73$\pm$3.26} \\
      \midrule

      \multirow{10}{*}{\textbf{GraphMAE2}} 
      & Fine-tune    & \underline{44.09$\pm$7.82} & 36.99$\pm$6.07 & \underline{53.38$\pm$8.31} & 32.89$\pm$10.75 & 36.53$\pm$17.43 & 36.51$\pm$9.74  & 24.85$\pm$3.67 & 21.54$\pm$2.10 & 20.67$\pm$1.86 \\
      & Linear Probe & 42.65$\pm$7.06 & 40.60$\pm$6.53 & 48.89$\pm$8.41 & 29.41$\pm$7.06  & 29.39$\pm$12.62 & 28.56$\pm$7.37  & \underline{25.34$\pm$2.84} & 21.05$\pm$1.58 & 19.52$\pm$2.15 \\
      & GPF          & 41.37$\pm$10.95 & 37.48$\pm$8.00 & 48.93$\pm$10.92 & 31.12$\pm$13.52 & 39.68$\pm$23.18 & 34.51$\pm$13.60 & 23.57$\pm$2.77 & 20.61$\pm$0.82 & 20.05$\pm$2.79 \\
      & GPF+         & 42.82$\pm$9.36 & 39.26$\pm$8.44 & 52.30$\pm$10.85 & 30.87$\pm$14.30 & 42.01$\pm$20.66 & 33.73$\pm$14.09 & 23.53$\pm$2.86 & 20.47$\pm$0.58 & 20.15$\pm$2.76 \\
      & GPPT         & 42.54$\pm$8.48 & 41.88$\pm$7.56 & 47.63$\pm$9.26 & 28.04$\pm$8.71  & 26.64$\pm$10.11 & 27.52$\pm$7.78  & 24.65$\pm$3.22 & 20.74$\pm$1.30 & 20.70$\pm$1.61 \\
      & GraphPrompt  & 42.12$\pm$8.53 & 38.48$\pm$6.83 & 50.16$\pm$7.40 & 28.24$\pm$7.42  & 29.42$\pm$12.06 & 28.25$\pm$7.04  & 24.86$\pm$2.92 & 21.01$\pm$1.59 & 19.70$\pm$2.65 \\
      & EdgePrompt   & 42.07$\pm$10.62 & 35.74$\pm$7.76 & 50.30$\pm$11.38 & 32.30$\pm$13.89 & 37.20$\pm$24.50 & 35.76$\pm$11.85 & 24.00$\pm$2.78 & 20.89$\pm$1.17 & 19.46$\pm$2.55 \\
      & EdgePrompt+  & 42.28$\pm$7.63 & 36.19$\pm$8.46 & 51.04$\pm$9.15 & 31.18$\pm$15.09 & 35.45$\pm$22.32 & 35.10$\pm$15.38 & 23.11$\pm$3.35 & 20.90$\pm$1.02 & 19.55$\pm$2.45 \\
      & All-in-One   & 35.40$\pm$7.48 & 35.99$\pm$7.95 & 45.74$\pm$6.38 & 31.46$\pm$13.27 & 36.67$\pm$19.96 & 30.86$\pm$15.08 & 23.30$\pm$2.86 &\textbf{22.17$\pm$2.48} & 20.82$\pm$2.81 \\
      & HS-GPPT      & 38.82$\pm$7.20 & 33.33$\pm$8.73 & 49.62$\pm$6.76 & 31.34$\pm$15.05 & \underline{42.54$\pm$22.05} & 30.55$\pm$16.02 & 22.02$\pm$2.35 & 21.35$\pm$1.47 & \underline{21.83$\pm$2.35} \\
      & ProNoG       & 43.83$\pm$7.76 & 35.39$\pm$7.62 & 51.83$\pm$9.53 & 30.34$\pm$15.61 & 36.14$\pm$20.84 & 38.22$\pm$13.65 & 25.20$\pm$3.16 & 21.06$\pm$1.50 & 20.35$\pm$3.48 \\
      & UniPrompt    & 42.01$\pm$9.24 & \underline{48.55$\pm$7.58} & 46.93$\pm$8.29 & \underline{48.71$\pm$15.62} & 39.23$\pm$18.73 & \underline{45.53$\pm$16.60} & 24.21$\pm$3.21 & 21.33$\pm$1.76 & 20.83$\pm$2.71 \\
      & DAPrompt     & 35.83$\pm$9.24 & 26.25$\pm$4.25 & 47.82$\pm$7.15 & 30.73$\pm$14.38 & 26.27$\pm$12.49 & 31.50$\pm$11.32 & 22.72$\pm$2.39 & 21.45$\pm$1.76 & 19.90$\pm$2.49 \\
      \cmidrule{2-11}
      & \cellcolor{gray!15}\textbf{MINT (Ours)} & \cellcolor{gray!15}\textbf{44.45$\pm$8.76} & \cellcolor{gray!15}\textbf{49.34$\pm$7.30} & \cellcolor{gray!15}\textbf{53.88$\pm$8.14} & \cellcolor{gray!15}\textbf{50.76$\pm$14.79} & \cellcolor{gray!15}\textbf{48.70$\pm$20.65} & \cellcolor{gray!15}\textbf{56.83$\pm$16.76} & \cellcolor{gray!15}\textbf{25.43$\pm$2.41} & \cellcolor{gray!15}\underline{22.04$\pm$1.76} & \cellcolor{gray!15}\textbf{22.34$\pm$2.54} \\
      \midrule
      
      \multirow{10}{*}{\textbf{MaskGAE}} 
      & Fine-tune     & 48.35$\pm$8.86 & 38.92$\pm$7.08 & 53.55$\pm$9.81 & 39.52$\pm$12.91 & 37.99$\pm$19.89 & 32.80$\pm$10.11 & 27.64$\pm$4.95 & 22.58$\pm$2.10 & 21.29$\pm$2.67 \\
      & Linear Probe  & 47.82$\pm$8.85 & 38.63$\pm$7.43 & 53.26$\pm$9.73 & 38.66$\pm$13.72 & 37.22$\pm$19.75 & 31.75$\pm$9.62 & 27.50$\pm$4.72 & 22.47$\pm$2.05 & 21.15$\pm$2.68 \\
      & GPF           & 34.48$\pm$6.80 & 28.83$\pm$5.94 & 48.59$\pm$8.25 & 39.66$\pm$11.88 & \underline{55.29$\pm$15.58} & 33.73$\pm$10.86 & 28.53$\pm$5.01 & \underline{25.99$\pm$4.36} & 21.25$\pm$2.74 \\
      & GPF+          & 40.09$\pm$8.19 & 33.53$\pm$8.09 & 52.18$\pm$8.11 & 41.26$\pm$14.38 & 49.05$\pm$18.34 & 37.50$\pm$12.29 & \underline{28.87$\pm$4.40} & 25.86$\pm$3.68 & 20.49$\pm$2.53 \\
      & GPPT          & 48.79$\pm$8.87 & 38.75$\pm$8.31 & 54.05$\pm$7.57 & 33.17$\pm$14.00 & 40.56$\pm$19.22 & 31.00$\pm$13.93 & 27.78$\pm$4.72 & 22.55$\pm$2.14 & 20.90$\pm$2.12 \\
      & GraphPrompt   & 48.38$\pm$9.24 & 39.28$\pm$7.64 & 54.39$\pm$9.49 & 40.59$\pm$12.89 & 45.32$\pm$20.01 & 32.26$\pm$9.35 & 28.39$\pm$4.39 & 23.00$\pm$2.23 & 20.86$\pm$2.82 \\
      & EdgePrompt    & 48.81$\pm$8.56 & 39.81$\pm$7.56 & 53.87$\pm$9.47 & 37.96$\pm$13.49 & 46.32$\pm$16.65 & 32.91$\pm$9.23 & 28.60$\pm$4.81 & 23.38$\pm$2.40 & 21.06$\pm$2.38 \\
      & EdgePrompt+   & 48.65$\pm$8.72 & 38.01$\pm$7.27 & 53.91$\pm$9.88 & 36.92$\pm$13.72 & 40.37$\pm$20.36 & 34.37$\pm$8.80 & 28.18$\pm$4.99 & 23.17$\pm$2.43 & 21.17$\pm$2.66 \\
      & All-in-One    & 49.91$\pm$7.51 & 37.04$\pm$6.91 & 51.35$\pm$10.74 & 45.60$\pm$12.66 & 48.68$\pm$21.37 & 37.20$\pm$12.39 & 28.08$\pm$4.29 & 22.89$\pm$2.52 & 20.39$\pm$2.52 \\
      & HS-GPPT       & \underline{49.94$\pm$8.14} & 36.56$\pm$7.92 & 51.46$\pm$10.78 & 46.33$\pm$12.04 & 50.95$\pm$21.82 & 34.78$\pm$13.38 & 27.95$\pm$4.14 & 22.83$\pm$2.53 & 20.68$\pm$2.95 \\
      & ProNoG        & 34.18$\pm$6.18 & 26.75$\pm$4.41 & 51.27$\pm$8.88 & 42.04$\pm$14.58 & 53.81$\pm$16.70 & 37.68$\pm$12.04 & 28.01$\pm$4.55 & 25.66$\pm$3.91 & 21.21$\pm$2.36 \\
      & UniPrompt     & 48.51$\pm$9.45 & \underline{48.84$\pm$6.83} & \underline{54.42$\pm$9.23} & \underline{50.08$\pm$15.16} & 44.68$\pm$14.47 & \underline{62.09$\pm$15.87} & 25.67$\pm$2.91 & 20.99$\pm$1.47 & \underline{21.52$\pm$2.43} \\
      & DAPrompt      & 49.46$\pm$8.42 & 37.57$\pm$8.09 & 53.54$\pm$9.72 & 38.94$\pm$17.17 & 43.12$\pm$17.15 & 32.09$\pm$12.34 & 27.36$\pm$4.98 & 23.13$\pm$2.56 & 20.27$\pm$2.24 \\
      \cmidrule{2-11}
      & \cellcolor{gray!15}\textbf{MINT (Ours)} & \cellcolor{gray!15}\textbf{51.78$\pm$9.97} & \cellcolor{gray!15}\textbf{49.81$\pm$7.11} & \cellcolor{gray!15}\textbf{55.11$\pm$8.48} & \cellcolor{gray!15}\textbf{55.10$\pm$6.97} & \cellcolor{gray!15}\textbf{55.85$\pm$16.08} & \cellcolor{gray!15}\textbf{66.81$\pm$15.28} & \cellcolor{gray!15}\textbf{33.65$\pm$3.71} & \cellcolor{gray!15}\textbf{27.26$\pm$3.19}  & \cellcolor{gray!15}\textbf{22.31$\pm$1.82} \\
      \bottomrule
    \end{tabular}%
  }
\end{table*}

\subsection{Main Results and Analysis}
\label{subsec:main_results}

Table~\ref{tab:1shot_main_results} summarizes the classification results. We organize the observations across the evaluated settings into three empirical trends.
First, in the evaluated 1-shot settings, several prompting baselines perform below linear probing, whereas MINT remains competitive across the reported datasets.
Second, on Cora, CiteSeer, and PubMed, MINT matches or modestly improves on strong baselines in many settings, although the performance margins over competing methods are generally modest. These comparisons are distinct from the end-to-end component ablations reported below.
Performance also varies with the pretrained backbone. On Cora, for example, the highest mean accuracy under GraphMAE2 is approximately $44\%$, compared with approximately $53\%$ under GraphMAE. Under GraphMAE2 on Wisconsin, MINT reaches $56.83\%$ while several baselines lie in the $20\%$--$30\%$ range.

Third, MINT obtains some of its largest margins on Texas, Wisconsin, and Cornell. The ablations identify a strong but graph-dependent role for the $k$-NN augmentation branch in several heterophilic settings. The complete 3-shot and 5-shot matrices show similar dataset- and backbone-dependent trends: MINT remains competitive across pretrained backbones and achieves comparatively high mean accuracy in several heterophilic settings, while rankings remain dataset- and backbone-dependent and some citation-network settings favor fine-tuning or other prompting methods. Complete results are reported in Appendix F-F and F-G.

\subsection{Cross-Domain Few-Shot Transfer}
\label{subsec:cross_domain}

To examine source-dependent transfer, we conduct 1-shot node classification from PubMed and Actor to Cora, CiteSeer, Texas, Cornell, and Wisconsin. Each source--method--target setting is evaluated over 30 predetermined target splits. We use source-specific pretrained checkpoints and deterministic, label-free feature alignment; full alignment, model-selection, evaluation, and method-specific adaptation details are provided in Appendix F-B.

\begin{table*}[!t]
  \centering
  \caption{Cross-domain 1-shot node classification accuracy (\%). }
  \label{tab:cross_domain_results}
  \renewcommand{\arraystretch}{0.95}
  \setlength{\tabcolsep}{5pt}
  \resizebox{\textwidth}{!}{%
  \begin{tabular}{llccccc}
      \toprule
      \textbf{Source} & \textbf{Method} 
      & \textbf{Cora} 
      & \textbf{CiteSeer} 
      & \textbf{Texas} 
      & \textbf{Cornell} 
      & \textbf{Wisconsin} \\
      \midrule

      \multirow{7}{*}{\textbf{PubMed}}
      & DGI Fine-tune
      & $30.56\pm6.20$
      & $24.94\pm5.15$
      & $41.27\pm20.72$
      & $41.46\pm14.04$
      & $37.97\pm12.59$ \\

      & DGI Linear
      & $19.45\pm3.89$
      & $21.29\pm3.37$
      & $33.62\pm19.47$
      & $34.79\pm14.02$
      & $31.76\pm14.46$ \\

      & GraphMAE Linear
      & $33.24\pm6.60$
      & $30.54\pm5.56$
      & $32.35\pm20.66$
      & $26.75\pm9.71$
      & $32.99\pm13.22$ \\

      & GPPT
      & $34.38\pm6.72$
      & $29.54\pm5.69$
      & $29.89\pm21.34$
      & $26.89\pm15.59$
      & $33.24\pm13.78$ \\

      & UniPrompt
      & $38.00\pm7.56$
      & $44.13\pm6.69$
      & $\mathbf{46.11\pm15.85}$
      & $50.64\pm15.44$
      & $55.08\pm17.35$ \\

      & DAPrompt
      & $19.77\pm3.21$
      & $19.33\pm2.61$
      & $31.75\pm16.50$
      & $26.95\pm9.09$
      & $29.52\pm9.28$ \\

      \rowcolor{gray!15}
      & \textbf{MINT (Ours)}
      & $\mathbf{40.10\pm7.46}$
      & $\mathbf{44.38\pm6.68}$
      & $45.77\pm17.12$
      & $\mathbf{51.57\pm14.86}$
      & $\mathbf{56.10\pm18.02}$ \\

      \midrule

      \multirow{7}{*}{\textbf{Actor}}
      & DGI Fine-tune
      & $27.73\pm7.63$
      & $24.07\pm3.56$
      & $38.73\pm21.46$
      & $38.24\pm14.02$
      & $39.45\pm12.83$ \\

      & DGI Linear
      & $18.36\pm2.61$
      & $18.37\pm2.58$
      & $37.75\pm22.42$
      & $29.50\pm11.98$
      & $36.19\pm13.66$ \\

      & GraphMAE Linear
      & $29.89\pm6.65$
      & $25.48\pm4.38$
      & $31.93\pm17.91$
      & $29.64\pm11.89$
      & $36.65\pm12.33$ \\

      & GPPT
      & $24.16\pm8.31$
      & $18.80\pm5.05$
      & $31.06\pm20.91$
      & $22.30\pm11.25$
      & $19.64\pm14.37$ \\

      & UniPrompt
      & $37.73\pm6.73$
      & $\mathbf{40.87\pm5.86}$
      & $\mathbf{48.65\pm20.28}$
      & $46.86\pm15.51$
      & $51.18\pm18.28$ \\

      & DAPrompt
      & $17.91\pm3.27$
      & $18.48\pm2.84$
      & $32.67\pm15.20$
      & $29.89\pm9.48$
      & $30.64\pm10.68$ \\

      \rowcolor{gray!15}
      & \textbf{MINT (Ours)}
      & $\mathbf{39.12\pm7.29}$
      & $39.37\pm5.52$
      & $47.04\pm20.57$
      & $\mathbf{47.25\pm15.14}$
      & $\mathbf{53.64\pm17.18}$ \\

      \bottomrule
  \end{tabular}%
  }
\end{table*}

\textbf{Analysis:} Table~\ref{tab:cross_domain_results} shows that performance varies with the pretraining source. The strongest method also varies by target: MINT has the highest mean in seven of the ten source--target columns, while UniPrompt has the highest mean in the other three. MINT therefore remains competitive across both sources but is not universally best. In particular, UniPrompt is slightly higher on Texas with PubMed as source and on CiteSeer and Texas with Actor as source. These source- and target-dependent results motivate reporting the complete transfer matrix.

\subsection{Mechanism Analysis of Prompt Allocation}
\label{subsec:mechanism_analysis}

We examine graph-wide allocation in prompt-induced updates and the frozen-encoder forward path. Unperturbed training runs show that local normalization can coexist with concentrated aggregate use, while $V_{\Delta X_P}$ measures the injected-update geometry more directly than allocation Gini alone.

The controlled prompt-attractiveness intervention raises the dominant prompt share under matched routing to approximately $0.997$ and reduces $V_{\Delta X_P}$ to near zero. Holding the learned state fixed, the corresponding frozen-forward accuracy changes are $-6.26$, $-30.48$, and $-21.95$ percentage points on Cora, Texas, and Wisconsin, respectively. Retraining substantially compensates on Cora and Wisconsin, while Texas retains an approximately $-5.17$-point change. In the same-topology control, matched-softmax and fixed-marginal Sinkhorn routing use the same frozen weighted propagation graph and matched model state. Fixed-marginal Sinkhorn reduces allocation Gini by $0.0406$--$0.2515$ across the four tested settings, whereas the paired accuracy differences are only $+0.05$ to $+0.14$ points and all paired 95\% confidence intervals include zero. These results separate the direct allocation-control effect from downstream prediction: fixed-marginal routing directly controls prompt utilization, whereas its immediate matched-state effect and accumulated training-time contribution are distinct quantities. Appendix F-C reports the complete protocol and measurements.

\subsection{Robustness Analysis under Extreme Perturbations}
\label{subsec:robustness}

We evaluate MINT under structural and feature perturbations in extreme few-shot settings. Figure~\ref{fig:robustness} shows accuracy across perturbation ratios $p$; Table~\ref{tab:robustness_summary} reports both absolute accuracy and relative change at $p=0.8$, providing complementary views of robustness.

\textbf{Robustness against Edge Drop:} 
Removing up to $80\%$ of edges substantially changes native connectivity. MINT maintains a stable trajectory across several evaluated settings; for example, it retains $78.57\%$ accuracy on Wisconsin 5-shot at $p=0.8$, a $1.8\%$ relative decrease. GPF and DAPrompt occasionally improve after edge removal on heterophilic graphs, showing that edge removal is not uniformly detrimental in these settings. The feature-derived $k$-NN augmentation branch supplies additional propagation paths alongside the perturbed native graph, while the transport branch operates in node--prompt space.

\textbf{Robustness against Noise Edge:} 
Noise Edge deliberately injects cross-class edges to create heterophilic structural perturbations. MINT has a $1.3\%$ relative decrease on Wisconsin 1-shot and a $13.3\%$ decrease on Texas 1-shot, whereas DAPrompt accuracy decreases by $43.0\%$ and $44.4\%$ in the corresponding settings.

\textbf{Sensitivity to Feature Mask:} 
Feature masking directly perturbs the feature space used to construct the Sinkhorn cost matrix. MINT shows substantial relative degradation in several settings as the masking ratio increases, while retaining competitive absolute accuracy in others. Methods with lower clean-graph accuracy can exhibit smaller relative decreases, which motivates reporting both measures.

\begin{figure*}[!t]
    \centering
    \includegraphics[width=0.8\textwidth]{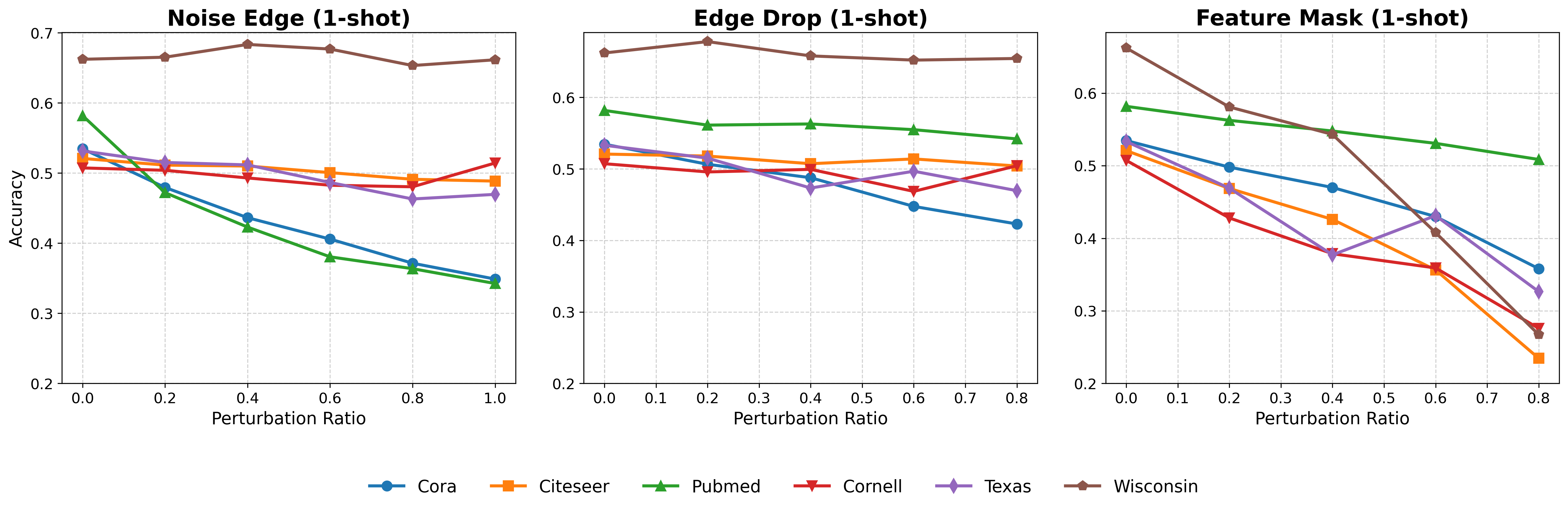} 
    \caption{Robustness analysis of MINT under extreme 1-shot settings on the GraphMAE backbone. The three panels show classification accuracy as the perturbation ratio varies for Noise Edge (left), Edge Drop (middle), and Feature Mask (right).}
    \label{fig:robustness}
\end{figure*}

\begin{table*}[!t]
  \centering
  \caption{Robustness evaluation on the GraphMAE backbone under structural and feature perturbations at ratio $0.8$. Subscripts denote relative change from the clean graph; $\downarrow$ indicates degradation and $\uparrow$ indicates improvement. Methods achieving both the highest absolute accuracy and the smallest relative degradation are highlighted in \textbf{bold}; otherwise, the highest absolute accuracy is \underline{underlined}.}
  \label{tab:robustness_summary}
  \resizebox{\textwidth}{!}{
    \begin{tabular}{c l cc cc cc}
      \toprule
      \multirow{2}{*}{\textbf{Perturbation}} & \multirow{2}{*}{\textbf{Methods}} & \multicolumn{2}{c}{\textbf{Cora} (Homo.)} & \multicolumn{2}{c}{\textbf{Texas} (Heter.)} & \multicolumn{2}{c}{\textbf{Wisconsin} (Heter.)} \\
      \cmidrule(lr){3-4} \cmidrule(lr){5-6} \cmidrule(lr){7-8}
      & & \textbf{1-shot} & \textbf{5-shot} & \textbf{1-shot} & \textbf{5-shot} & \textbf{1-shot} & \textbf{5-shot} \\
      \midrule
      \multicolumn{8}{c}{\textbf{Clean Graph (Ratio = 0.0)}} \\
      \midrule
      \multirow{3}{*}{None} 
      & GPF       & 51.40 & 69.52 & 41.22 & 50.52 & 32.91 & 38.44 \\
      & DAPrompt  & 36.97 & 50.77 & 36.19 & 43.25 & 37.09 & 42.00 \\
      & MINT (Ours) & 53.46 & 73.32 & 53.39 & 64.35 & 66.22 & 80.04 \\
      \midrule
      \multicolumn{8}{c}{\textbf{Extreme Perturbation (Ratio = 0.8)}} \\
      \midrule
      \multirow{3}{*}{\textbf{Noise Edge}}
      & GPF       & 27.37$_{\downarrow 46.8\%}$ & 41.79$_{\downarrow 39.9\%}$ & 30.40$_{\downarrow 26.2\%}$ & 44.55$_{\downarrow 11.8\%}$ & 26.43$_{\downarrow 19.7\%}$ & 36.80$_{\downarrow 4.3\%}$ \\
      & DAPrompt  & 21.08$_{\downarrow 43.0\%}$ & 38.77$_{\downarrow 23.6\%}$ & 20.13$_{\downarrow 44.4\%}$ & 26.35$_{\downarrow 39.1\%}$ & 20.94$_{\downarrow 43.5\%}$ & 30.56$_{\downarrow 27.2\%}$ \\
      & \cellcolor{gray!15}\textbf{MINT (Ours)} & \cellcolor{gray!15}\textbf{37.14}$_{\downarrow \textbf{30.5\%}}$ & \cellcolor{gray!15}\underline{49.08}$_{\downarrow 33.1\%}$ & \cellcolor{gray!15}\textbf{46.30}$_{\downarrow \textbf{13.3\%}}$ & \cellcolor{gray!15}\underline{56.72}$_{\downarrow 11.9\%}$ & \cellcolor{gray!15}\textbf{65.33}$_{\downarrow \textbf{1.3\%}}$ & \cellcolor{gray!15}\textbf{77.05}$_{\downarrow \textbf{3.7\%}}$ \\
      \midrule
      \multirow{3}{*}{\textbf{Feature Mask}}
      & GPF       & 32.15$_{\downarrow 37.5\%}$ & 53.09$_{\downarrow 23.6\%}$ & \textbf{35.71}$_{\downarrow \textbf{13.4\%}}$ & \textbf{45.83}$_{\downarrow \textbf{9.3\%}}$ & \textbf{26.84}$_{\downarrow \textbf{18.4\%}}$ & 32.56$_{\downarrow 15.3\%}$ \\
      & DAPrompt  & 17.93$_{\downarrow 51.5\%}$ & 35.52$_{\downarrow 30.0\%}$ & 26.01$_{\downarrow 28.1\%}$ & 36.12$_{\downarrow 16.5\%}$ & 25.33$_{\downarrow 31.7\%}$ & 28.57$_{\downarrow 32.0\%}$ \\
      & \cellcolor{gray!15}\textbf{MINT (Ours)} & \cellcolor{gray!15}\textbf{35.81}$_{\downarrow \textbf{33.0\%}}$ & \cellcolor{gray!15}\textbf{56.93}$_{\downarrow \textbf{22.4\%}}$ & \cellcolor{gray!15}32.65$_{\downarrow 38.8\%}$ & \cellcolor{gray!15}30.35$_{\downarrow 52.8\%}$ & \cellcolor{gray!15}26.72$_{\downarrow 59.6\%}$ & \cellcolor{gray!15}\underline{48.52}$_{\downarrow 39.4\%}$ \\
      \midrule
      \multirow{3}{*}{\textbf{Edge Drop}}
      & GPF       & 37.56$_{\downarrow 26.9\%}$ & 58.66$_{\downarrow 15.6\%}$ & 33.36$_{\downarrow 19.1\%}$ & 47.54$_{\downarrow 5.9\%}$ & 35.37$_{\uparrow 7.5\%}$ & 47.42$_{\uparrow 23.4\%}$ \\
      & DAPrompt  & 19.32$_{\downarrow 47.7\%}$ & 42.13$_{\downarrow 17.0\%}$ & 30.45$_{\downarrow 15.9\%}$ & 44.90$_{\uparrow 3.8\%}$ & 27.68$_{\downarrow 25.4\%}$ & 35.88$_{\downarrow 14.6\%}$ \\
      & \cellcolor{gray!15}\textbf{MINT (Ours)} & \cellcolor{gray!15}\textbf{42.30}$_{\downarrow \textbf{20.9\%}}$ & \cellcolor{gray!15}\underline{61.43}$_{\downarrow 16.2\%}$ & \cellcolor{gray!15}\textbf{46.96}$_{\downarrow \textbf{12.0\%}}$ & \cellcolor{gray!15}\underline{60.58}$_{\downarrow 5.9\%}$ & \cellcolor{gray!15}\underline{65.45}$_{\downarrow 1.2\%}$ & \cellcolor{gray!15}\underline{78.57}$_{\downarrow 1.8\%}$ \\
      \bottomrule
    \end{tabular}%
  }
\end{table*}

\subsection{Hyperparameter Sensitivity Analysis}
\label{subsec:hyperparameter_analysis}

The main text summarizes three representative hyperparameters: the prompt bank size $M$, the transport-cost coefficient $\beta$, and the $k$-NN neighborhood size $k$. Tables~\ref{tab:sens_M}, \ref{tab:sens_beta}, and \ref{tab:sens_k} cover representative graph regimes; Appendix F-L reports all four hyperparameters, including $\epsilon$.

\begin{table}[ht]
  \centering
  \caption{Sensitivity to Prompt Bank Size $M$ (GraphMAE, Texas)}
  \label{tab:sens_M}
  \resizebox{0.8\columnwidth}{!}{
    \begin{tabular}{l ccc}
      \toprule
      \textbf{Setting} & $\boldsymbol{M=1}$ & $\boldsymbol{M=10}$ & $\boldsymbol{M=50}$ \\
      \midrule
      1-Shot & 34.84\% & \textbf{51.90\%} & 44.29\% \\
      5-Shot & 55.39\% & \textbf{62.00\%} & 58.52\% \\
      \bottomrule
    \end{tabular}%
  }
\end{table}

\textbf{Prompt Bank Size ($\boldsymbol{M}$):} The prompt bank size $M$ controls the number of learnable prompt vectors available for adaptation. Table~\ref{tab:sens_M} shows non-monotonic sensitivity on Texas, which has 183 nodes and 1,703 features. In the 1-shot setting, accuracy is $34.84\%$ for $M=1$, rises by $17.06$ points to $51.90\%$ for $M=10$, and decreases to $44.29\%$ for $M=50$. The tested configuration therefore exhibits a non-monotonic dependence on prompt bank size and achieves its highest value at an intermediate $M$ in the tested settings.

\begin{table}[ht]
  \centering
  \caption{Sensitivity to Transport-Cost Coefficient $\beta$ (GraphMAE, 5-Shot)}
  \label{tab:sens_beta}
  \resizebox{0.8\columnwidth}{!}{
    \begin{tabular}{l ccc}
      \toprule
      \textbf{Dataset} & $\boldsymbol{\beta=0.001}$ & $\boldsymbol{\beta=0.01}$ & $\boldsymbol{\beta=0.1}$ \\
      \midrule
      Cornell (Heter.) & 55.71\% & \textbf{63.71\%} & 52.10\% \\
      Texas (Heter.)   & 50.70\% & \textbf{67.83\%} & 58.17\% \\
      \bottomrule
    \end{tabular}%
  }
\end{table}

\textbf{Transport-Cost Coefficient ($\boldsymbol{\beta}$):} The coefficient $\beta$ balances task supervision ($\mathcal{L}_{cls}$) and the transport cost ($\mathcal{L}_{OT}$). Among the tested values, $\beta=0.01$ gives the highest mean 5-shot accuracy on both datasets. Relative to $\beta=0.001$, it increases accuracy by $17.13$ points on Texas and $8.0$ points on Cornell; increasing the coefficient to $0.1$ reduces accuracy on both. These results indicate that the transport-cost coefficient requires calibration in the evaluated few-shot settings.

\begin{table}[ht]
  \centering
  \caption{Sensitivity to $k$-NN Neighborhood Size $k$ (GraphMAE)}
  \label{tab:sens_k}
  \resizebox{0.9\columnwidth}{!}{
    \begin{tabular}{l cccc}
      \toprule
      \textbf{Setting} & $\boldsymbol{k=10}$ & $\boldsymbol{k=50}$ & $\boldsymbol{k=100}$ & $\boldsymbol{k=200}$ \\
      \midrule
      Cora (1-Shot)      & \textbf{50.32\%} & 49.55\% & 44.66\% & 40.49\% \\
      Wisconsin (5-Shot) & 54.74\% & 74.16\% & \textbf{78.50\%} & 69.77\% \\
      \bottomrule
    \end{tabular}%
  }
\end{table}

\textbf{$k$-NN Neighborhood Size ($\boldsymbol{k}$):} The parameter $k$ determines the size of the discrete $k$-NN neighborhood used by the topology augmentation component. Table~\ref{tab:sens_k} shows different selected values across the two evaluated graph regimes. On Cora ($\text{Homophily}=0.81$), the highest mean occurs at $k=10$, and accuracy decreases by $9.83$ points at $k=200$. On Wisconsin ($\text{Homophily}=0.20$), accuracy rises from $54.74\%$ at $k=10$ to $78.50\%$ at $k=100$, then decreases to $69.77\%$ at $k=200$. The selected $k$ is therefore graph dependent in the reported settings.

\subsection{Ablation Study on Core Components}
\label{subsec:ablation}

We evaluate five ablation variants across shot settings. \textbf{\textit{w/o SA}} removes GCN-style structural context pre-aggregation; \textbf{\textit{w/o OT}} replaces Sinkhorn routing with row-wise softmax without a prompt-side marginal constraint; \textbf{\textit{w/o FA}} disables the scaled prompt-induced feature update; \textbf{\textit{w/o kNN}} uses only native edges; and \textbf{\textit{w/o Reg}} sets $\beta=0$, removing the transport-cost regularizer while retaining the routing formulation.

\begin{figure}[!t]
    \centering
    \includegraphics[width=0.8\columnwidth]{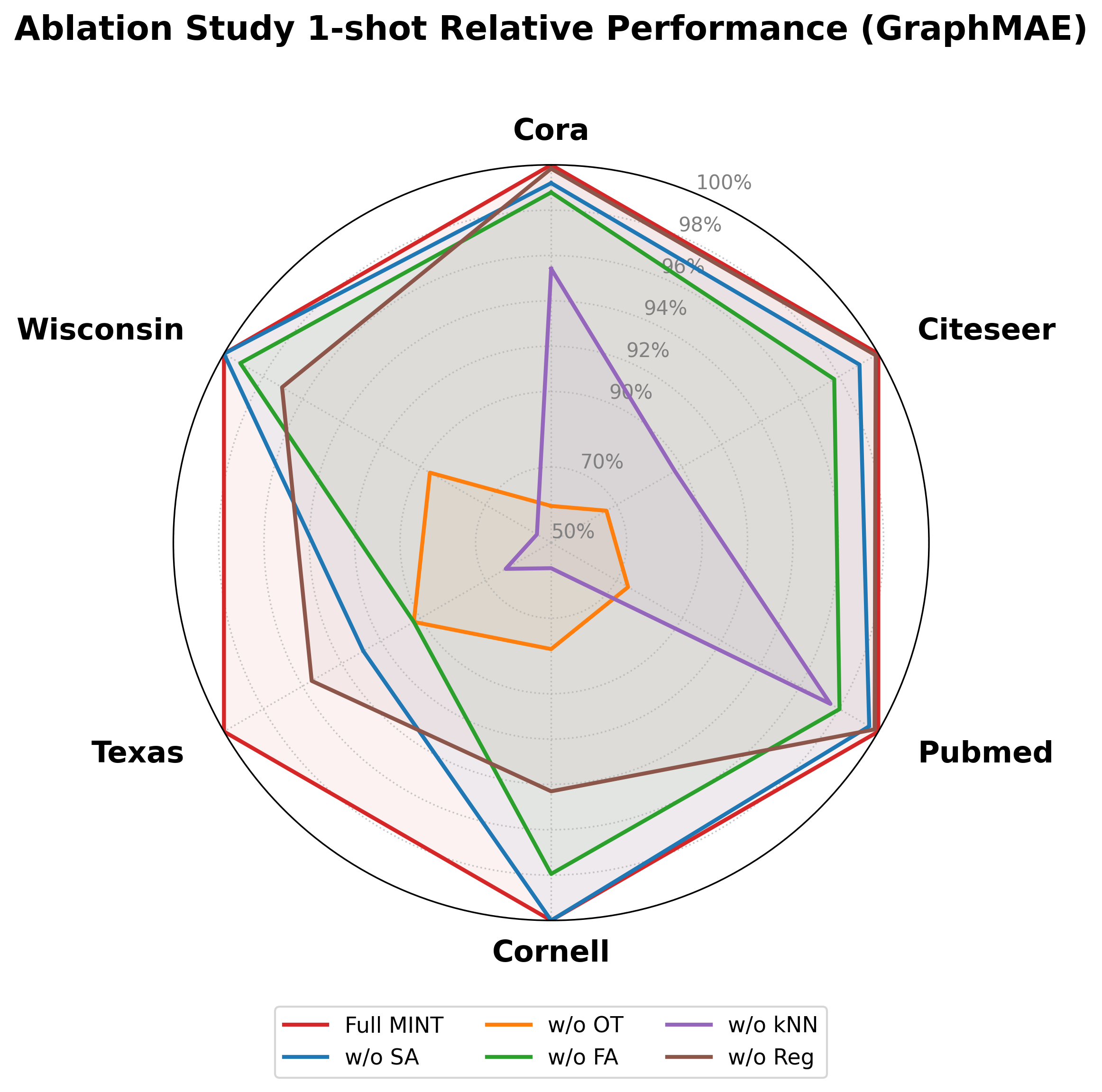} 
    \caption{Relative performance retention of MINT variants under 1-shot classification (GraphMAE).}
    \label{fig:ablation}
\end{figure}

The radar chart in Fig.~\ref{fig:ablation} summarizes representative performance retention across homophilic and heterophilic datasets. The complete ablation matrix across GraphMAE, GraphMAE2, and MaskGAE under 1-, 3-, and 5-shot settings is reported in Appendix F-J.

\subsubsection{Effect of OT Routing and Transport-Cost Regularization}

\begin{table}[ht]
  \centering
  \caption{Ablation of the OT Engine and Transport-Cost Regularizer on GraphMAE. Subscripts denote absolute performance degradation.}
  \label{tab:abl_ot}
  \resizebox{0.95\columnwidth}{!}{
    \begin{tabular}{l | cc | cc}
      \toprule
      \multirow{2}{*}{\textbf{Model}} & \multicolumn{2}{c|}{\textbf{Cora} (Homophily: 0.81)} & \multicolumn{2}{c}{\textbf{Texas} (Homophily: 0.11)} \\
      \cmidrule(lr){2-3} \cmidrule(lr){4-5}
      & 1-shot & 5-shot & 1-shot & 5-shot \\
      \midrule
      \textbf{Full MINT} & 53.46\% & 73.32\% & 53.39\% & 64.35\% \\
      \midrule
      \textit{w/o OT}    & 31.92\%$_{\downarrow \mathbf{21.54\%}}$ & 54.79\%$_{\downarrow 18.53\%}$ & 48.10\%$_{\downarrow 5.29\%}$  & 56.12\%$_{\downarrow 8.23\%}$ \\
      \textit{w/o Reg}   & 53.37\%$_{\downarrow 0.09\%}$ & 73.09\%$_{\downarrow 0.23\%}$ & 50.87\%$_{\downarrow 2.52\%}$ & 57.01\%$_{\downarrow \mathbf{7.34\%}}$ \\
      \bottomrule
    \end{tabular}%
  }
\end{table}

Representative ablations show regime-dependent roles. Under GraphMAE 1-shot, replacing OT with row-wise softmax reduces accuracy from $53.46\%$ to $31.92\%$ on Cora, from $52.08\%$ to $34.84\%$ on CiteSeer, and from $58.18\%$ to $42.77\%$ on PubMed, corresponding to decreases of $21.54$, $17.24$, and $15.41$ percentage points. This end-to-end ablation captures the accumulated effect of fixed-marginal routing within the full adaptation pipeline. Table~\ref{tab:abl_ot} additionally shows that removing the transport-cost regularizer has little effect on Cora but lowers 5-shot Texas accuracy by $7.34$ percentage points.

\subsubsection{The Graph-Dependent Role of Topology Augmentation}

\begin{table}[ht]
  \centering
  \caption{Ablation of the $k$-NN Augmentation Branch (\textit{w/o kNN}) on GraphMAE. }
  \label{tab:abl_knn}
  \resizebox{0.95\columnwidth}{!}{
    \begin{tabular}{l | cc | cc}
      \toprule
      \multirow{2}{*}{\textbf{Model}} & \multicolumn{2}{c|}{\textbf{Cora} (Homophily: 0.81)} & \multicolumn{2}{c}{\textbf{Wisconsin} (Homophily: 0.20)} \\
      \cmidrule(lr){2-3} \cmidrule(lr){4-5}
      & 1-shot & 5-shot & 1-shot & 5-shot \\
      \midrule
      \textbf{Full MINT} & 53.46\% & 73.32\% & 66.22\% & 80.04\% \\
      \midrule
      \textit{w/o kNN}   & 51.01\%$_{\downarrow 2.45\%}$ & \textbf{73.52\%}$_{\uparrow 0.20\%}$ & 35.99\%$_{\downarrow \mathbf{30.23\%}}$ & 40.19\%$_{\downarrow \mathbf{39.85\%}}$ \\
      \bottomrule
    \end{tabular}%
  }
\end{table}

Under GraphMAE 1-shot, removing the $k$-NN augmentation branch lowers accuracy by $21.93$, $19.37$, and $30.23$ points on Cornell, Texas, and Wisconsin, respectively. This effect is not universal: Table~\ref{tab:abl_knn} shows little change on Cora and a $0.20$-point improvement in the 5-shot setting. Thus, the $k$-NN augmentation branch has a distinct, graph-dependent role.

Together, the ablations indicate complementary routing and topology effects, with relative importance varying by backbone, dataset, and shot setting. The complete matrix is reported in Appendix F-J.

\subsection{Large-Scale Validation}
\label{subsec:large_scale_validation}

On OGBN-Arxiv (169,343 nodes and 1,166,243 raw directed edges), a 30-seed paired evaluation shows that measure-constrained routing remains feasible on a substantially larger graph. In this topology-compatible MINT evaluation, fixed-marginal Sinkhorn keeps allocation Gini below $10^{-7}$ and the maximum prompt shares near $0.100$ for $M=10$, while its accuracy differs from matched softmax by at most $0.12$ percentage points on the tested original and augmented graph variants. The largest mean peak reserved memory is $11.066$ GiB on a 24-GiB RTX 4090, with $0.201$--$0.234$s per epoch. Static topology augmentation lowers accuracy by approximately $1.5$ percentage points, providing a graph regime in which topology augmentation is not beneficial. Complete accuracy, resource, and preprocessing measurements are reported in Appendix F-K.

\section{Conclusion}
\label{sec:conclusion}

We studied graph-wide utilization of a finite shared prompt bank as a design dimension of graph prompting. MINT couples fixed-marginal node--prompt allocation with local compatibility and a separate $k$-NN augmentation branch. The variance decomposition characterizes the prompt-induced feature update, and the conditional forward analysis bounds the corresponding frozen-encoder deviation under single-prompt concentration. Experiments show competitive few-shot adaptation, measurable end-to-end effects of fixed-marginal routing in citation-network settings, and complementary, graph-dependent effects of topology augmentation. These results motivate considering graph-wide prompt utilization alongside local relevance as prompt banks and instance-conditioned routing become richer.


\bibliographystyle{unsrt}
\bibliography{references}

\clearpage
\appendices

\section{Local Routing and Graph-Wide Prompt Allocation}
\label{app:theoretical_diagnosis}

The main text analyzes graph-wide utilization of a shared prompt bank. For a row-normalized node-to-prompt routing matrix $R$, local normalization ensures $R\mathbf{1}_M=\mathbf{1}_N$, while the aggregate assignment masses $\frac{1}{N}\mathbf{1}_N^\top R$ remain free unless column constraints are imposed. Continuous attention can therefore be locally valid and globally concentrated. Hard routing raises a separate optimization consideration: when an argmax or threshold lies on the trainable path, its selected index is locally constant away from decision boundaries. Representation and accuracy effects are evaluated through the complete frozen graph neural network (GNN) forward path.

\subsection{Continuous Local Routing}
\label{app:soft_routing_proof}
For a softmax mixture $w_i=\operatorname{softmax}(s_i)$, the Jacobian is
\begin{equation}
    \frac{\partial w_{ij}}{\partial s_{ik}} = w_{ij} (\delta_{jk} - w_{ik}),
\end{equation}
so saturated local mixtures can have small gradients. More importantly for this work, row-wise softmax leaves $\sum_i w_{ij}$ unconstrained. Concentration is therefore possible; the extent of concentration and its downstream effects are determined empirically. The applicable forward-level consequence under observed concentration is the prompt-update bound proved in Appendix~\ref{app:proof_measure_conservation}.

\subsection{Discrete Decisions}
\label{app:hard_routing_proof}
Argmax and threshold operators are non-differentiable at their decision boundaries and locally constant elsewhere. This can obstruct direct gradient-based training when no surrogate or alternate pathway is used. Implementations may instead use detached graph construction, straight-through estimators, or other trainable paths, so the relevant conclusion is specific to the complete optimization pathway.

\subsection{Metric, Global, and Node-Specific Designs}
\label{app:metric_routing_proof}
Nearest-neighbor graph construction, task-token logits, a single global additive vector, and node-specific prompt generation are different computational objects. Their graph-wide use can be measured according to their respective forward paths. Adding a common vector leaves Euclidean pairwise differences unchanged, although it can change directions, nonlinear activations, and subsequent GNN outputs. For Measure-INtegrity Transport (MINT), when multiple shared prompts define the prompt-induced feature update $\Delta X_P=RP$, the prescribed prompt-side marginal directly specifies aggregate assignment mass across the shared prompt bank.

\section{Interpretation of Feature Dirichlet Energy}
\label{app:dirichlet_energy_regimes}

For a representation $Q\in\mathbb{R}^{N\times d}$ and the normalized Laplacian $L$ of the graph on which it is evaluated,
\begin{equation}
    E_D(Q) = \frac{1}{N} \operatorname{tr}(Q^\top LQ).
\end{equation}
Writing $L=U\Lambda U^\top$ gives the exact spectral decomposition
\begin{equation}
    E_D(Q)=\frac{1}{N}\sum_{k=1}^{N}\lambda_k\|u_k^\top Q\|_2^2.
    \label{eq:energy_variance_relation}
\end{equation}
Energy therefore depends jointly on graph frequencies, feature scale, and layer geometry. Numerical values such as $10^{-5}$, $10^{-1}$, or $10^1$ describe runs under a fixed evaluation protocol rather than architecture-independent thresholds. A routing Gini determines $E_D$ only with additional assumptions on $P$, $X$, $G'$, and $F_\theta$. Accordingly, the main text uses the variance of $\Delta X_P=RP$ for the allocation mechanism and treats final-layer energy and accuracy as downstream measurements.

The controlled prompt-attractiveness interventions make this separation explicit. Severe concentration drives prompt-update variance nearly to zero, whereas the frozen-state accuracy effect differs substantially across Cora, Texas, and Wisconsin; retraining can further compensate. In unperturbed training runs, the association between Gini and prompt-update variance varies as the prompt matrix and routing co-adapt. Accordingly, we retain $E_D$ as a within-protocol diagnostic while locating the exact allocation statement at the prompt-update level.

\section{Prompt Allocation and Forward Stability}
\label{app:proof_measure_conservation}

\subsection{Normalized Prompt Allocation}
Let $w_{ij}\ge0$ be generic node-to-prompt assignment weights with $\sum_{i,j}w_{ij}>0$. Define the raw prompt masses and their normalized profile by
\begin{equation}
 a_j:=\sum_{i=1}^Nw_{ij},
 \qquad
 \rho_j=\frac{a_j}{\sum_{\ell=1}^M a_\ell},
 \qquad \sum_j\rho_j=1.
\end{equation}
If $w_{ij}=R_{ij}$ for row-stochastic $R$, then the total mass is $N$ and
\begin{equation}
 \rho_j=\frac{1}{N}\sum_iR_{ij}.
\end{equation}
For MINT, $R=N\Pi^\star$ and $\Pi^\star\in\mathcal{U}(\mu,\nu)$, so
\begin{equation}
 \rho_j=\frac{1}{N}\sum_iR_{ij}
 =\sum_i\pi^\star_{ij}=\nu_j.
 \label{eq:app_normalized_capacity}
\end{equation}
Gini is invariant to common positive scaling, so $R$ and $\Pi^\star$ induce the same normalized profile. For $S\subseteq\{1,\ldots,M\}$, let $\rho(S)=\sum_{j\in S}\rho_j$. Thus MINT gives $\rho(S)=\sum_{j\in S}\nu_j$, and uniform $\nu$ gives $\rho(S)=|S|/M$. This formulation-level identity specifies graph-wide prompt utilization; the implemented 20-step Sinkhorn solver realizes the prescribed marginals to numerical tolerance.

\subsection{Variance Decomposition}
Let $R\ge0$ satisfy $R\mathbf{1}_M=\mathbf{1}_N$, and let
$\Delta x_{P,i}=\sum_jR_{ij}p_j$. The mean prompt-induced feature update is
\begin{equation}
\begin{aligned}
 \overline{\Delta x}_P
 &=\frac{1}{N}\sum_i\Delta x_{P,i}\\
 &=\sum_j\left(\frac{1}{N}\sum_iR_{ij}\right)p_j
 =\sum_j\rho_jp_j.
\end{aligned}
 \label{eq:app_mean_identity}
\end{equation}
Define
\begin{equation}
 V_P(\rho):=\sum_j\rho_j
 \|p_j-\overline{\Delta x}_P\|_2^2.
\end{equation}

\begin{theoremfinal}[Variance Decomposition]
\begin{equation}
 V_P(\rho)=V_{\Delta X_P}
 +\frac{1}{N}\sum_i\sum_jR_{ij}
 \|p_j-\Delta x_{P,i}\|_2^2.
 \label{eq:app_variance_decomposition}
\end{equation}
Consequently, $V_{\Delta X_P}\le V_P(\rho)$.
\end{theoremfinal}
\begin{proof}
For a fixed node $i$, write
$p_j-\overline{\Delta x}_P=(p_j-\Delta x_{P,i})+(\Delta x_{P,i}-\overline{\Delta x}_P)$. Expanding the square gives
\begin{equation}
\begin{aligned}
 &\sum_jR_{ij}\|p_j-\overline{\Delta x}_P\|_2^2\\
 &=\sum_jR_{ij}\|p_j-\Delta x_{P,i}\|_2^2
 +\|\Delta x_{P,i}-\overline{\Delta x}_P\|_2^2,
\end{aligned}
\end{equation}
because the cross term vanishes exactly:
\begin{equation}
 \sum_jR_{ij}(p_j-\Delta x_{P,i})
 =\sum_jR_{ij}p_j-\Delta x_{P,i}\sum_jR_{ij}=0.
\end{equation}
Averaging over $i$, the left-hand side becomes $V_P(\rho)$ by~\eqref{eq:app_mean_identity}, the second term becomes $V_{\Delta X_P}$, and the first term is the stated nonnegative within-node mixing dispersion. This proves the identity and inequality.
\end{proof}

\subsection{Subset and Single-Prompt Bounds}
Let $S\subseteq\{1,\ldots,M\}$ be nonempty, let $P_S=\{p_j:j\in S\}$, and define
$\operatorname{diam}(P_S)=\max_{j,k\in S}\|p_j-p_k\|_2$ and
$\operatorname{diam}(P)=\max_{j,k}\|p_j-p_k\|_2$.

\begin{corollaryfinal}[Subset Concentration]
\begin{equation}
 V_{\Delta X_P}
 \le\rho(S)\operatorname{diam}(P_S)^2
 +[1-\rho(S)]\operatorname{diam}(P)^2.
 \label{eq:app_subset_variance_bound}
\end{equation}
\end{corollaryfinal}
\begin{proof}
The theorem gives $V_{\Delta X_P}\le V_P(\rho)$. Since $\overline{\Delta x}_P=\sum_j\rho_jp_j$ minimizes weighted squared deviation, for any $p_s\in P_S$,
\begin{equation}
\begin{aligned}
 V_P(\rho)
 &\le\sum_j\rho_j\|p_j-p_s\|_2^2\\
 &\le\rho(S)\operatorname{diam}(P_S)^2
 +[1-\rho(S)]\operatorname{diam}(P)^2.
\end{aligned}
\end{equation}
\end{proof}

For $S=\{j^\star\}$, $\operatorname{diam}(P_S)=0$, and therefore
\begin{equation}
 V_{\Delta X_P}\le(1-\rho_{j^\star})\operatorname{diam}(P)^2.
\end{equation}
Under MINT, $\rho(S)=\sum_{j\in S}\nu_j$, and uniform $\nu$ gives $\rho(S)=|S|/M$.

\subsection{Additive Cost Potentials}
\begin{remarkfinal}[Additive Cost Potentials]
For $C'_{ij}=C_{ij}+\xi_i+\zeta_j$ and any $\Pi\in\mathcal{U}(\mu,\nu)$,
\begin{equation}
 \langle\Pi,C'\rangle
 =\langle\Pi,C\rangle+\sum_i\mu_i\xi_i+\sum_j\nu_j\zeta_j.
\end{equation}
The added term is coupling-independent. Thus additive row and column potentials leave the fixed-marginal optimal transport (OT) optimizer unchanged; this classical property supports the prompt-attractiveness-bias interpretation in the main text.
\end{remarkfinal}

\subsection{Forward Stability}
For any prompt $j^\star$,
$\Delta x_{P,i}-p_{j^\star}=\sum_jR_{ij}(p_j-p_{j^\star})$. Convexity of the squared norm gives
\begin{equation}
 \|\Delta x_{P,i}-p_{j^\star}\|_2^2
 \le\sum_jR_{ij}\|p_j-p_{j^\star}\|_2^2.
\end{equation}
After averaging,
\begin{equation}
\begin{aligned}
 \frac{1}{N}\|\Delta X_P-\mathbf{1}_Np_{j^\star}^\top\|_F^2
 &\le\sum_j\rho_j\|p_j-p_{j^\star}\|_2^2\\
 &\le(1-\rho_{j^\star})\operatorname{diam}(P)^2.
\end{aligned}
 \label{eq:app_common_prompt_proximity}
\end{equation}

\begin{propositionfinal}[Forward Stability]
Fix $G'$ and the frozen parameters $\theta$. Suppose $U\mapsto F_\theta(U,G')$ is $L_F$-Lipschitz in Frobenius norm on a convex neighborhood containing $X+\alpha\Delta X_P$ and $X+\alpha\mathbf{1}_Np_{j^\star}^\top$. Then
\begin{equation}
\begin{aligned}
 &\frac{1}{\sqrt N}\big\|
 F_\theta(X+\alpha\Delta X_P,G')\\
 &\quad-F_\theta(X+\alpha\mathbf{1}_Np_{j^\star}^\top,G')\big\|_F\\
 &\le L_F|\alpha|\operatorname{diam}(P)
 \sqrt{1-\rho_{j^\star}}.
\end{aligned}
\end{equation}
\end{propositionfinal}
\begin{proof}
The Lipschitz condition and~\eqref{eq:app_common_prompt_proximity} give
\begin{equation}
\begin{aligned}
 &\frac{1}{\sqrt N}\big\|
 F_\theta(X+\alpha\Delta X_P,G')\\
 &\quad-F_\theta(X+\alpha\mathbf{1}_Np_{j^\star}^\top,G')\big\|_F\\
 &\quad\le\frac{L_F|\alpha|}{\sqrt N}
 \|\Delta X_P-\mathbf{1}_Np_{j^\star}^\top\|_F\\
 &\quad\le L_F|\alpha|\operatorname{diam}(P)
 \sqrt{1-\rho_{j^\star}}.
\end{aligned}
\end{equation}
\end{proof}
Here $F_\theta$ is the frozen encoder representation map; predictions use the separately trained classifier $h_\phi$. The proposition characterizes the frozen-state intervention with the realized weighted propagation graph $G'$ held fixed, while complete training and prediction behavior is evaluated empirically.

\section{Differentiability of the Entropic Transport Computation}
\label{app:proof_differentiable_transport}

For positive marginals and $\epsilon>0$, MINT solves
\begin{equation}
    \Pi^\star = \operatorname{argmin}_{\Pi \in \mathcal{U}(\mu, \nu)} \langle \Pi, C \rangle - \epsilon \mathcal{H}(\Pi),
\label{eq:entropic_ot_proof}
\end{equation}
whose strictly convex primal objective has a unique positive solution of the Sinkhorn form
\begin{equation}
    \Pi^\star_{ij} = u_i \exp\left(-\frac{C_{ij}}{\epsilon}\right) v_j.
\label{eq:primal_solution_proof}
\end{equation}
The scaling variables satisfy
\begin{equation}
    \begin{aligned}
        \mathcal{F}_1(u, v, C) &= \operatorname{diag}(u) \exp(-C/\epsilon) v - \mu = \mathbf{0}, \\
        \mathcal{F}_2(u, v, C) &= \operatorname{diag}(v) \exp(-C^\top/\epsilon) u - \nu = \mathbf{0}.
    \end{aligned}
\end{equation}
After fixing the one-dimensional gauge freedom of the dual potentials, the implicit system is nonsingular under the standard positive-marginal assumptions, so $\Pi^\star$ depends smoothly on finite costs. The Original MINT implementation instead differentiates through a finite numerical map: it sets $K=\exp(-C/\epsilon)$, initializes $u=v=\mathbf{1}$, and performs exactly 20 $u$-then-$v$ updates with $10^{-8}$ denominator safeguards before forming $\widehat\Pi=\operatorname{diag}(u)K\operatorname{diag}(v)$. This computation uses no convergence stopping, log-domain reformulation, final projection, renormalization, or gradient detachment, and realizes the fixed marginals to numerical tolerance.

\section{Differentiation Through Prompt Adaptation}
\label{sec:appendix_optimization}

The prompt-adaptation path combines differentiable transport with the prompt-induced feature update and adapted-feature-dependent topology-augmentation weights. The $k$-nearest-neighbor ($k$-NN) indices are the discrete component and remain fixed after construction from the raw input features.

\subsection{Differentiation Through Entropic Transport}
Each of the $L_{\mathrm{SK}}=20$ stabilized Sinkhorn scaling updates is composed of differentiable tensor operations. Writing the resulting finite coupling as $\widehat\Pi(C)$, the classification loss differentiates through the complete unrolled computation as
\begin{equation}
    \frac{\partial \mathcal{L}_{cls}}{\partial C}
    =\frac{\partial \mathcal{L}_{cls}}{\partial \widehat\Pi}
    \frac{\partial \widehat\Pi(C)}{\partial C}.
\end{equation}

\subsection{Gradient Pathways for the Prompt Matrix}
The prompt matrix receives gradients through three pathways:
\begin{enumerate}
    \item \textbf{Feature-adaptation pathway:} $\mathcal{L}_{cls}$ differentiates through the frozen encoder and $X_{\mathrm{adapted}}=X+\alpha\Delta X_P$, with $\widehat R=N\widehat\Pi$ and $\Delta X_P=\widehat R P$.
    \item \textbf{Transport-cost pathway:} both the finite coupling and normalized cost depend on $P$, and $\mathcal{L}_{OT}$ adds this pathway when $\beta>0$.
    \item \textbf{Topology-weight pathway:} Original MINT recomputes ReLU-cosine scalar weights from $X_{\mathrm{adapted}}$ on the fixed $k$-NN indices; these weights are differentiable almost everywhere through the downstream computation.
\end{enumerate}
The index construction itself is discrete and held fixed.

\section{Additional Experimental Details and Results}
\label{app:Experiments}

\subsection{Experimental Setup and Computational Overhead}
\label{app_Computational_Overhead}

\subsubsection{Detailed Training Protocols}
\label{subsubsec:detailed_training_protocols}

We report the training and evaluation procedures for the evaluated methods. The pretrained-backbone evaluation uses method-specific adaptation implementations; DAPrompt uses the GraphMAE-backbone adaptation path, and GPPT uses its task-token implementation.

\textbf{Optimization and Early Stopping:} MINT jointly trains the prompt parameters and linear classifier with Adam while keeping the pretrained encoder frozen. The default weight decay is $5 \times 10^{-5}$, the maximum is $2000$ epochs, and patience is $100$. The evaluation selector pairs pre-update embeddings with the post-update classifier and stores the corresponding post-update states; the selected prompt and classifier states are restored before a recomputed full-graph test forward. Each fixed hyperparameter candidate is evaluated on the same 30 predetermined splits (effective seeds $43$--$72$) and ranked by mean final test accuracy. All splits are retained without seed selection or removal. GraphMAE and GraphMAE2 use this selector; for MaskGAE, whose original selector state is unavailable, the maximum recorded result is used when necessary. These results therefore follow the fixed repeated-split tuning protocol. The controlled mechanism experiments instead use post-update validation-only selection and evaluate the test set after restoration.

\textbf{Repeated Evaluation Across Predetermined Seeds:} We report mean accuracy and standard deviation over $30$ trials with few-shot split seeds $43$--$72$. Python, NumPy, and PyTorch are initialized once at base seed $42$; prompt/classifier initialization and optimization then draw from the continuing random number generator (RNG) stream.

\textbf{Perturbation and Ablation Implementations:} For the extreme robustness analysis, we dynamically inject structural and feature perturbations during data loading. Given a pre-perturbation \texttt{edge\_index} with $E$ directed edge entries and ratio $p$, \textit{Noise Edge} repeatedly samples ordered node pairs uniformly at random, retains non-self pairs whose ground-truth class labels differ, and appends $\lfloor pE\rfloor$ directed cross-class edge entries. The implementation does not symmetrize or deduplicate the sampled entries and does not explicitly exclude pairs already present in the graph. Ground-truth labels are used only to construct this evaluation perturbation and are not exposed to the adaptation model beyond the standard few-shot supervision. \textit{Edge Drop} uniformly drops native edges while preserving undirected symmetry, and \textit{Feature Mask} applies a binary dropout mask to raw node features before structural context injection.
For the ablation study, we disable specific architectural modules via computational flags within the forward pass. In particular, \texttt{w/o OT} bypasses the Sinkhorn iterations and uses row-wise softmax without a prompt-side marginal constraint.

\subsubsection{Hardware Configuration}
The original experiments used an Intel(R) Xeon(R) Platinum 8358P CPU @ 2.60GHz system with 128 cores, 503 GB RAM, one 80-GB NVIDIA A800-SXM4 GPU, and CUDA 12.2. Controlled mechanism and OGBN-Arxiv experiments used 24-GB NVIDIA RTX 4090 GPUs.

\subsubsection{Computational Footprint Evaluation}
To quantify the computational efficiency of the proposed method, we profile the native execution overhead. The measurements include execution time, peak GPU VRAM usage, and peak CPU RAM usage during 1-shot fine-tuning under the selected configurations, using one profiling run per setting. The detailed footprint across three pretrained backbones and nine datasets is summarized in Table~\ref{tab:computational_overhead}.

\subsection{Cross-Domain Few-Shot Transfer Protocol}
\label{app:cross_domain_protocol}

The Cross-Domain experiment evaluates 1-shot node classification from PubMed (feature width 500) and Actor (feature width 932) to Cora, CiteSeer, Texas, Cornell, and Wisconsin. Each of the seven methods is evaluated for every source--target pair with predetermined target-split seeds 43--72, giving 70 source--method--target cells and 30 trials per cell. All 2,100 trials completed successfully. For a fixed target and seed, every method uses the same target split and a checkpoint pretrained on the specified source domain.

Model selection uses validation accuracy only. A state is retained only after a strict validation-accuracy improvement, so the earliest state is kept when validation scores tie. Training uses patience 100 and a maximum of 2,000 epochs; after training, the validation-best state is restored before one test evaluation.

Feature alignment is deterministic, label-free, non-trainable, and independently defined for each source--target pair. Let $d_{tgt}$ and $d_{src}$ denote the target and source feature widths. If $d_{tgt}>d_{src}$, we fit TruncatedSVD with $n_{components}=\min(d_{src},n_{samples}-1)$ and \texttt{random\_state=42}, followed by right-zero-padding if the resulting component count is smaller than $d_{src}$. If $d_{tgt}<d_{src}$, target features are right-zero-padded directly to $d_{src}$; equal widths use the identity map.

The comparison uses method-specific adaptation procedures. DGI Fine-tune optimizes its graph convolutional network (GCN) encoder and linear head rather than the detached \texttt{DGI.embed()} path; DGI Linear and GraphMAE Linear freeze their encoders and train linear heads. GPPT uses Xavier-initialized class task tokens with dot-product logits. UniPrompt optimizes fixed $k$-NN edge scalars based on cosine similarity and its head. DAPrompt uses the GraphMAE path with two structure prompts, class semantic prompts, and a head. MINT freezes the source-specific GraphMAE checkpoint and optimizes its prompt matrix, adaptation scalars, and classifier while recomputing fixed-marginal routing and topology-augmentation weights under the reported cross-domain configuration.

\subsection{Mechanism Analysis of Prompt Allocation}
\label{app:mechanism_analysis}

The mechanism analysis combines unperturbed routing behavior, a controlled single-prompt attractiveness intervention, frozen-forward evaluation, and retraining. Across 600 training runs under the bias intervention and 600 frozen-forward interventions, severe prompt-attractiveness bias raises the dominant prompt share to approximately $0.997$ and drives $V_{\Delta X_P}$ close to zero. In the frozen intervention, accuracy changes by $-6.26$, $-30.48$, and $-21.95$ percentage points on Cora, Texas, and Wisconsin, respectively. Sinkhorn routing remains near its prescribed maximum share of $0.100$ and is invariant to the column potential up to numerical precision. Retraining compensates strongly on Cora and Wisconsin, whereas Texas retains an approximately $-5.17$ point effect. Across unperturbed training runs, $V_{\Delta X_P}$ remains consistently associated with final-layer representation geometry, whereas its association with allocation Gini varies across settings. Together, allocation Gini and maximum share characterize aggregate use, $V_{\Delta X_P}$ measures the prompt-induced feature update, and final-layer geometry and prediction changes characterize downstream effects.

To isolate the routing effect under identical graph propagation, we additionally compare matched-softmax and fixed-marginal Sinkhorn routing on the same frozen weighted graph over 30 predetermined seeds for Cora 1-shot, Texas 5-shot, and Wisconsin 1- and 5-shot. Fixed-marginal routing reduces Prompt Allocation Gini by $0.0406$, $0.2515$, $0.1754$, and $0.1038$, respectively, whereas the corresponding test-accuracy differences are only $+0.12$, $+0.14$, $+0.05$, and $+0.13$ percentage points; all paired 95\% confidence intervals for accuracy include zero. These controls show the expected graph-wide allocation behavior while indicating that the predictive contribution of routing under fixed topology is dataset dependent.

\subsection{Description of Datasets}
\label{subsec:datasets}
We summarize the nine real-world datasets used in our experiments.

\begin{itemize}
    \item \textit{Cora}~\cite{yang2016cora}, \textit{CiteSeer}~\cite{yang2016cora}, and \textit{PubMed}~\cite{yang2016cora} are citation datasets in which nodes represent papers and edges represent citation relationships. Each feature dimension corresponds to a word, and labels denote paper categories.
    
    \item \textit{Cornell}~\cite{pei2020cornell}, \textit{Texas}~\cite{pei2020cornell}, and \textit{Wisconsin}~\cite{pei2020cornell} are sub-datasets of WebKB~\cite{garciaplaza2016webkb}, which is a webpage dataset collected from Carnegie Mellon University. Nodes represent web pages, and edges represent hyperlinks between web pages.
    
    \item \textit{Chameleon}~\cite{pei2020cornell} and \textit{Squirrel}~\cite{pei2020cornell} are page-to-page networks on specific topics collected from Wikipedia~\cite{rozemberczki2021wikipedia}; nodes represent web pages, and edges represent links between them. The prediction target is average monthly webpage traffic grouped into five categories.
    
    \item \textit{Actor}~\cite{pei2020cornell} is the actor-only induced subgraph of the film-director-actor-writer network. Each node corresponds to an actor, and an edge denotes co-occurrence on the same Wikipedia page. Node features correspond to keywords on the Wikipedia pages. The task is to classify nodes into five categories based on terms from each actor's Wikipedia page.
\end{itemize}

\subsection{Description of Baselines}

\subsubsection{Graph Pretraining Models}
\paragraph{GraphMAE~\cite{hou2022graphmae}}GraphMAE pretrains a graph encoder by masking node features and reconstructing them with an encoder--decoder architecture. In the prompt-adaptation experiments, its pretrained encoder remains frozen while each prompting method optimizes its designated prompt and prediction parameters.

\begin{table*}[htbp]
  \centering
  \caption{Mean accuracy $\pm$ standard deviation under 3-shot node classification. The best results across all methods are highlighted in \textbf{bold}, and the second-best results are \underline{underlined}.}
\label{tab:3shot_main_results}
  \resizebox{\textwidth}{!}{
    \begin{tabular}{c l ccc cccccc}
      \toprule
      \multirow{2}{*}{\textbf{Pretrain}} & \multirow{2}{*}{\textbf{Methods}} & \multicolumn{3}{c}{\textbf{Homophilic Graphs}} & \multicolumn{6}{c}{\textbf{Heterophilic Graphs}} \\
      \cmidrule(lr){3-5} \cmidrule(lr){6-11}
      & & \textbf{Cora} & \textbf{CiteSeer} & \textbf{PubMed} & \textbf{Cornell} & \textbf{Texas} & \textbf{Wisconsin} & \textbf{Chameleon} & \textbf{Squirrel} & \textbf{Actor} \\
      \midrule

      \multirow{11}{*}{\textbf{GraphMAE}} 
      & Fine-tune     & 62.54$\pm$4.93 & 56.86$\pm$3.93 & \textbf{66.65$\pm$6.70} & 48.33$\pm$9.93 & 44.33$\pm$14.29 & 46.09$\pm$6.64 & 26.53$\pm$2.27 & 21.48$\pm$1.16 & 21.12$\pm$2.24 \\
      & Linear Probe  & \underline{65.83$\pm$5.20} & 57.51$\pm$5.20 & 65.56$\pm$5.39 & 28.72$\pm$8.18 & 41.42$\pm$20.51 & 38.61$\pm$6.00 & 27.11$\pm$3.83 & 21.13$\pm$1.03 & 20.11$\pm$2.00 \\
      & GPF           & 63.88$\pm$8.44 & 56.47$\pm$5.96 & 64.62$\pm$5.94 & 29.91$\pm$11.83 & 41.06$\pm$19.81 & 36.87$\pm$9.54 & 26.75$\pm$3.40 & 20.94$\pm$0.67 & 20.75$\pm$2.00 \\
      & GPF+          & 60.92$\pm$9.55 & 56.89$\pm$6.24 & 64.56$\pm$5.52 & 31.49$\pm$14.34 & 41.42$\pm$20.81 & 33.70$\pm$6.27 & 26.22$\pm$3.44 & 20.94$\pm$0.74 & 21.34$\pm$2.77 \\
      & GPPT          & 64.91$\pm$4.71 & 57.21$\pm$4.68 & 61.79$\pm$6.11 & 27.02$\pm$9.70 & 38.75$\pm$22.34 & 34.35$\pm$7.53 & 26.78$\pm$3.25 & 20.79$\pm$0.69 & 21.64$\pm$2.72 \\
      & GraphPrompt   & 65.43$\pm$4.85 & 56.98$\pm$5.88 & 64.42$\pm$5.71 & 29.05$\pm$11.62 & 48.94$\pm$19.01 & 39.50$\pm$10.56 & 26.98$\pm$3.41 & 21.10$\pm$0.87 & 20.68$\pm$2.63 \\
      & EdgePrompt    & 62.52$\pm$9.81 & 57.20$\pm$6.58 & 65.23$\pm$5.19 & 30.00$\pm$11.39 & 40.50$\pm$27.17 & 34.28$\pm$14.29 & 26.36$\pm$3.67 & 20.92$\pm$0.80 & 21.11$\pm$2.30 \\
      & EdgePrompt+   & 58.41$\pm$8.66 & 56.26$\pm$6.61 & 65.25$\pm$5.14 & 32.05$\pm$13.01 & 42.36$\pm$24.86 & 34.02$\pm$7.95 & 26.88$\pm$3.69 & 20.91$\pm$0.80 & 21.24$\pm$2.69 \\
      & All-in-one    & 55.95$\pm$7.04 & 33.00$\pm$6.34 & 47.32$\pm$7.98 & 36.79$\pm$18.39 & 41.03$\pm$17.18 & 26.33$\pm$11.56 & 22.81$\pm$3.33 & \textbf{24.48$\pm$3.14} & 21.20$\pm$1.53 \\
      & HS-GPPT        & 63.00$\pm$5.32 & 50.07$\pm$6.11 & 59.56$\pm$5.55 & 32.80$\pm$9.22 & 40.67$\pm$19.22 & 38.19$\pm$10.93 & 24.77$\pm$2.75 & 22.70$\pm$2.79 & 20.90$\pm$2.33 \\
      & ProNoG        & 60.41$\pm$5.81 & 47.50$\pm$6.44 & 61.38$\pm$7.68 & 35.60$\pm$12.43 & 48.67$\pm$17.55 & 37.37$\pm$11.51 & \textbf{28.47$\pm$3.69} & 21.00$\pm$1.38 & 21.06$\pm$3.50 \\
      & UniPrompt     & 62.50$\pm$3.55 & \underline{60.95$\pm$5.42} & 61.16$\pm$6.00 & \underline{60.30$\pm$10.09} & \underline{59.31$\pm$11.63} & \underline{69.06$\pm$6.66} & 26.53$\pm$2.73 & 21.62$\pm$1.36 & \underline{22.42$\pm$2.59} \\
      & DAPrompt      & 44.28$\pm$4.79 & 37.81$\pm$4.71 & 53.60$\pm$5.60 & 32.14$\pm$10.20 & 41.50$\pm$14.15 & 39.54$\pm$6.73 & 26.37$\pm$3.28 & 21.44$\pm$1.40 & 21.00$\pm$1.79 \\
      \cmidrule{2-11}
      & \cellcolor{gray!15}\textbf{MINT (Ours)} & \cellcolor{gray!15}\textbf{66.90$\pm$4.18} & \cellcolor{gray!15}\textbf{61.54$\pm$4.87} & \cellcolor{gray!15}\underline{66.10$\pm$5.45} & \cellcolor{gray!15}\textbf{60.63$\pm$6.72} & \cellcolor{gray!15}\textbf{60.08$\pm$13.04} & \cellcolor{gray!15}\textbf{76.56$\pm$9.40} & \cellcolor{gray!15}\underline{27.57$\pm$2.56} & \cellcolor{gray!15}\underline{22.77$\pm$1.78} & \cellcolor{gray!15}\textbf{23.09$\pm$2.52} \\
      \midrule

      \multirow{10}{*}{\textbf{GraphMAE2}} 
      & Fine-tune     & \textbf{61.18$\pm$5.90} & 53.04$\pm$5.72 & \textbf{64.88$\pm$5.27} & 45.39$\pm$11.11 & 52.39$\pm$11.84 & 48.54$\pm$7.22 & 27.39$\pm$2.27 & 22.10$\pm$1.91 & 20.88$\pm$1.91 \\
      & Linear Probe  & 56.36$\pm$6.61 & 52.09$\pm$6.19 & 61.14$\pm$6.43 & 34.70$\pm$8.12 & 38.78$\pm$8.35 & 39.33$\pm$5.77 & 28.41$\pm$3.11 & 21.56$\pm$1.61 & 21.67$\pm$1.93 \\
      & GPF           & 58.88$\pm$7.36 & 47.20$\pm$7.70 & 61.09$\pm$9.50 & 35.24$\pm$13.83 & 47.78$\pm$18.19 & 38.54$\pm$16.01 & 26.62$\pm$3.87 & 21.06$\pm$0.91 & 21.48$\pm$1.76 \\
      & GPF+          & 56.46$\pm$6.41 & 50.10$\pm$5.97 & 62.55$\pm$5.11 & 36.10$\pm$13.64 & 49.75$\pm$18.13 & 40.09$\pm$10.29 & 27.10$\pm$3.85 & 21.10$\pm$1.11 & 21.26$\pm$2.43 \\
      & GPPT          & 57.38$\pm$4.71 & 51.85$\pm$5.50 & 60.54$\pm$5.75 & 34.70$\pm$8.66 & 36.83$\pm$9.63 & 35.09$\pm$5.44 & 27.27$\pm$3.27 & 21.28$\pm$1.56 & 21.37$\pm$1.53 \\
      & GraphPrompt   & 56.57$\pm$6.62 & 52.93$\pm$5.98 & 61.21$\pm$7.37 & 36.07$\pm$8.98 & 39.92$\pm$7.51 & 38.22$\pm$5.72 & \textbf{28.67$\pm$2.87} & 21.71$\pm$1.58 & 21.80$\pm$2.33 \\
      & EdgePrompt    & 53.77$\pm$11.55 & 48.77$\pm$5.57 & 60.90$\pm$8.64 & 38.66$\pm$10.67 & 42.53$\pm$22.54 & 41.74$\pm$8.37 & 27.29$\pm$3.13 & 21.47$\pm$1.42 & 21.98$\pm$2.53 \\
      & EdgePrompt+   & 57.59$\pm$5.17 & 48.81$\pm$7.28 & 61.94$\pm$6.63 & 38.04$\pm$10.40 & 40.83$\pm$21.55 & 41.65$\pm$9.83 & 26.98$\pm$4.05 & 21.77$\pm$1.37 & \underline{22.12$\pm$1.91} \\
      & All-in-one    & 52.87$\pm$4.29 & 49.36$\pm$5.82 & 57.81$\pm$6.92 & 41.88$\pm$13.66 & 41.58$\pm$20.42 & 35.39$\pm$9.31 & 27.79$\pm$3.56 & \textbf{22.80$\pm$2.05} & 21.80$\pm$2.57 \\
      & HS-GPPT        & 56.32$\pm$5.30 & 48.32$\pm$6.18 & 61.45$\pm$5.71 & 32.02$\pm$13.01 & 46.11$\pm$17.82 & 34.28$\pm$8.85 & 27.83$\pm$3.01 & 22.08$\pm$2.23 & 21.64$\pm$2.58 \\
      & ProNoG        & 58.76$\pm$5.65 & 45.07$\pm$6.42 & 60.60$\pm$10.53 & 35.51$\pm$12.40 & 48.72$\pm$15.45 & 38.80$\pm$11.26 & 28.23$\pm$2.94 & 21.97$\pm$1.96 & 20.84$\pm$3.73 \\
      & UniPrompt     & 58.36$\pm$5.40 & \underline{56.85$\pm$5.70} & 59.05$\pm$6.82 & \underline{57.80$\pm$8.99} & \underline{57.19$\pm$15.15} & \underline{55.17$\pm$12.04} & 26.36$\pm$2.89 & 21.20$\pm$1.42 & 21.05$\pm$2.29 \\
      & DAPrompt      & 56.00$\pm$7.99 & 37.55$\pm$5.23 & 58.82$\pm$6.78 & 36.43$\pm$9.35 & 37.69$\pm$12.72 & 37.13$\pm$7.63 & 26.13$\pm$3.47 & 21.63$\pm$1.42 & 21.89$\pm$1.97 \\
      \cmidrule{2-11}
      & \cellcolor{gray!15}\textbf{MINT (Ours)} & \cellcolor{gray!15}\underline{60.15$\pm$5.31} & \cellcolor{gray!15}\textbf{57.80$\pm$5.55} & \cellcolor{gray!15}\underline{62.82$\pm$7.14} & \cellcolor{gray!15}\textbf{60.00$\pm$8.89} & \cellcolor{gray!15}\textbf{62.33$\pm$15.52} & \cellcolor{gray!15}\textbf{66.17$\pm$12.22} & \cellcolor{gray!15}\underline{28.54$\pm$2.33} & \cellcolor{gray!15}\underline{22.46$\pm$1.70} & \cellcolor{gray!15}\textbf{23.63$\pm$2.53} \\
      \midrule
      \multirow{10}{*}{\textbf{MaskGAE}} 
      & Fine-tune     & 66.27$\pm$5.08 & 53.84$\pm$3.83 & 63.92$\pm$6.30 & 40.30$\pm$11.02 & 48.25$\pm$12.07 & 36.98$\pm$9.62 & 33.45$\pm$3.85 & 24.77$\pm$1.84 & 21.45$\pm$2.58 \\
      & Linear Probe  & 65.90$\pm$5.52 & 52.38$\pm$4.35 & 63.80$\pm$6.36 & 40.03$\pm$11.49 & 46.00$\pm$13.32 & 36.41$\pm$10.28 & 33.39$\pm$3.92 & 24.73$\pm$1.83 & 21.40$\pm$2.60 \\
      & GPF           & 57.59$\pm$7.51 & 44.37$\pm$6.57 & 58.90$\pm$5.64 & 46.49$\pm$10.19 & 61.39$\pm$8.93 & 39.44$\pm$9.57 & \underline{33.72$\pm$4.13} & \underline{28.76$\pm$3.38} & 21.11$\pm$2.77 \\
      & GPF+          & 61.10$\pm$5.85 & 48.41$\pm$5.43 & 62.05$\pm$6.29 & 41.25$\pm$10.19 & 59.03$\pm$10.73 & 39.72$\pm$9.98 & 33.16$\pm$3.87 & 28.22$\pm$2.80 & 21.38$\pm$3.01 \\
      & GPPT          & 63.52$\pm$5.52 & 53.45$\pm$4.82 & 62.49$\pm$5.83 & 36.40$\pm$9.74 & 49.53$\pm$9.73 & 35.26$\pm$9.17 & 32.06$\pm$4.38 & 23.53$\pm$2.02 & 21.04$\pm$2.25 \\
      & GraphPrompt   & 65.68$\pm$4.71 & 54.29$\pm$5.43 & \underline{64.13$\pm$6.79} & 39.76$\pm$11.94 & 50.03$\pm$11.44 & 38.94$\pm$9.30 & 33.41$\pm$4.22 & 24.81$\pm$1.96 & 21.36$\pm$2.65 \\
      & EdgePrompt    & 66.13$\pm$5.83 & 53.71$\pm$4.68 & 63.90$\pm$6.26 & 37.53$\pm$9.84 & 48.69$\pm$11.71 & 41.22$\pm$9.78 & 33.62$\pm$3.86 & 25.01$\pm$2.01 & 21.60$\pm$2.47 \\
      & EdgePrompt+   & 66.39$\pm$5.03 & 53.50$\pm$4.95 & 64.06$\pm$6.39 & 39.02$\pm$8.96 & 47.86$\pm$8.60 & 39.02$\pm$9.32 & 33.37$\pm$4.05 & 24.91$\pm$2.64 & 21.12$\pm$2.82 \\
      & All-in-one    & 65.99$\pm$5.33 & 50.81$\pm$6.69 & 62.58$\pm$6.10 & 49.91$\pm$13.58 & 58.92$\pm$12.53 & 42.30$\pm$12.70 & 32.34$\pm$4.68 & 24.42$\pm$1.92 & 21.49$\pm$2.54 \\
      & HS-GPPT        & \underline{66.48$\pm$5.09} & 51.28$\pm$6.10 & 62.39$\pm$6.24 & 49.26$\pm$14.88 & 56.61$\pm$16.75 & 41.63$\pm$13.38 & 32.58$\pm$4.54 & 24.22$\pm$2.18 & 21.28$\pm$2.61 \\
      & ProNoG        & 56.89$\pm$7.09 & 42.20$\pm$6.69 & 60.83$\pm$5.13 & 46.82$\pm$8.84 & \underline{63.33$\pm$7.33} & 39.94$\pm$10.59 & 32.79$\pm$4.05 & 28.43$\pm$2.96 & 21.11$\pm$2.37 \\
      & UniPrompt     & 61.78$\pm$5.33 & \underline{59.40$\pm$5.03} & 62.72$\pm$6.59 & \underline{59.58$\pm$10.11} & 62.00$\pm$10.19 & \underline{74.67$\pm$8.07} & 28.75$\pm$3.21 & 21.63$\pm$1.73 & \textbf{22.30$\pm$2.54} \\
      & DAPrompt      & 66.37$\pm$4.59 & 52.72$\pm$4.74 & 64.10$\pm$5.98 & 39.79$\pm$12.15 & 53.78$\pm$12.20 & 39.69$\pm$10.96 & 33.27$\pm$3.80 & 24.86$\pm$2.09 & 20.75$\pm$2.15 \\
      \cmidrule{2-11}
      & \cellcolor{gray!15}\textbf{MINT (Ours)} & \cellcolor{gray!15}\textbf{67.19$\pm$4.18} & \cellcolor{gray!15}\textbf{60.28$\pm$3.94} & \cellcolor{gray!15}\textbf{64.54$\pm$5.71} & \cellcolor{gray!15}\textbf{61.61$\pm$10.04} & \cellcolor{gray!15}\textbf{68.06$\pm$7.49} & \cellcolor{gray!15}\textbf{75.89$\pm$4.31} & \cellcolor{gray!15}\textbf{35.65$\pm$2.63} & \cellcolor{gray!15}\textbf{29.17$\pm$2.75} & \cellcolor{gray!15}\underline{22.11$\pm$2.83} \\
      \bottomrule
    \end{tabular}%
  }
\end{table*}

\begin{table*}[htbp]
  \centering
  \caption{Mean accuracy $\pm$ standard deviation under 5-shot node classification. The best results across all methods are highlighted in \textbf{bold}, and the second-best results are \underline{underlined}.}
  \label{tab:5shot_main_results}
  \resizebox{\textwidth}{!}{
    \begin{tabular}{c l ccc cccccc}
      \toprule
      \multirow{2}{*}{\textbf{Pretrain}} & \multirow{2}{*}{\textbf{Methods}} & \multicolumn{3}{c}{\textbf{Homophilic Graphs}} & \multicolumn{6}{c}{\textbf{Heterophilic Graphs}} \\
      \cmidrule(lr){3-5} \cmidrule(lr){6-11}
      & & \textbf{Cora} & \textbf{CiteSeer} & \textbf{PubMed} & \textbf{Cornell} & \textbf{Texas} & \textbf{Wisconsin} & \textbf{Chameleon} & \textbf{Squirrel} & \textbf{Actor} \\
      \midrule

      \multirow{11}{*}{\textbf{GraphMAE}} 
      & Fine-tune     & 67.12$\pm$6.86 & \underline{66.01$\pm$2.68} & \textbf{72.22$\pm$5.52} & 62.41$\pm$10.19 & 53.28$\pm$14.30 & 48.09$\pm$8.17 & 27.39$\pm$2.90 & 21.52$\pm$1.63 & 21.06$\pm$1.45 \\
      & Linear Probe  & \underline{72.35$\pm$3.23} & 63.67$\pm$3.17 & 69.07$\pm$4.99 & 30.35$\pm$10.32 & 45.04$\pm$17.89 & 39.56$\pm$8.41 & 28.87$\pm$3.57 & 21.36$\pm$0.98 & 20.62$\pm$2.57 \\
      & GPF           & 69.52$\pm$8.04 & 62.47$\pm$4.31 & 68.05$\pm$5.21 & 29.11$\pm$12.61 & 50.52$\pm$15.17 & 38.44$\pm$11.54 & 27.88$\pm$3.28 & 21.00$\pm$0.83 & 20.37$\pm$2.85 \\
      & GPF+          & 66.30$\pm$8.48 & 62.72$\pm$4.30 & 68.58$\pm$4.87 & 30.10$\pm$15.50 & 48.67$\pm$14.16 & 39.79$\pm$8.01 & 27.11$\pm$3.97 & 21.18$\pm$0.97 & 20.48$\pm$2.70 \\
      & GPPT          & 69.82$\pm$3.76 & 63.25$\pm$3.34 & 66.54$\pm$6.00 & 30.06$\pm$9.21 & 40.93$\pm$20.94 & 37.73$\pm$7.90 & 27.32$\pm$4.08 & 20.98$\pm$0.97 & 21.34$\pm$1.73 \\
      & GraphPrompt   & 70.95$\pm$5.14 & 63.38$\pm$3.78 & 68.75$\pm$4.86 & 31.97$\pm$11.47 & 55.25$\pm$15.25 & 39.77$\pm$11.21 & 28.88$\pm$2.76 & 21.30$\pm$0.91 & 21.32$\pm$3.00 \\
      & EdgePrompt    & 71.97$\pm$5.04 & 62.68$\pm$4.80 & 66.70$\pm$7.08 & 29.05$\pm$10.15 & 44.09$\pm$16.79 & 39.04$\pm$9.22 & 28.06$\pm$3.66 & 21.11$\pm$0.75 & 21.28$\pm$1.86 \\
      & EdgePrompt+   & 66.20$\pm$11.74 & 62.52$\pm$3.78 & 65.50$\pm$8.31 & 30.89$\pm$10.37 & 46.52$\pm$21.68 & 38.34$\pm$10.56 & 27.09$\pm$4.03 & 20.73$\pm$0.91 & 21.25$\pm$2.49 \\
      & All-in-one    & 64.49$\pm$6.63 & 39.84$\pm$7.56 & 55.43$\pm$9.32 & 41.05$\pm$16.63 & 49.07$\pm$16.84 & 28.57$\pm$13.01 & 22.95$\pm$2.82 & \textbf{24.62$\pm$3.10} & 21.55$\pm$1.90 \\
      & HS-GPPT        & 69.98$\pm$4.11 & 56.33$\pm$4.24 & 64.44$\pm$4.76 & 32.41$\pm$7.93 & 53.36$\pm$11.19 & 41.39$\pm$11.02 & 25.18$\pm$2.92 & 23.14$\pm$1.68 & 20.97$\pm$2.00 \\
      & ProNoG        & 65.97$\pm$4.80 & 56.29$\pm$4.15 & 64.48$\pm$10.49 & 35.71$\pm$12.99 & 52.75$\pm$14.83 & 38.11$\pm$12.82 & \textbf{31.01$\pm$4.20} & 22.15$\pm$1.57 & 20.94$\pm$2.66 \\
      & UniPrompt     & 66.35$\pm$4.44 & 65.53$\pm$3.18 & 64.02$\pm$5.04 & \underline{65.40$\pm$6.70} & \underline{62.67$\pm$10.72} & \underline{71.18$\pm$5.42} & 28.31$\pm$2.40 & 21.95$\pm$1.45 & \underline{23.15$\pm$2.14} \\
      & DAPrompt      & 50.77$\pm$3.51 & 40.67$\pm$3.34 & 57.86$\pm$4.19 & 34.98$\pm$9.56 & 43.25$\pm$14.36 & 42.00$\pm$6.18 & 27.34$\pm$2.88 & 21.49$\pm$1.42 & 20.68$\pm$1.83 \\
      \cmidrule{2-11}
      & \cellcolor{gray!15}\textbf{MINT (Ours)} & \cellcolor{gray!15}\textbf{73.32$\pm$3.07} & \cellcolor{gray!15}\textbf{66.22$\pm$2.57} & \cellcolor{gray!15}\underline{69.50$\pm$4.49} & \cellcolor{gray!15}\textbf{66.79$\pm$8.33} & \cellcolor{gray!15}\textbf{64.35$\pm$12.55} & \cellcolor{gray!15}\textbf{80.04$\pm$5.52} & \cellcolor{gray!15}\underline{29.14$\pm$2.36} & \cellcolor{gray!15}\underline{23.16$\pm$1.84} & \cellcolor{gray!15}\textbf{24.05$\pm$2.75} \\
      \midrule

      \multirow{10}{*}{\textbf{GraphMAE2}}
      & Fine-tune     & \textbf{68.25$\pm$4.46} & 60.62$\pm$6.28 & \textbf{69.56$\pm$5.32} & 51.68$\pm$10.51 & 56.14$\pm$8.88 & 53.93$\pm$7.77 & 30.47$\pm$2.99 & 22.14$\pm$1.65 & 21.31$\pm$1.58 \\
      & Linear Probe  & 65.12$\pm$5.96 & 56.91$\pm$5.11 & 65.57$\pm$4.71 & 39.65$\pm$7.28 & 41.16$\pm$8.12 & 39.79$\pm$7.05 & \underline{32.35$\pm$2.74} & 22.07$\pm$1.44 & 21.69$\pm$2.55 \\
      & GPF           & 65.42$\pm$6.98 & 52.49$\pm$6.29 & 66.59$\pm$5.30 & 39.97$\pm$7.75 & 50.06$\pm$17.32 & 42.91$\pm$10.94 & 30.23$\pm$2.94 & 21.67$\pm$1.32 & 21.37$\pm$2.00 \\
      & GPF+          & 63.78$\pm$5.92 & 53.33$\pm$6.63 & 66.87$\pm$4.71 & 38.89$\pm$11.05 & 52.06$\pm$18.90 & 41.19$\pm$11.37 & 30.73$\pm$3.56 & 21.58$\pm$1.32 & \underline{22.03$\pm$2.56} \\
      & GPPT          & 64.66$\pm$3.68 & 56.25$\pm$4.04 & 65.70$\pm$4.57 & 40.16$\pm$7.53 & 39.19$\pm$8.41 & 39.00$\pm$5.20 & 31.98$\pm$2.56 & 21.62$\pm$1.38 & 21.56$\pm$1.33 \\
      & GraphPrompt   & 64.66$\pm$3.68 & 57.56$\pm$4.79 & 65.61$\pm$4.93 & 40.67$\pm$8.45 & 42.41$\pm$7.50 & 40.39$\pm$6.29 & 32.25$\pm$3.21 & 22.05$\pm$1.16 & 21.61$\pm$2.53 \\
      & EdgePrompt    & 65.44$\pm$4.35 & 55.47$\pm$4.61 & 65.96$\pm$5.27 & 44.98$\pm$8.46 & 46.52$\pm$14.09 & 40.35$\pm$8.89 & 30.94$\pm$3.48 & 21.99$\pm$1.37 & 21.33$\pm$1.74 \\
      & EdgePrompt+   & 65.98$\pm$4.63 & 54.02$\pm$5.01 & 66.60$\pm$4.34 & 42.19$\pm$11.82 & 46.43$\pm$19.17 & 38.46$\pm$11.82 & 30.50$\pm$3.53 & 21.72$\pm$1.63 & 21.57$\pm$1.78 \\
      & All-in-one    & 60.56$\pm$4.32 & 55.15$\pm$4.40 & 62.92$\pm$4.76 & 41.71$\pm$8.55 & 46.99$\pm$15.32 & 40.00$\pm$8.01 & 28.23$\pm$4.14 & \textbf{23.06$\pm$1.88} & 21.46$\pm$2.73 \\
      & HS-GPPT        & 63.89$\pm$4.27 & 53.75$\pm$4.12 & 66.16$\pm$3.88 & 35.30$\pm$11.04 & 47.10$\pm$10.60 & 38.03$\pm$9.80 & 30.55$\pm$2.79 & 22.38$\pm$1.92 & 21.98$\pm$3.23 \\
      & ProNoG        & 65.77$\pm$6.03 & 50.96$\pm$5.48 & 66.26$\pm$4.92 & 39.90$\pm$12.80 & 55.16$\pm$15.44 & 42.79$\pm$9.70 & \textbf{32.64$\pm$4.21} & 22.26$\pm$1.86 & 21.11$\pm$2.07 \\
      & UniPrompt     & 63.51$\pm$3.42 & \underline{60.91$\pm$3.33} & 62.15$\pm$5.75 & \underline{62.00$\pm$6.81} & \underline{62.81$\pm$9.78} & \underline{60.10$\pm$9.17} & 28.57$\pm$2.23 & 21.95$\pm$1.61 & 21.95$\pm$1.88 \\
      & DAPrompt      & 65.78$\pm$4.40 & 45.92$\pm$3.49 & 63.96$\pm$4.89 & 40.06$\pm$7.97 & 43.59$\pm$9.68 & 40.02$\pm$8.99 & 28.83$\pm$2.68 & 22.16$\pm$1.69 & 21.73$\pm$2.26 \\
      \cmidrule{2-11}
      & \cellcolor{gray!15}\textbf{MINT (Ours)} & \cellcolor{gray!15}\underline{66.25$\pm$4.63} & \cellcolor{gray!15}\textbf{62.35$\pm$2.96} & \cellcolor{gray!15}\underline{67.24$\pm$4.19} & \cellcolor{gray!15}\textbf{63.46$\pm$6.95} & \cellcolor{gray!15}\textbf{65.94$\pm$10.40} & \cellcolor{gray!15}\textbf{70.52$\pm$4.63} & \cellcolor{gray!15}30.89$\pm$1.91 & \cellcolor{gray!15}\underline{22.75$\pm$1.70} & \cellcolor{gray!15}\textbf{24.05$\pm$2.84} \\
      \midrule
      \multirow{10}{*}{\textbf{MaskGAE}}
      & Fine-tune     & 72.41$\pm$3.66 & 59.38$\pm$3.17 & 69.07$\pm$4.77 & 40.98$\pm$11.82 & 53.45$\pm$7.95 & 42.66$\pm$8.06 & 35.15$\pm$3.91 & 25.45$\pm$2.38 & 21.75$\pm$2.41 \\
      & Linear Probe  & 72.63$\pm$3.26 & 57.67$\pm$3.13 & 69.01$\pm$4.39 & 39.84$\pm$11.10 & 51.51$\pm$7.59 & 42.54$\pm$8.11 & 35.20$\pm$4.35 & 25.36$\pm$2.29 & 21.83$\pm$2.47 \\
      & GPF           & 66.71$\pm$5.07 & 52.12$\pm$5.86 & 65.62$\pm$6.40 & 53.02$\pm$7.57 & 63.80$\pm$7.44 & 47.13$\pm$7.67 & 34.90$\pm$3.25 & \underline{29.92$\pm$3.27} & 21.94$\pm$2.50 \\
      & GPF+          & 69.20$\pm$4.49 & 53.76$\pm$4.67 & 67.83$\pm$5.11 & 46.03$\pm$12.25 & 59.62$\pm$7.63 & 43.78$\pm$8.55 & 34.10$\pm$4.50 & 28.52$\pm$3.40 & \underline{22.37$\pm$1.83} \\
      & GPPT          & 69.23$\pm$4.10 & 58.75$\pm$3.49 & 66.75$\pm$5.34 & 39.37$\pm$12.30 & 50.20$\pm$9.18 & 36.86$\pm$8.48 & 34.77$\pm$3.46 & 24.90$\pm$2.18 & 21.90$\pm$1.88 \\
      & GraphPrompt   & 71.81$\pm$3.08 & 60.17$\pm$2.83 & 69.04$\pm$5.72 & 40.73$\pm$8.07 & 53.88$\pm$10.42 & 43.87$\pm$7.81 & \underline{35.25$\pm$4.13} & 25.53$\pm$2.29 & 22.13$\pm$2.14 \\
      & EdgePrompt    & \underline{72.75$\pm$3.28} & 58.64$\pm$3.95 & 69.13$\pm$4.58 & 40.00$\pm$11.34 & 54.23$\pm$11.13 & 43.43$\pm$7.90 & 34.42$\pm$3.21 & 25.59$\pm$2.58 & 21.64$\pm$2.09 \\
      & EdgePrompt+   & \textbf{72.81$\pm$3.70} & 59.70$\pm$2.50 & \underline{69.15$\pm$4.53} & 41.62$\pm$12.91 & 54.29$\pm$10.37 & 43.80$\pm$8.24 & 34.91$\pm$4.30 & 25.86$\pm$2.50 & 21.77$\pm$2.37 \\
      & All-in-one    & 72.05$\pm$3.64 & 57.13$\pm$4.22 & 66.40$\pm$5.65 & 52.35$\pm$12.25 & 63.86$\pm$7.31 & 43.20$\pm$12.09 & 34.83$\pm$3.80 & 25.61$\pm$2.08 & 20.49$\pm$2.39 \\
      & HS-GPPT        & 71.75$\pm$3.88 & 57.59$\pm$4.14 & 67.75$\pm$5.28 & 52.70$\pm$12.43 & 62.84$\pm$9.83 & 44.68$\pm$12.02 & 35.07$\pm$3.81 & 25.43$\pm$2.17 & 21.20$\pm$2.59 \\
      & ProNoG        & 67.00$\pm$4.44 & 51.97$\pm$6.05 & 66.06$\pm$6.42 & 49.94$\pm$10.50 & 60.43$\pm$8.52 & 44.07$\pm$6.36 & 34.77$\pm$3.79 & 28.95$\pm$2.96 & 21.27$\pm$3.03 \\
      & UniPrompt     & 65.32$\pm$3.27 & \underline{63.78$\pm$2.88} & 64.78$\pm$4.06 & \underline{63.17$\pm$8.07} & \underline{64.70$\pm$8.04} & \textbf{79.25$\pm$7.35} & 29.83$\pm$2.81 & 22.56$\pm$2.04 & 22.24$\pm$2.49 \\
      & DAPrompt      & 72.17$\pm$2.89 & 58.65$\pm$3.48 & 69.05$\pm$4.91 & 44.19$\pm$11.71 & 55.48$\pm$11.50 & 44.66$\pm$9.22 & 35.14$\pm$3.95 & 25.56$\pm$2.23 & 21.08$\pm$1.99 \\
      
      \cmidrule{2-11}
      & \cellcolor{gray!15}\textbf{MINT (Ours)} & \cellcolor{gray!15}71.85$\pm$3.02 & \cellcolor{gray!15}\textbf{64.12$\pm$2.90} & \cellcolor{gray!15}\textbf{69.21$\pm$3.84} & \cellcolor{gray!15}\textbf{64.48$\pm$5.61} & \cellcolor{gray!15}\textbf{69.97$\pm$6.89} & \cellcolor{gray!15}\underline{78.92$\pm$5.45} & \cellcolor{gray!15}\textbf{35.54$\pm$2.67} & \cellcolor{gray!15}\textbf{30.05$\pm$2.70} & \cellcolor{gray!15}\textbf{22.43$\pm$2.53}\\
      \bottomrule
    \end{tabular}%
  }
\end{table*}

\paragraph{GraphMAE2~\cite{hou2023graphmae2}}GraphMAE2 extends masked feature reconstruction with multi-view random re-mask decoding and latent representation prediction. Its pretrained encoder is frozen during the downstream prompt-adaptation experiments.

\begin{table*}[!t]
\centering
\caption{Hyperparameter settings for different models across nine datasets.}
\label{tab:hyperparams}
\small 
\resizebox{\textwidth}{!}{%
\begin{tabular}{l l ccc ccc ccc}
\toprule
\multirow{2}{*}{\textbf{Model}} & \multirow{2}{*}{\textbf{Dataset}} & \multicolumn{3}{c}{\textbf{1-Shot}} & \multicolumn{3}{c}{\textbf{3-Shot}} & \multicolumn{3}{c}{\textbf{5-Shot}} \\
\cmidrule(lr){3-5} \cmidrule(lr){6-8} \cmidrule(lr){9-11}
 & & \textbf{LR} & \textbf{Beta} & \textbf{k} & \textbf{LR} & \textbf{Beta} & \textbf{k} & \textbf{LR} & \textbf{Beta} & \textbf{k} \\
\midrule
\multirow{9}{*}{\textbf{GraphMAE}} 
 & Cora      & 0.001  & 0.015  & 3   & 0.001  & 0.035  & 3   & 0.0012 & 0.03   & 1   \\
 & CiteSeer  & 0.005  & 0.001  & 50  & 0.0008 & 0.05   & 75  & 0.001  & 0.005  & 50  \\
 & PubMed    & 0.005  & 0.0005 & 1   & 0.004  & 0.3    & 1   & 0.007  & 0.4    & 1   \\
 & Cornell   & 0.005  & 0.05   & 60  & 0.001  & 0.01   & 60  & 0.08   & 0.005  & 40  \\
 & Texas     & 0.01   & 0.001  & 80  & 0.01   & 0.001  & 60  & 0.07   & 0.008  & 55  \\
 & Wisconsin & 0.005  & 0.005  & 120 & 0.0005 & 0.001  & 80  & 0.0005 & 0.001  & 80  \\
 & Chameleon & 0.0005 & 0.01   & 160 & 0.0004 & 1.2    & 20  & 0.005  & 0.08   & 20  \\
 & Squirrel  & 0.05   & 0.005  & 40  & 0.007  & 0.08   & 35  & 0.01   & 0.008  & 35  \\
 & Actor     & 0.08   & 0.005  & 140 & 0.001  & 0.01   & 120 & 0.0015 & 0.2    & 120 \\
\midrule
\multirow{9}{*}{\textbf{GraphMAE2}} 
 & Cora      & 0.01   & 0.003  & 10  & 0.015  & 0.05   & 7   & 0.001  & 0.05   & 2   \\
 & CiteSeer  & 0.0015 & 0.008  & 60  & 0.0015 & 0.03   & 70  & 0.0015 & 0.003  & 110 \\
 & PubMed    & 0.037  & 0.001  & 8   & 0.05   & 0.05   & 2   & 0.01   & 0.05   & 2   \\
 & Cornell   & 0.0008 & 0.05   & 100 & 0.001  & 0.12   & 80  & 0.005  & 0.3    & 40  \\
 & Texas     & 0.001  & 0.001  & 120 & 0.01   & 0.01   & 110 & 0.001  & 0.1    & 120 \\
 & Wisconsin & 0.001  & 0.1    & 120 & 0.05   & 0.01   & 120 & 0.0003 & 0.1    & 160 \\
 & Chameleon & 0.01   & 0.001  & 200 & 0.002  & 0.05   & 240 & 0.005  & 0.005  & 160 \\
 & Squirrel  & 0.011  & 0.8    & 22  & 0.008  & 0.3    & 15  & 0.003  & 0.008  & 10  \\
 & Actor     & 0.008  & 0.01   & 240 & 0.05   & 0.0008 & 150 & 0.008  & 0.003  & 240 \\
\midrule
\multirow{9}{*}{\textbf{MaskGAE}} 
 & Cora      & 0.008  & 0.0035 & 11  & 0.001  & 0.03   & 1   & 0.0005 & 0.03   & 1   \\
 & CiteSeer  & 0.0006 & 0.04   & 55  & 0.0008 & 0.03   & 30  & 0.0005 & 0.00001& 50  \\
 & PubMed    & 0.01   & 0.0001 & 5   & 0.08   & 0.08   & 2   & 0.05   & 0.008  & 2   \\
 & Cornell   & 0.0013 & 0.1    & 70  & 0.007  & 0.12   & 35  & 0.008  & 0.08   & 45  \\
 & Texas     & 0.00095& 0.0035 & 82  & 0.005  & 0.05   & 120 & 0.005  & 0.2    & 120 \\
 & Wisconsin & 0.00032& 0.13   & 78  & 0.01   & 0.05   & 82  & 0.001  & 0.04   & 80  \\
 & Chameleon & 0.1    & 0.025  & 550 & 0.012  & 0.35   & 400 & 0.005  & 0.06   & 400 \\
 & Squirrel  & 1.2    & 15.0   & 2   & 1.5    & 50.0   & 1   & 1.8    & 50.0   & 1   \\
 & Actor     & 0.008  & 0.05   & 140 & 0.03   & 0.001  & 85  & 0.05   & 0.04   & 130 \\
\bottomrule
\end{tabular}}
\end{table*}

\paragraph{MaskGAE~\cite{li2023maskgae}}MaskGAE masks graph edges and pretrains a GNN encoder by reconstructing graph structure with edge- and degree-oriented decoders. The encoder remains frozen during prompt adaptation; the corresponding topology-compatible variants use binary graph replacement because this backbone lacks the scalar edge-weight interface used by Original MINT.

\begin{table*}[!t]
\centering
\caption{Robustness Evaluation Results Under Noise Edge Perturbation (Accuracy$_{\pm \text{Std}}$)}
\label{tab:robustness_noise}
\resizebox{\textwidth}{!}{
\begin{tabular}{l c c ccccc ccccc}
\toprule
\multirow{2}{*}{\textbf{Model}} & \multirow{2}{*}{\textbf{Shot}} & \multirow{2}{*}{\textbf{Noise}} & \multicolumn{5}{c}{\textbf{Homophilic / Small Heterophilic Datasets}} & \multicolumn{4}{c}{\textbf{Dense Heterophilic Datasets}} \\
\cmidrule(lr){4-8} \cmidrule(lr){9-12}
 &  &  & \textbf{Cora} & \textbf{CiteSeer} & \textbf{PubMed} & \textbf{Cornell} & \textbf{Texas} & \textbf{Wisconsin} & \textbf{Chameleon} & \textbf{Squirrel} & \textbf{Actor} \\
\midrule
\multirow{18}{*}{\textbf{GraphMAE}} 
& \multirow{6}{*}{1} 
& 0.0 & $0.5346_{\text{\tiny $\pm$0.0736}}$ & $0.5208_{\text{\tiny $\pm$0.0662}}$ & $0.5818_{\text{\tiny $\pm$0.0720}}$ & $0.5073_{\text{\tiny $\pm$0.1382}}$ & $0.5339_{\text{\tiny $\pm$0.2014}}$ & $0.6622_{\text{\tiny $\pm$0.1485}}$ & $0.2514_{\text{\tiny $\pm$0.0293}}$ & $0.2183_{\text{\tiny $\pm$0.0200}}$ & $0.2273_{\text{\tiny $\pm$0.0326}}$ \\
& & 0.2 & $0.4794_{\text{\tiny $\pm$0.0868}}$ & $0.5113_{\text{\tiny $\pm$0.0684}}$ & $0.4720_{\text{\tiny $\pm$0.0505}}$ & $0.5039_{\text{\tiny $\pm$0.1544}}$ & $0.5151_{\text{\tiny $\pm$0.2106}}$ & $0.6652_{\text{\tiny $\pm$0.1370}}$ & $0.2404_{\text{\tiny $\pm$0.0271}}$ & $0.2013_{\text{\tiny $\pm$0.0107}}$ & $0.2269_{\text{\tiny $\pm$0.0290}}$ \\
& & 0.4 & $0.4363_{\text{\tiny $\pm$0.0646}}$ & $0.5100_{\text{\tiny $\pm$0.0665}}$ & $0.4229_{\text{\tiny $\pm$0.0522}}$ & $0.4930_{\text{\tiny $\pm$0.1454}}$ & $0.5116_{\text{\tiny $\pm$0.2108}}$ & $0.6834_{\text{\tiny $\pm$0.1334}}$ & $0.2369_{\text{\tiny $\pm$0.0255}}$ & $0.1991_{\text{\tiny $\pm$0.0062}}$ & $0.2198_{\text{\tiny $\pm$0.0344}}$ \\
& & 0.6 & $0.4058_{\text{\tiny $\pm$0.0763}}$ & $0.5006_{\text{\tiny $\pm$0.0732}}$ & $0.3805_{\text{\tiny $\pm$0.0660}}$ & $0.4826_{\text{\tiny $\pm$0.1530}}$ & $0.4865_{\text{\tiny $\pm$0.2129}}$ & $0.6768_{\text{\tiny $\pm$0.1240}}$ & $0.2389_{\text{\tiny $\pm$0.0199}}$ & $0.2018_{\text{\tiny $\pm$0.0050}}$ & $0.2212_{\text{\tiny $\pm$0.0317}}$ \\
& & 0.8 & $0.3714_{\text{\tiny $\pm$0.0587}}$ & $0.4911_{\text{\tiny $\pm$0.0696}}$ & $0.3636_{\text{\tiny $\pm$0.0654}}$ & $0.4804_{\text{\tiny $\pm$0.1484}}$ & $0.4630_{\text{\tiny $\pm$0.2177}}$ & $0.6533_{\text{\tiny $\pm$0.1376}}$ & $0.2391_{\text{\tiny $\pm$0.0203}}$ & $0.2016_{\text{\tiny $\pm$0.0047}}$ & $0.2229_{\text{\tiny $\pm$0.0289}}$ \\
& & 1.0 & $0.3488_{\text{\tiny $\pm$0.0608}}$ & $0.4885_{\text{\tiny $\pm$0.0602}}$ & $0.3425_{\text{\tiny $\pm$0.0613}}$ & $0.5143_{\text{\tiny $\pm$0.1321}}$ & $0.4696_{\text{\tiny $\pm$0.2258}}$ & $0.6615_{\text{\tiny $\pm$0.1418}}$ & $0.2352_{\text{\tiny $\pm$0.0253}}$ & $0.2004_{\text{\tiny $\pm$0.0030}}$ & $0.2206_{\text{\tiny $\pm$0.0246}}$ \\
\cmidrule(lr){2-12}
& \multirow{6}{*}{3} 
& 0.0 & $0.6690_{\text{\tiny $\pm$0.0418}}$ & $0.6154_{\text{\tiny $\pm$0.0487}}$ & $0.6610_{\text{\tiny $\pm$0.0545}}$ & $0.6063_{\text{\tiny $\pm$0.0672}}$ & $0.6008_{\text{\tiny $\pm$0.1304}}$ & $0.7656_{\text{\tiny $\pm$0.0940}}$ & $0.2757_{\text{\tiny $\pm$0.0256}}$ & $0.2277_{\text{\tiny $\pm$0.0178}}$ & $0.2309_{\text{\tiny $\pm$0.0252}}$ \\
& & 0.2 & $0.6150_{\text{\tiny $\pm$0.0450}}$ & $0.6103_{\text{\tiny $\pm$0.0526}}$ & $0.5465_{\text{\tiny $\pm$0.0528}}$ & $0.5988_{\text{\tiny $\pm$0.0730}}$ & $0.6092_{\text{\tiny $\pm$0.1402}}$ & $0.7643_{\text{\tiny $\pm$0.0864}}$ & $0.2425_{\text{\tiny $\pm$0.0202}}$ & $0.2080_{\text{\tiny $\pm$0.0106}}$ & $0.2246_{\text{\tiny $\pm$0.0233}}$ \\
& & 0.4 & $0.5694_{\text{\tiny $\pm$0.0505}}$ & $0.6039_{\text{\tiny $\pm$0.0669}}$ & $0.4703_{\text{\tiny $\pm$0.0628}}$ & $0.6012_{\text{\tiny $\pm$0.0791}}$ & $0.5903_{\text{\tiny $\pm$0.1230}}$ & $0.7511_{\text{\tiny $\pm$0.0925}}$ & $0.2367_{\text{\tiny $\pm$0.0199}}$ & $0.2069_{\text{\tiny $\pm$0.0089}}$ & $0.2210_{\text{\tiny $\pm$0.0268}}$ \\
& & 0.6 & $0.5181_{\text{\tiny $\pm$0.0518}}$ & $0.6002_{\text{\tiny $\pm$0.0661}}$ & $0.4351_{\text{\tiny $\pm$0.0510}}$ & $0.5920_{\text{\tiny $\pm$0.0942}}$ & $0.5497_{\text{\tiny $\pm$0.1789}}$ & $0.7519_{\text{\tiny $\pm$0.0846}}$ & $0.2305_{\text{\tiny $\pm$0.0245}}$ & $0.2017_{\text{\tiny $\pm$0.0073}}$ & $0.2216_{\text{\tiny $\pm$0.0265}}$ \\
& & 0.8 & $0.4877_{\text{\tiny $\pm$0.0518}}$ & $0.5969_{\text{\tiny $\pm$0.0634}}$ & $0.3999_{\text{\tiny $\pm$0.0473}}$ & $0.5920_{\text{\tiny $\pm$0.0981}}$ & $0.5692_{\text{\tiny $\pm$0.1545}}$ & $0.7356_{\text{\tiny $\pm$0.0837}}$ & $0.2301_{\text{\tiny $\pm$0.0228}}$ & $0.2010_{\text{\tiny $\pm$0.0079}}$ & $0.2180_{\text{\tiny $\pm$0.0257}}$ \\
& & 1.0 & $0.4574_{\text{\tiny $\pm$0.0517}}$ & $0.5959_{\text{\tiny $\pm$0.0652}}$ & $0.3868_{\text{\tiny $\pm$0.0450}}$ & $0.5937_{\text{\tiny $\pm$0.0915}}$ & $0.5539_{\text{\tiny $\pm$0.1543}}$ & $0.7300_{\text{\tiny $\pm$0.0844}}$ & $0.2232_{\text{\tiny $\pm$0.0218}}$ & $0.2029_{\text{\tiny $\pm$0.0045}}$ & $0.2150_{\text{\tiny $\pm$0.0288}}$ \\
\cmidrule(lr){2-12}
& \multirow{6}{*}{5} 
& 0.0 & $0.7332_{\text{\tiny $\pm$0.0307}}$ & $0.6622_{\text{\tiny $\pm$0.0257}}$ & $0.6950_{\text{\tiny $\pm$0.0449}}$ & $0.6679_{\text{\tiny $\pm$0.0833}}$ & $0.6435_{\text{\tiny $\pm$0.1255}}$ & $0.8004_{\text{\tiny $\pm$0.0552}}$ & $0.2914_{\text{\tiny $\pm$0.0236}}$ & $0.2316_{\text{\tiny $\pm$0.0184}}$ & $0.2405_{\text{\tiny $\pm$0.0275}}$ \\
& & 0.2 & $0.6662_{\text{\tiny $\pm$0.0366}}$ & $0.6587_{\text{\tiny $\pm$0.0265}}$ & $0.5808_{\text{\tiny $\pm$0.0393}}$ & $0.6422_{\text{\tiny $\pm$0.0861}}$ & $0.6330_{\text{\tiny $\pm$0.1382}}$ & $0.7929_{\text{\tiny $\pm$0.0608}}$ & $0.2594_{\text{\tiny $\pm$0.0229}}$ & $0.2082_{\text{\tiny $\pm$0.0104}}$ & $0.2426_{\text{\tiny $\pm$0.0183}}$ \\
& & 0.4 & $0.5988_{\text{\tiny $\pm$0.0320}}$ & $0.6503_{\text{\tiny $\pm$0.0318}}$ & $0.5041_{\text{\tiny $\pm$0.0494}}$ & $0.6416_{\text{\tiny $\pm$0.1384}}$ & $0.6049_{\text{\tiny $\pm$0.1273}}$ & $0.7871_{\text{\tiny $\pm$0.0614}}$ & $0.2494_{\text{\tiny $\pm$0.0210}}$ & $0.2046_{\text{\tiny $\pm$0.0088}}$ & $0.2348_{\text{\tiny $\pm$0.0202}}$ \\
& & 0.6 & $0.5451_{\text{\tiny $\pm$0.0362}}$ & $0.6452_{\text{\tiny $\pm$0.0298}}$ & $0.4456_{\text{\tiny $\pm$0.0384}}$ & $0.6130_{\text{\tiny $\pm$0.1288}}$ & $0.5751_{\text{\tiny $\pm$0.1221}}$ & $0.7859_{\text{\tiny $\pm$0.0595}}$ & $0.2388_{\text{\tiny $\pm$0.0225}}$ & $0.2029_{\text{\tiny $\pm$0.0063}}$ & $0.2337_{\text{\tiny $\pm$0.0226}}$ \\
& & 0.8 & $0.4908_{\text{\tiny $\pm$0.0403}}$ & $0.6467_{\text{\tiny $\pm$0.0300}}$ & $0.4250_{\text{\tiny $\pm$0.0506}}$ & $0.6232_{\text{\tiny $\pm$0.0950}}$ & $0.5672_{\text{\tiny $\pm$0.1666}}$ & $0.7705_{\text{\tiny $\pm$0.0641}}$ & $0.2367_{\text{\tiny $\pm$0.0195}}$ & $0.2012_{\text{\tiny $\pm$0.0067}}$ & $0.2333_{\text{\tiny $\pm$0.0183}}$ \\
& & 1.0 & $0.4584_{\text{\tiny $\pm$0.0381}}$ & $0.6416_{\text{\tiny $\pm$0.0306}}$ & $0.4034_{\text{\tiny $\pm$0.0414}}$ & $0.6019_{\text{\tiny $\pm$0.1080}}$ & $0.5919_{\text{\tiny $\pm$0.1559}}$ & $0.7699_{\text{\tiny $\pm$0.0644}}$ & $0.2270_{\text{\tiny $\pm$0.0221}}$ & $0.2026_{\text{\tiny $\pm$0.0046}}$ & $0.2317_{\text{\tiny $\pm$0.0161}}$ \\
\midrule
\multirow{18}{*}{\textbf{GraphMAE2}} 
& \multirow{6}{*}{1} 
& 0.0 & $0.4445_{\text{\tiny $\pm$0.0876}}$ & $0.4934_{\text{\tiny $\pm$0.0730}}$ & $0.5388_{\text{\tiny $\pm$0.0814}}$ & $0.5076_{\text{\tiny $\pm$0.1479}}$ & $0.4870_{\text{\tiny $\pm$0.2065}}$ & $0.5683_{\text{\tiny $\pm$0.1676}}$ & $0.2543_{\text{\tiny $\pm$0.0241}}$ & $0.2204_{\text{\tiny $\pm$0.0176}}$ & $0.2234_{\text{\tiny $\pm$0.0254}}$ \\
& & 0.2 & $0.4087_{\text{\tiny $\pm$0.0946}}$ & $0.4876_{\text{\tiny $\pm$0.0711}}$ & $0.4882_{\text{\tiny $\pm$0.0710}}$ & $0.5202_{\text{\tiny $\pm$0.1263}}$ & $0.3474_{\text{\tiny $\pm$0.1834}}$ & $0.5185_{\text{\tiny $\pm$0.1731}}$ & $0.2380_{\text{\tiny $\pm$0.0236}}$ & $0.2047_{\text{\tiny $\pm$0.0102}}$ & $0.2169_{\text{\tiny $\pm$0.0301}}$ \\
& & 0.4 & $0.3726_{\text{\tiny $\pm$0.0884}}$ & $0.4816_{\text{\tiny $\pm$0.0710}}$ & $0.4597_{\text{\tiny $\pm$0.0592}}$ & $0.5123_{\text{\tiny $\pm$0.1403}}$ & $0.3638_{\text{\tiny $\pm$0.1709}}$ & $0.5291_{\text{\tiny $\pm$0.1658}}$ & $0.2312_{\text{\tiny $\pm$0.0233}}$ & $0.2040_{\text{\tiny $\pm$0.0094}}$ & $0.2176_{\text{\tiny $\pm$0.0292}}$ \\
& & 0.6 & $0.3471_{\text{\tiny $\pm$0.0725}}$ & $0.4760_{\text{\tiny $\pm$0.0768}}$ & $0.4318_{\text{\tiny $\pm$0.0694}}$ & $0.5011_{\text{\tiny $\pm$0.1445}}$ & $0.3878_{\text{\tiny $\pm$0.1857}}$ & $0.5422_{\text{\tiny $\pm$0.1595}}$ & $0.2288_{\text{\tiny $\pm$0.0182}}$ & $0.2009_{\text{\tiny $\pm$0.0063}}$ & $0.2149_{\text{\tiny $\pm$0.0271}}$ \\
& & 0.8 & $0.3354_{\text{\tiny $\pm$0.0855}}$ & $0.4768_{\text{\tiny $\pm$0.0725}}$ & $0.4343_{\text{\tiny $\pm$0.0477}}$ & $0.5025_{\text{\tiny $\pm$0.1386}}$ & $0.3577_{\text{\tiny $\pm$0.1766}}$ & $0.5333_{\text{\tiny $\pm$0.1582}}$ & $0.2292_{\text{\tiny $\pm$0.0211}}$ & $0.1997_{\text{\tiny $\pm$0.0053}}$ & $0.2142_{\text{\tiny $\pm$0.0283}}$ \\
& & 1.0 & $0.3145_{\text{\tiny $\pm$0.0776}}$ & $0.4725_{\text{\tiny $\pm$0.0738}}$ & $0.4001_{\text{\tiny $\pm$0.0500}}$ & $0.5000_{\text{\tiny $\pm$0.1325}}$ & $0.3566_{\text{\tiny $\pm$0.1728}}$ & $0.5292_{\text{\tiny $\pm$0.1567}}$ & $0.2305_{\text{\tiny $\pm$0.0223}}$ & $0.2000_{\text{\tiny $\pm$0.0052}}$ & $0.2151_{\text{\tiny $\pm$0.0294}}$ \\
\cmidrule(lr){2-12}
& \multirow{6}{*}{3} 
& 0.0 & $0.6015_{\text{\tiny $\pm$0.0531}}$ & $0.5780_{\text{\tiny $\pm$0.0555}}$ & $0.6282_{\text{\tiny $\pm$0.0714}}$ & $0.6000_{\text{\tiny $\pm$0.0889}}$ & $0.6233_{\text{\tiny $\pm$0.1552}}$ & $0.6617_{\text{\tiny $\pm$0.1222}}$ & $0.2854_{\text{\tiny $\pm$0.0233}}$ & $0.2246_{\text{\tiny $\pm$0.0170}}$ & $0.2363_{\text{\tiny $\pm$0.0253}}$ \\
& & 0.2 & $0.5605_{\text{\tiny $\pm$0.0538}}$ & $0.5683_{\text{\tiny $\pm$0.0624}}$ & $0.5342_{\text{\tiny $\pm$0.0624}}$ & $0.5952_{\text{\tiny $\pm$0.0888}}$ & $0.5886_{\text{\tiny $\pm$0.1434}}$ & $0.5944_{\text{\tiny $\pm$0.1618}}$ & $0.2739_{\text{\tiny $\pm$0.0261}}$ & $0.2038_{\text{\tiny $\pm$0.0086}}$ & $0.2204_{\text{\tiny $\pm$0.0247}}$ \\
& & 0.4 & $0.5500_{\text{\tiny $\pm$0.0510}}$ & $0.5670_{\text{\tiny $\pm$0.0623}}$ & $0.4580_{\text{\tiny $\pm$0.0773}}$ & $0.5905_{\text{\tiny $\pm$0.0823}}$ & $0.5836_{\text{\tiny $\pm$0.1484}}$ & $0.5804_{\text{\tiny $\pm$0.1484}}$ & $0.2702_{\text{\tiny $\pm$0.0225}}$ & $0.2012_{\text{\tiny $\pm$0.0051}}$ & $0.2172_{\text{\tiny $\pm$0.0252}}$ \\
& & 0.6 & $0.5135_{\text{\tiny $\pm$0.0540}}$ & $0.5608_{\text{\tiny $\pm$0.0613}}$ & $0.4228_{\text{\tiny $\pm$0.0640}}$ & $0.5798_{\text{\tiny $\pm$0.1023}}$ & $0.5711_{\text{\tiny $\pm$0.1459}}$ & $0.5987_{\text{\tiny $\pm$0.1578}}$ & $0.2615_{\text{\tiny $\pm$0.0261}}$ & $0.2008_{\text{\tiny $\pm$0.0067}}$ & $0.2162_{\text{\tiny $\pm$0.0221}}$ \\
& & 0.8 & $0.4933_{\text{\tiny $\pm$0.0397}}$ & $0.5573_{\text{\tiny $\pm$0.0639}}$ & $0.4163_{\text{\tiny $\pm$0.0694}}$ & $0.5833_{\text{\tiny $\pm$0.0864}}$ & $0.5556_{\text{\tiny $\pm$0.1574}}$ & $0.6011_{\text{\tiny $\pm$0.1548}}$ & $0.2559_{\text{\tiny $\pm$0.0251}}$ & $0.2021_{\text{\tiny $\pm$0.0067}}$ & $0.2128_{\text{\tiny $\pm$0.0269}}$ \\
& & 1.0 & $0.4575_{\text{\tiny $\pm$0.0475}}$ & $0.5597_{\text{\tiny $\pm$0.0716}}$ & $0.4052_{\text{\tiny $\pm$0.0686}}$ & $0.5869_{\text{\tiny $\pm$0.0862}}$ & $0.5522_{\text{\tiny $\pm$0.1594}}$ & $0.6165_{\text{\tiny $\pm$0.1478}}$ & $0.2515_{\text{\tiny $\pm$0.0233}}$ & $0.2017_{\text{\tiny $\pm$0.0054}}$ & $0.2200_{\text{\tiny $\pm$0.0319}}$ \\
\cmidrule(lr){2-12}
& \multirow{6}{*}{5} 
& 0.0 & $0.6625_{\text{\tiny $\pm$0.0463}}$ & $0.6235_{\text{\tiny $\pm$0.0296}}$ & $0.6724_{\text{\tiny $\pm$0.0419}}$ & $0.6346_{\text{\tiny $\pm$0.0695}}$ & $0.6594_{\text{\tiny $\pm$0.1040}}$ & $0.7052_{\text{\tiny $\pm$0.0463}}$ & $0.3089_{\text{\tiny $\pm$0.0191}}$ & $0.2275_{\text{\tiny $\pm$0.0170}}$ & $0.2405_{\text{\tiny $\pm$0.0284}}$ \\
& & 0.2 & $0.6034_{\text{\tiny $\pm$0.0534}}$ & $0.6188_{\text{\tiny $\pm$0.0368}}$ & $0.5909_{\text{\tiny $\pm$0.0470}}$ & $0.6292_{\text{\tiny $\pm$0.0717}}$ & $0.5820_{\text{\tiny $\pm$0.1365}}$ & $0.6848_{\text{\tiny $\pm$0.0591}}$ & $0.2970_{\text{\tiny $\pm$0.0200}}$ & $0.2053_{\text{\tiny $\pm$0.0074}}$ & $0.2322_{\text{\tiny $\pm$0.0321}}$ \\
& & 0.4 & $0.5667_{\text{\tiny $\pm$0.0404}}$ & $0.6135_{\text{\tiny $\pm$0.0363}}$ & $0.5291_{\text{\tiny $\pm$0.0547}}$ & $0.6152_{\text{\tiny $\pm$0.0655}}$ & $0.5890_{\text{\tiny $\pm$0.1444}}$ & $0.6850_{\text{\tiny $\pm$0.0661}}$ & $0.2927_{\text{\tiny $\pm$0.0177}}$ & $0.2028_{\text{\tiny $\pm$0.0070}}$ & $0.2285_{\text{\tiny $\pm$0.0294}}$ \\
& & 0.6 & $0.5261_{\text{\tiny $\pm$0.0372}}$ & $0.6095_{\text{\tiny $\pm$0.0399}}$ & $0.4970_{\text{\tiny $\pm$0.0452}}$ & $0.6041_{\text{\tiny $\pm$0.0651}}$ & $0.5838_{\text{\tiny $\pm$0.1335}}$ & $0.6944_{\text{\tiny $\pm$0.0673}}$ & $0.2857_{\text{\tiny $\pm$0.0210}}$ & $0.2032_{\text{\tiny $\pm$0.0061}}$ & $0.2263_{\text{\tiny $\pm$0.0358}}$ \\
& & 0.8 & $0.4803_{\text{\tiny $\pm$0.0455}}$ & $0.6166_{\text{\tiny $\pm$0.0312}}$ & $0.4683_{\text{\tiny $\pm$0.0461}}$ & $0.5803_{\text{\tiny $\pm$0.0645}}$ & $0.6003_{\text{\tiny $\pm$0.1345}}$ & $0.6784_{\text{\tiny $\pm$0.0718}}$ & $0.2774_{\text{\tiny $\pm$0.0171}}$ & $0.2030_{\text{\tiny $\pm$0.0050}}$ & $0.2249_{\text{\tiny $\pm$0.0316}}$ \\
& & 1.0 & $0.4528_{\text{\tiny $\pm$0.0411}}$ & $0.6121_{\text{\tiny $\pm$0.0348}}$ & $0.4340_{\text{\tiny $\pm$0.0414}}$ & $0.5822_{\text{\tiny $\pm$0.0721}}$ & $0.6096_{\text{\tiny $\pm$0.1312}}$ & $0.6811_{\text{\tiny $\pm$0.0670}}$ & $0.2763_{\text{\tiny $\pm$0.0200}}$ & $0.2072_{\text{\tiny $\pm$0.0060}}$ & $0.2257_{\text{\tiny $\pm$0.0337}}$ \\
\midrule
\multirow{18}{*}{\textbf{MaskGAE}} 
& \multirow{6}{*}{1} 
& 0.0 & $0.5178_{\text{\tiny $\pm$0.0997}}$ & $0.4981_{\text{\tiny $\pm$0.0711}}$ & $0.5511_{\text{\tiny $\pm$0.0848}}$ & $0.5510_{\text{\tiny $\pm$0.0697}}$ & $0.5585_{\text{\tiny $\pm$0.1608}}$ & $0.6681_{\text{\tiny $\pm$0.1528}}$ & $0.3365_{\text{\tiny $\pm$0.0371}}$ & $0.2726_{\text{\tiny $\pm$0.0319}}$ & $0.2231_{\text{\tiny $\pm$0.0182}}$ \\
& & 0.2 & $0.5065_{\text{\tiny $\pm$0.0971}}$ & $0.4939_{\text{\tiny $\pm$0.0750}}$ & $0.5284_{\text{\tiny $\pm$0.0821}}$ & $0.5501_{\text{\tiny $\pm$0.0663}}$ & $0.5196_{\text{\tiny $\pm$0.1646}}$ & $0.6422_{\text{\tiny $\pm$0.1675}}$ & $0.3272_{\text{\tiny $\pm$0.0392}}$ & $0.2552_{\text{\tiny $\pm$0.0260}}$ & $0.2224_{\text{\tiny $\pm$0.0182}}$ \\
& & 0.4 & $0.4966_{\text{\tiny $\pm$0.0969}}$ & $0.4871_{\text{\tiny $\pm$0.0726}}$ & $0.5143_{\text{\tiny $\pm$0.0791}}$ & $0.5244_{\text{\tiny $\pm$0.1174}}$ & $0.5251_{\text{\tiny $\pm$0.1591}}$ & $0.6403_{\text{\tiny $\pm$0.1602}}$ & $0.3307_{\text{\tiny $\pm$0.0390}}$ & $0.2433_{\text{\tiny $\pm$0.0338}}$ & $0.2209_{\text{\tiny $\pm$0.0199}}$ \\
& & 0.6 & $0.4869_{\text{\tiny $\pm$0.0957}}$ & $0.4872_{\text{\tiny $\pm$0.0698}}$ & $0.4908_{\text{\tiny $\pm$0.0785}}$ & $0.5303_{\text{\tiny $\pm$0.1093}}$ & $0.5399_{\text{\tiny $\pm$0.1554}}$ & $0.6373_{\text{\tiny $\pm$0.1576}}$ & $0.3261_{\text{\tiny $\pm$0.0450}}$ & $0.2324_{\text{\tiny $\pm$0.0375}}$ & $0.2188_{\text{\tiny $\pm$0.0220}}$ \\
& & 0.8 & $0.4705_{\text{\tiny $\pm$0.0985}}$ & $0.4847_{\text{\tiny $\pm$0.0699}}$ & $0.4686_{\text{\tiny $\pm$0.0757}}$ & $0.5134_{\text{\tiny $\pm$0.1204}}$ & $0.5474_{\text{\tiny $\pm$0.1455}}$ & $0.6351_{\text{\tiny $\pm$0.1559}}$ & $0.3155_{\text{\tiny $\pm$0.0499}}$ & $0.2367_{\text{\tiny $\pm$0.0300}}$ & $0.2183_{\text{\tiny $\pm$0.0223}}$ \\
& & 1.0 & $0.4556_{\text{\tiny $\pm$0.0970}}$ & $0.4831_{\text{\tiny $\pm$0.0723}}$ & $0.4601_{\text{\tiny $\pm$0.0731}}$ & $0.5017_{\text{\tiny $\pm$0.1187}}$ & $0.5233_{\text{\tiny $\pm$0.1564}}$ & $0.6365_{\text{\tiny $\pm$0.1619}}$ & $0.3285_{\text{\tiny $\pm$0.0367}}$ & $0.2290_{\text{\tiny $\pm$0.0375}}$ & $0.2188_{\text{\tiny $\pm$0.0238}}$ \\
\cmidrule(lr){2-12}
& \multirow{6}{*}{3} 
& 0.0 & $0.6719_{\text{\tiny $\pm$0.0418}}$ & $0.6028_{\text{\tiny $\pm$0.0394}}$ & $0.6454_{\text{\tiny $\pm$0.0571}}$ & $0.6161_{\text{\tiny $\pm$0.1004}}$ & $0.6806_{\text{\tiny $\pm$0.0749}}$ & $0.7589_{\text{\tiny $\pm$0.0431}}$ & $0.3565_{\text{\tiny $\pm$0.0263}}$ & $0.2917_{\text{\tiny $\pm$0.0275}}$ & $0.2211_{\text{\tiny $\pm$0.0283}}$ \\
& & 0.2 & $0.6256_{\text{\tiny $\pm$0.0479}}$ & $0.5964_{\text{\tiny $\pm$0.0351}}$ & $0.5764_{\text{\tiny $\pm$0.0595}}$ & $0.6027_{\text{\tiny $\pm$0.1023}}$ & $0.6619_{\text{\tiny $\pm$0.0818}}$ & $0.7502_{\text{\tiny $\pm$0.0551}}$ & $0.3496_{\text{\tiny $\pm$0.0281}}$ & $0.2644_{\text{\tiny $\pm$0.0308}}$ & $0.2199_{\text{\tiny $\pm$0.0251}}$ \\
& & 0.4 & $0.5810_{\text{\tiny $\pm$0.0486}}$ & $0.5877_{\text{\tiny $\pm$0.0421}}$ & $0.5399_{\text{\tiny $\pm$0.0531}}$ & $0.6092_{\text{\tiny $\pm$0.0990}}$ & $0.6617_{\text{\tiny $\pm$0.0883}}$ & $0.7454_{\text{\tiny $\pm$0.0597}}$ & $0.3460_{\text{\tiny $\pm$0.0270}}$ & $0.2517_{\text{\tiny $\pm$0.0332}}$ & $0.2183_{\text{\tiny $\pm$0.0252}}$ \\
& & 0.6 & $0.5439_{\text{\tiny $\pm$0.0498}}$ & $0.5826_{\text{\tiny $\pm$0.0486}}$ & $0.5050_{\text{\tiny $\pm$0.0538}}$ & $0.5976_{\text{\tiny $\pm$0.1036}}$ & $0.6728_{\text{\tiny $\pm$0.0837}}$ & $0.7424_{\text{\tiny $\pm$0.0614}}$ & $0.3431_{\text{\tiny $\pm$0.0283}}$ & $0.2572_{\text{\tiny $\pm$0.0246}}$ & $0.2175_{\text{\tiny $\pm$0.0263}}$ \\
& & 0.8 & $0.5020_{\text{\tiny $\pm$0.0542}}$ & $0.5815_{\text{\tiny $\pm$0.0545}}$ & $0.4659_{\text{\tiny $\pm$0.0582}}$ & $0.5943_{\text{\tiny $\pm$0.1036}}$ & $0.6622_{\text{\tiny $\pm$0.0935}}$ & $0.7391_{\text{\tiny $\pm$0.0642}}$ & $0.3443_{\text{\tiny $\pm$0.0254}}$ & $0.2526_{\text{\tiny $\pm$0.0308}}$ & $0.2170_{\text{\tiny $\pm$0.0248}}$ \\
& & 1.0 & $0.4644_{\text{\tiny $\pm$0.0546}}$ & $0.5810_{\text{\tiny $\pm$0.0548}}$ & $0.4379_{\text{\tiny $\pm$0.0607}}$ & $0.5804_{\text{\tiny $\pm$0.1093}}$ & $0.6622_{\text{\tiny $\pm$0.0924}}$ & $0.7396_{\text{\tiny $\pm$0.0682}}$ & $0.3335_{\text{\tiny $\pm$0.0268}}$ & $0.2511_{\text{\tiny $\pm$0.0319}}$ & $0.2194_{\text{\tiny $\pm$0.0248}}$ \\
\cmidrule(lr){2-12}
& \multirow{6}{*}{5} 
& 0.0 & $0.7185_{\text{\tiny $\pm$0.0302}}$ & $0.6412_{\text{\tiny $\pm$0.0290}}$ & $0.6921_{\text{\tiny $\pm$0.0384}}$ & $0.6448_{\text{\tiny $\pm$0.0561}}$ & $0.6997_{\text{\tiny $\pm$0.0689}}$ & $0.7892_{\text{\tiny $\pm$0.0545}}$ & $0.3554_{\text{\tiny $\pm$0.0267}}$ & $0.3005_{\text{\tiny $\pm$0.0270}}$ & $0.2243_{\text{\tiny $\pm$0.0253}}$ \\
& & 0.2 & $0.6735_{\text{\tiny $\pm$0.0392}}$ & $0.6346_{\text{\tiny $\pm$0.0295}}$ & $0.6319_{\text{\tiny $\pm$0.0443}}$ & $0.6505_{\text{\tiny $\pm$0.0682}}$ & $0.7049_{\text{\tiny $\pm$0.0714}}$ & $0.7807_{\text{\tiny $\pm$0.0615}}$ & $0.3515_{\text{\tiny $\pm$0.0299}}$ & $0.2765_{\text{\tiny $\pm$0.0308}}$ & $0.2247_{\text{\tiny $\pm$0.0242}}$ \\
& & 0.4 & $0.6316_{\text{\tiny $\pm$0.0382}}$ & $0.6338_{\text{\tiny $\pm$0.0263}}$ & $0.5704_{\text{\tiny $\pm$0.0449}}$ & $0.6378_{\text{\tiny $\pm$0.0756}}$ & $0.6971_{\text{\tiny $\pm$0.0783}}$ & $0.7763_{\text{\tiny $\pm$0.0632}}$ & $0.3500_{\text{\tiny $\pm$0.0297}}$ & $0.2691_{\text{\tiny $\pm$0.0308}}$ & $0.2203_{\text{\tiny $\pm$0.0280}}$ \\
& & 0.6 & $0.5916_{\text{\tiny $\pm$0.0361}}$ & $0.6313_{\text{\tiny $\pm$0.0325}}$ & $0.5384_{\text{\tiny $\pm$0.0442}}$ & $0.6390_{\text{\tiny $\pm$0.0736}}$ & $0.6849_{\text{\tiny $\pm$0.0736}}$ & $0.7823_{\text{\tiny $\pm$0.0580}}$ & $0.3458_{\text{\tiny $\pm$0.0276}}$ & $0.2617_{\text{\tiny $\pm$0.0335}}$ & $0.2220_{\text{\tiny $\pm$0.0266}}$ \\
& & 0.8 & $0.5469_{\text{\tiny $\pm$0.0387}}$ & $0.6335_{\text{\tiny $\pm$0.0285}}$ & $0.4902_{\text{\tiny $\pm$0.0423}}$ & $0.6359_{\text{\tiny $\pm$0.0555}}$ & $0.6742_{\text{\tiny $\pm$0.0910}}$ & $0.7846_{\text{\tiny $\pm$0.0527}}$ & $0.3465_{\text{\tiny $\pm$0.0273}}$ & $0.2559_{\text{\tiny $\pm$0.0256}}$ & $0.2200_{\text{\tiny $\pm$0.0274}}$ \\
& & 1.0 & $0.5116_{\text{\tiny $\pm$0.0381}}$ & $0.6313_{\text{\tiny $\pm$0.0271}}$ & $0.4522_{\text{\tiny $\pm$0.0443}}$ & $0.6375_{\text{\tiny $\pm$0.0574}}$ & $0.6745_{\text{\tiny $\pm$0.0766}}$ & $0.7692_{\text{\tiny $\pm$0.0637}}$ & $0.3422_{\text{\tiny $\pm$0.0278}}$ & $0.2517_{\text{\tiny $\pm$0.0220}}$ & $0.2229_{\text{\tiny $\pm$0.0250}}$ \\
\bottomrule
\end{tabular}
}
\end{table*}

\begin{table*}[!t]
\centering
\caption{Robustness Evaluation Results Under Edge Drop Perturbation (Accuracy$_{\pm \text{Std}}$)}
\label{tab:robustness_edge_drop}
\resizebox{\textwidth}{!}{
\begin{tabular}{l c c ccccc ccccc}
\toprule
\multirow{2}{*}{\textbf{Model}} & \multirow{2}{*}{\textbf{Shot}} & \multirow{2}{*}{\textbf{Drop}} & \multicolumn{5}{c}{\textbf{Homophilic / Small Heterophilic Datasets}} & \multicolumn{4}{c}{\textbf{Dense Heterophilic Datasets}} \\
\cmidrule(lr){4-8} \cmidrule(lr){9-12}
 &  &  & \textbf{Cora} & \textbf{CiteSeer} & \textbf{PubMed} & \textbf{Cornell} & \textbf{Texas} & \textbf{Wisconsin} & \textbf{Chameleon} & \textbf{Squirrel} & \textbf{Actor} \\
\midrule
\multirow{15}{*}{\textbf{GraphMAE}} 
& \multirow{5}{*}{1} 
& 0.0 & $0.5346_{\text{\tiny $\pm$0.0736}}$ & $0.5208_{\text{\tiny $\pm$0.0662}}$ & $0.5818_{\text{\tiny $\pm$0.0720}}$ & $0.5073_{\text{\tiny $\pm$0.1382}}$ & $0.5339_{\text{\tiny $\pm$0.2014}}$ & $0.6622_{\text{\tiny $\pm$0.1485}}$ & $0.2514_{\text{\tiny $\pm$0.0293}}$ & $0.2183_{\text{\tiny $\pm$0.0200}}$ & $0.2273_{\text{\tiny $\pm$0.0326}}$ \\
& & 0.2 & $0.5065_{\pm 0.0835}$ & $0.5180_{\pm 0.0673}$ & $0.5613_{\pm 0.0729}$ & $0.4958_{\pm 0.1436}$ & $0.5151_{\pm 0.1900}$ & $0.6781_{\pm 0.1337}$ & $0.2490_{\pm 0.0312}$ & $0.2194_{\pm 0.0218}$ & $0.2183_{\pm 0.0365}$ \\
& & 0.4 & $0.4877_{\pm 0.0775}$ & $0.5073_{\pm 0.0669}$ & $0.5629_{\pm 0.0634}$ & $0.4994_{\pm 0.1395}$ & $0.4733_{\pm 0.2009}$ & $0.6581_{\pm 0.1369}$ & $0.2480_{\pm 0.0290}$ & $0.2161_{\pm 0.0211}$ & $0.2236_{\pm 0.0325}$ \\
& & 0.6 & $0.4476_{\pm 0.0741}$ & $0.5140_{\pm 0.0672}$ & $0.5549_{\pm 0.0573}$ & $0.4686_{\pm 0.1470}$ & $0.4968_{\pm 0.2030}$ & $0.6522_{\pm 0.1356}$ & $0.2444_{\pm 0.0320}$ & $0.2127_{\pm 0.0145}$ & $0.2207_{\pm 0.0381}$ \\
& & 0.8 & $0.4230_{\pm 0.0787}$ & $0.5041_{\pm 0.0629}$ & $0.5420_{\pm 0.0659}$ & $0.5045_{\pm 0.1329}$ & $0.4696_{\pm 0.2012}$ & $0.6545_{\pm 0.1377}$ & $0.2397_{\pm 0.0338}$ & $0.2139_{\pm 0.0153}$ & $0.2213_{\pm 0.0365}$ \\
\cmidrule(lr){2-12}
& \multirow{5}{*}{3} 
& 0.0 & $0.6690_{\text{\tiny $\pm$0.0418}}$ & $0.6154_{\text{\tiny $\pm$0.0487}}$ & $0.6610_{\text{\tiny $\pm$0.0545}}$ & $0.6063_{\text{\tiny $\pm$0.0672}}$ & $0.6008_{\text{\tiny $\pm$0.1304}}$ & $0.7656_{\text{\tiny $\pm$0.0940}}$ & $0.2757_{\text{\tiny $\pm$0.0256}}$ & $0.2277_{\text{\tiny $\pm$0.0178}}$ & $0.2309_{\text{\tiny $\pm$0.0252}}$ \\
& & 0.2 & $0.6441_{\pm 0.0412}$ & $0.6141_{\pm 0.0477}$ & $0.6440_{\pm 0.0578}$ & $0.6045_{\pm 0.0813}$ & $0.5778_{\pm 0.1699}$ & $0.7391_{\pm 0.1071}$ & $0.2716_{\pm 0.0315}$ & $0.2250_{\pm 0.0150}$ & $0.2257_{\pm 0.0296}$ \\
& & 0.4 & $0.6240_{\pm 0.0435}$ & $0.6110_{\pm 0.0496}$ & $0.6329_{\pm 0.0564}$ & $0.6039_{\pm 0.0778}$ & $0.5789_{\pm 0.1616}$ & $0.7402_{\pm 0.1050}$ & $0.2683_{\pm 0.0293}$ & $0.2234_{\pm 0.0180}$ & $0.2382_{\pm 0.0231}$ \\
& & 0.6 & $0.5920_{\pm 0.0416}$ & $0.6069_{\pm 0.0507}$ & $0.6242_{\pm 0.0526}$ & $0.6036_{\pm 0.0718}$ & $0.5631_{\pm 0.1697}$ & $0.7515_{\pm 0.1023}$ & $0.2608_{\pm 0.0309}$ & $0.2204_{\pm 0.0189}$ & $0.2403_{\pm 0.0245}$ \\
& & 0.8 & $0.5619_{\pm 0.0464}$ & $0.6072_{\pm 0.0516}$ & $0.6071_{\pm 0.0655}$ & $0.5783_{\pm 0.0957}$ & $0.5658_{\pm 0.1556}$ & $0.7517_{\pm 0.1011}$ & $0.2608_{\pm 0.0309}$ & $0.2155_{\pm 0.0162}$ & $0.2416_{\pm 0.0242}$ \\
\cmidrule(lr){2-12}
& \multirow{5}{*}{5} 
& 0.0 & $0.7332_{\text{\tiny $\pm$0.0307}}$ & $0.6622_{\text{\tiny $\pm$0.0257}}$ & $0.6950_{\text{\tiny $\pm$0.0449}}$ & $0.6679_{\text{\tiny $\pm$0.0833}}$ & $0.6435_{\text{\tiny $\pm$0.1255}}$ & $0.8004_{\text{\tiny $\pm$0.0552}}$ & $0.2914_{\text{\tiny $\pm$0.0236}}$ & $0.2316_{\text{\tiny $\pm$0.0184}}$ & $0.2405_{\text{\tiny $\pm$0.0275}}$ \\
& & 0.2 & $0.7174_{\pm 0.0292}$ & $0.6589_{\pm 0.0261}$ & $0.6824_{\pm 0.0445}$ & $0.6184_{\pm 0.1254}$ & $0.5925_{\pm 0.1443}$ & $0.7884_{\pm 0.0561}$ & $0.2956_{\pm 0.0235}$ & $0.2246_{\pm 0.0171}$ & $0.2428_{\pm 0.0227}$ \\
& & 0.4 & $0.6927_{\pm 0.0372}$ & $0.6550_{\pm 0.0283}$ & $0.6726_{\pm 0.0520}$ & $0.6092_{\pm 0.1018}$ & $0.5759_{\pm 0.1276}$ & $0.7894_{\pm 0.0542}$ & $0.2862_{\pm 0.0189}$ & $0.2253_{\pm 0.0190}$ & $0.2467_{\pm 0.0203}$ \\
& & 0.6 & $0.6596_{\pm 0.0426}$ & $0.6515_{\pm 0.0303}$ & $0.6561_{\pm 0.0568}$ & $0.6454_{\pm 0.0995}$ & $0.6258_{\pm 0.1043}$ & $0.7902_{\pm 0.0572}$ & $0.2897_{\pm 0.0162}$ & $0.2180_{\pm 0.0138}$ & $0.2499_{\pm 0.0213}$ \\
& & 0.8 & $0.6143_{\pm 0.0409}$ & $0.6498_{\pm 0.0286}$ & $0.6359_{\pm 0.0440}$ & $0.6594_{\pm 0.0772}$ & $0.6058_{\pm 0.1561}$ & $0.7857_{\pm 0.0559}$ & $0.2878_{\pm 0.0243}$ & $0.2205_{\pm 0.0133}$ & $0.2488_{\pm 0.0217}$ \\
\midrule
\multirow{15}{*}{\textbf{GraphMAE2}} 
& \multirow{5}{*}{1} 
& 0.0 & $0.4445_{\text{\tiny $\pm$0.0876}}$ & $0.4934_{\text{\tiny $\pm$0.0730}}$ & $0.5388_{\text{\tiny $\pm$0.0814}}$ & $0.5076_{\text{\tiny $\pm$0.1479}}$ & $0.4870_{\text{\tiny $\pm$0.2065}}$ & $0.5683_{\text{\tiny $\pm$0.1676}}$ & $0.2543_{\text{\tiny $\pm$0.0241}}$ & $0.2204_{\text{\tiny $\pm$0.0176}}$ & $0.2234_{\text{\tiny $\pm$0.0254}}$ \\
& & 0.2 & $0.4246_{\pm 0.0814}$ & $0.4901_{\pm 0.0748}$ & $0.5248_{\pm 0.0689}$ & $0.4950_{\pm 0.1434}$ & $0.3431_{\pm 0.1768}$ & $0.5134_{\pm 0.1691}$ & $0.2402_{\pm 0.0292}$ & $0.2128_{\pm 0.0137}$ & $0.2216_{\pm 0.0275}$ \\
& & 0.4 & $0.4202_{\pm 0.0899}$ & $0.4898_{\pm 0.0736}$ & $0.5212_{\pm 0.0792}$ & $0.4986_{\pm 0.1473}$ & $0.3471_{\pm 0.1790}$ & $0.5045_{\pm 0.1661}$ & $0.2409_{\pm 0.0249}$ & $0.2086_{\pm 0.0151}$ & $0.2232_{\pm 0.0268}$ \\
& & 0.6 & $0.3840_{\pm 0.0918}$ & $0.4868_{\pm 0.0800}$ & $0.5011_{\pm 0.0667}$ & $0.5000_{\pm 0.1483}$ & $0.3709_{\pm 0.1812}$ & $0.5189_{\pm 0.1716}$ & $0.2375_{\pm 0.0275}$ & $0.2126_{\pm 0.0128}$ & $0.2206_{\pm 0.0275}$ \\
& & 0.8 & $0.3513_{\pm 0.0787}$ & $0.4879_{\pm 0.0742}$ & $0.5112_{\pm 0.0642}$ & $0.4930_{\pm 0.1482}$ & $0.3897_{\pm 0.1909}$ & $0.5307_{\pm 0.1693}$ & $0.2429_{\pm 0.0262}$ & $0.2139_{\pm 0.0164}$ & $0.2217_{\pm 0.0281}$ \\
\cmidrule(lr){2-12}
& \multirow{5}{*}{3} 
& 0.0 & $0.6015_{\text{\tiny $\pm$0.0531}}$ & $0.5780_{\text{\tiny $\pm$0.0555}}$ & $0.6282_{\text{\tiny $\pm$0.0714}}$ & $0.6000_{\text{\tiny $\pm$0.0889}}$ & $0.6233_{\text{\tiny $\pm$0.1552}}$ & $0.6617_{\text{\tiny $\pm$0.1222}}$ & $0.2854_{\text{\tiny $\pm$0.0233}}$ & $0.2246_{\text{\tiny $\pm$0.0170}}$ & $0.2363_{\text{\tiny $\pm$0.0253}}$ \\
& & 0.2 & $0.5863_{\pm 0.0454}$ & $0.5779_{\pm 0.0545}$ & $0.6174_{\pm 0.0640}$ & $0.5676_{\pm 0.1044}$ & $0.6189_{\pm 0.0907}$ & $0.5957_{\pm 0.1470}$ & $0.2754_{\pm 0.0256}$ & $0.2192_{\pm 0.0139}$ & $0.2313_{\pm 0.0297}$ \\
& & 0.4 & $0.5606_{\pm 0.0476}$ & $0.5730_{\pm 0.0547}$ & $0.5881_{\pm 0.0871}$ & $0.5685_{\pm 0.1092}$ & $0.6253_{\pm 0.1135}$ & $0.6454_{\pm 0.1045}$ & $0.2766_{\pm 0.0229}$ & $0.2218_{\pm 0.0152}$ & $0.2331_{\pm 0.0242}$ \\
& & 0.6 & $0.5421_{\pm 0.0459}$ & $0.5750_{\pm 0.0581}$ & $0.5842_{\pm 0.0818}$ & $0.5759_{\pm 0.0970}$ & $0.5400_{\pm 0.1690}$ & $0.6070_{\pm 0.1511}$ & $0.2709_{\pm 0.0269}$ & $0.2193_{\pm 0.0128}$ & $0.2321_{\pm 0.0246}$ \\
& & 0.8 & $0.5186_{\pm 0.0468}$ & $0.5780_{\pm 0.0570}$ & $0.5987_{\pm 0.0708}$ & $0.5875_{\pm 0.0847}$ & $0.5722_{\pm 0.1478}$ & $0.6007_{\pm 0.1639}$ & $0.2724_{\pm 0.0277}$ & $0.2163_{\pm 0.0115}$ & $0.2349_{\pm 0.0235}$ \\
\cmidrule(lr){2-12}
& \multirow{5}{*}{5} 
& 0.0 & $0.6625_{\text{\tiny $\pm$0.0463}}$ & $0.6235_{\text{\tiny $\pm$0.0296}}$ & $0.6724_{\text{\tiny $\pm$0.0419}}$ & $0.6346_{\text{\tiny $\pm$0.0695}}$ & $0.6594_{\text{\tiny $\pm$0.1040}}$ & $0.7052_{\text{\tiny $\pm$0.0463}}$ & $0.3089_{\text{\tiny $\pm$0.0191}}$ & $0.2275_{\text{\tiny $\pm$0.0170}}$ & $0.2405_{\text{\tiny $\pm$0.0284}}$ \\
& & 0.2 & $0.6454_{\pm 0.0529}$ & $0.6247_{\pm 0.0317}$ & $0.6654_{\pm 0.0535}$ & $0.5879_{\pm 0.0690}$ & $0.5858_{\pm 0.1192}$ & $0.7017_{\pm 0.0595}$ & $0.3049_{\pm 0.0206}$ & $0.2304_{\pm 0.0172}$ & $0.2405_{\pm 0.0220}$ \\
& & 0.4 & $0.6335_{\pm 0.0485}$ & $0.6203_{\pm 0.0307}$ & $0.6562_{\pm 0.0519}$ & $0.5978_{\pm 0.0789}$ & $0.6110_{\pm 0.1039}$ & $0.6965_{\pm 0.0543}$ & $0.2971_{\pm 0.0238}$ & $0.2240_{\pm 0.0134}$ & $0.2440_{\pm 0.0226}$ \\
& & 0.6 & $0.6021_{\pm 0.0468}$ & $0.6207_{\pm 0.0304}$ & $0.6490_{\pm 0.0422}$ & $0.5975_{\pm 0.0863}$ & $0.6345_{\pm 0.1092}$ & $0.7021_{\pm 0.0482}$ & $0.2996_{\pm 0.0189}$ & $0.2238_{\pm 0.0156}$ & $0.2388_{\pm 0.0255}$ \\
& & 0.8 & $0.5784_{\pm 0.0466}$ & $0.6179_{\pm 0.0352}$ & $0.6416_{\pm 0.0510}$ & $0.6063_{\pm 0.0876}$ & $0.6264_{\pm 0.1275}$ & $0.7023_{\pm 0.0496}$ & $0.3005_{\pm 0.0217}$ & $0.2202_{\pm 0.0115}$ & $0.2448_{\pm 0.0237}$ \\
\midrule
\multirow{15}{*}{\textbf{MaskGAE}} 
& \multirow{5}{*}{1} 
& 0.0 & $0.5178_{\text{\tiny $\pm$0.0997}}$ & $0.4981_{\text{\tiny $\pm$0.0711}}$ & $0.5511_{\text{\tiny $\pm$0.0848}}$ & $0.5510_{\text{\tiny $\pm$0.0697}}$ & $0.5585_{\text{\tiny $\pm$0.1608}}$ & $0.6681_{\text{\tiny $\pm$0.1528}}$ & $0.3365_{\text{\tiny $\pm$0.0371}}$ & $0.2726_{\text{\tiny $\pm$0.0319}}$ & $0.2231_{\text{\tiny $\pm$0.0182}}$ \\
& & 0.2 & $0.5022_{\pm 0.1015}$ & $0.4927_{\pm 0.0733}$ & $0.5485_{\pm 0.0831}$ & $0.5409_{\pm 0.0970}$ & $0.5130_{\pm 0.1596}$ & $0.6501_{\pm 0.1677}$ & $0.3215_{\pm 0.0392}$ & $0.2020_{\pm 0.0091}$ & $0.2226_{\pm 0.0185}$ \\
& & 0.4 & $0.5004_{\pm 0.0925}$ & $0.4939_{\pm 0.0750}$ & $0.5462_{\pm 0.0855}$ & $0.5361_{\pm 0.1096}$ & $0.5296_{\pm 0.1600}$ & $0.6431_{\pm 0.1667}$ & $0.3187_{\pm 0.0455}$ & $0.2063_{\pm 0.0126}$ & $0.2214_{\pm 0.0219}$ \\
& & 0.6 & $0.4973_{\pm 0.0884}$ & $0.4892_{\pm 0.0706}$ & $0.5469_{\pm 0.0798}$ & $0.5210_{\pm 0.1283}$ & $0.5143_{\pm 0.1675}$ & $0.6446_{\pm 0.1631}$ & $0.3207_{\pm 0.0462}$ & $0.2054_{\pm 0.0132}$ & $0.2227_{\pm 0.0185}$ \\
& & 0.8 & $0.4891_{\pm 0.0829}$ & $0.4934_{\pm 0.0685}$ & $0.5454_{\pm 0.0881}$ & $0.5339_{\pm 0.0989}$ & $0.5243_{\pm 0.1430}$ & $0.6572_{\pm 0.1607}$ & $0.3253_{\pm 0.0342}$ & $0.2041_{\pm 0.0082}$ & $0.2222_{\pm 0.0186}$ \\
\cmidrule(lr){2-12}
& \multirow{5}{*}{3} 
& 0.0 & $0.6719_{\text{\tiny $\pm$0.0418}}$ & $0.6028_{\text{\tiny $\pm$0.0394}}$ & $0.6454_{\text{\tiny $\pm$0.0571}}$ & $0.6161_{\text{\tiny $\pm$0.1004}}$ & $0.6806_{\text{\tiny $\pm$0.0749}}$ & $0.7589_{\text{\tiny $\pm$0.0431}}$ & $0.3565_{\text{\tiny $\pm$0.0263}}$ & $0.2917_{\text{\tiny $\pm$0.0275}}$ & $0.2211_{\text{\tiny $\pm$0.0283}}$ \\
& & 0.2 & $0.6591_{\pm 0.0450}$ & $0.5983_{\pm 0.0381}$ & $0.6416_{\pm 0.0662}$ & $0.6170_{\pm 0.1084}$ & $0.6394_{\pm 0.1009}$ & $0.7511_{\pm 0.0602}$ & $0.3339_{\pm 0.0282}$ & $0.2109_{\pm 0.0142}$ & $0.2205_{\pm 0.0238}$ \\
& & 0.4 & $0.6460_{\pm 0.0437}$ & $0.5950_{\pm 0.0394}$ & $0.6382_{\pm 0.0639}$ & $0.6036_{\pm 0.1071}$ & $0.6514_{\pm 0.0952}$ & $0.7485_{\pm 0.0690}$ & $0.3399_{\pm 0.0254}$ & $0.2073_{\pm 0.0138}$ & $0.2206_{\pm 0.0256}$ \\
& & 0.6 & $0.6229_{\pm 0.0503}$ & $0.5903_{\pm 0.0361}$ & $0.6398_{\pm 0.0641}$ & $0.5943_{\pm 0.1154}$ & $0.6564_{\pm 0.1097}$ & $0.7504_{\pm 0.0716}$ & $0.3405_{\pm 0.0218}$ & $0.2040_{\pm 0.0090}$ & $0.2165_{\pm 0.0267}$ \\
& & 0.8 & $0.6063_{\pm 0.0422}$ & $0.5924_{\pm 0.0394}$ & $0.6319_{\pm 0.0715}$ & $0.6116_{\pm 0.1072}$ & $0.6517_{\pm 0.0945}$ & $0.7444_{\pm 0.0721}$ & $0.3416_{\pm 0.0232}$ & $0.2061_{\pm 0.0101}$ & $0.2206_{\pm 0.0273}$ \\
\cmidrule(lr){2-12}
& \multirow{5}{*}{5} 
& 0.0 & $0.7185_{\text{\tiny $\pm$0.0302}}$ & $0.6412_{\text{\tiny $\pm$0.0290}}$ & $0.6921_{\text{\tiny $\pm$0.0384}}$ & $0.6448_{\text{\tiny $\pm$0.0561}}$ & $0.6997_{\text{\tiny $\pm$0.0689}}$ & $0.7892_{\text{\tiny $\pm$0.0545}}$ & $0.3554_{\text{\tiny $\pm$0.0267}}$ & $0.3005_{\text{\tiny $\pm$0.0270}}$ & $0.2243_{\text{\tiny $\pm$0.0253}}$ \\
& & 0.2 & $0.6941_{\pm 0.0272}$ & $0.6094_{\pm 0.0293}$ & $0.6681_{\pm 0.0442}$ & $0.5692_{\pm 0.1008}$ & $0.6049_{\pm 0.0973}$ & $0.7276_{\pm 0.0532}$ & $0.3499_{\pm 0.0312}$ & $0.2893_{\pm 0.0355}$ & $0.2191_{\pm 0.0230}$ \\
& & 0.4 & $0.6557_{\pm 0.0385}$ & $0.5612_{\pm 0.0311}$ & $0.6491_{\pm 0.0452}$ & $0.5292_{\pm 0.1172}$ & $0.5919_{\pm 0.1245}$ & $0.6971_{\pm 0.0955}$ & $0.3410_{\pm 0.0335}$ & $0.2926_{\pm 0.0276}$ & $0.2087_{\pm 0.0224}$ \\
& & 0.6 & $0.5807_{\pm 0.0378}$ & $0.4733_{\pm 0.0545}$ & $0.6144_{\pm 0.0640}$ & $0.4540_{\pm 0.1221}$ & $0.5426_{\pm 0.1412}$ & $0.5969_{\pm 0.0704}$ & $0.3104_{\pm 0.0459}$ & $0.3009_{\pm 0.0320}$ & $0.2037_{\pm 0.0181}$ \\
& & 0.8 & $0.4922_{\pm 0.0336}$ & $0.3064_{\pm 0.0436}$ & $0.5761_{\pm 0.0446}$ & $0.3762_{\pm 0.1602}$ & $0.5264_{\pm 0.1993}$ & $0.4593_{\pm 0.1279}$ & $0.3094_{\pm 0.0276}$ & $0.2884_{\pm 0.0356}$ & $0.1968_{\pm 0.0242}$ \\
\bottomrule
\end{tabular}
}
\end{table*}

\begin{table*}[!t]
\centering
\caption{Robustness Evaluation Results Under Feature Mask Ratio (Accuracy$_{\pm \text{Std}}$)}
\label{tab:robustness_feat_mask}
\resizebox{\textwidth}{!}{
\begin{tabular}{l c c ccccc ccccc}
\toprule
\multirow{2}{*}{\textbf{Model}} & \multirow{2}{*}{\textbf{Shot}} & \multirow{2}{*}{\textbf{Mask}} & \multicolumn{5}{c}{\textbf{Homophilic / Small Heterophilic Datasets}} & \multicolumn{4}{c}{\textbf{Dense Heterophilic Datasets}} \\
\cmidrule(lr){4-8} \cmidrule(lr){9-12}
 &  &  & \textbf{Cora} & \textbf{CiteSeer} & \textbf{PubMed} & \textbf{Cornell} & \textbf{Texas} & \textbf{Wisconsin} & \textbf{Chameleon} & \textbf{Squirrel} & \textbf{Actor} \\
\midrule
\multirow{15}{*}{\textbf{GraphMAE}} 
& \multirow{5}{*}{1} 
& 0.0 & $0.5346_{\text{\tiny $\pm$0.0736}}$ & $0.5208_{\text{\tiny $\pm$0.0662}}$ & $0.5818_{\text{\tiny $\pm$0.0720}}$ & $0.5073_{\text{\tiny $\pm$0.1382}}$ & $0.5339_{\text{\tiny $\pm$0.2014}}$ & $0.6622_{\text{\tiny $\pm$0.1485}}$ & $0.2514_{\text{\tiny $\pm$0.0293}}$ & $0.2183_{\text{\tiny $\pm$0.0200}}$ & $0.2273_{\text{\tiny $\pm$0.0326}}$ \\
& & 0.2 & $0.4980_{\pm 0.0860}$ & $0.4686_{\pm 0.0615}$ & $0.5625_{\pm 0.0799}$ & $0.4280_{\pm 0.1586}$ & $0.4690_{\pm 0.2376}$ & $0.5809_{\pm 0.1578}$ & $0.2581_{\pm 0.0282}$ & $0.2057_{\pm 0.0114}$ & $0.2026_{\pm 0.0473}$ \\
& & 0.4 & $0.4700_{\pm 0.0711}$ & $0.4260_{\pm 0.0692}$ & $0.5478_{\pm 0.0803}$ & $0.3787_{\pm 0.1403}$ & $0.3772_{\pm 0.1937}$ & $0.5431_{\pm 0.1435}$ & $0.2598_{\pm 0.0313}$ & $0.2026_{\pm 0.0065}$ & $0.1996_{\pm 0.0308}$ \\
& & 0.6 & $0.4299_{\pm 0.0739}$ & $0.3563_{\pm 0.0770}$ & $0.5308_{\pm 0.0846}$ & $0.3591_{\pm 0.1523}$ & $0.4312_{\pm 0.1848}$ & $0.4078_{\pm 0.1130}$ & $0.2573_{\pm 0.0357}$ & $0.2009_{\pm 0.0061}$ & $0.1962_{\pm 0.0406}$ \\
& & 0.8 & $0.3581_{\pm 0.0703}$ & $0.2346_{\pm 0.0516}$ & $0.5086_{\pm 0.0822}$ & $0.2759_{\pm 0.1461}$ & $0.3265_{\pm 0.2133}$ & $0.2672_{\pm 0.1515}$ & $0.2536_{\pm 0.0344}$ & $0.2037_{\pm 0.0088}$ & $0.2106_{\pm 0.0286}$ \\
\cmidrule(lr){2-12}
& \multirow{5}{*}{3} 
& 0.0 & $0.6690_{\text{\tiny $\pm$0.0418}}$ & $0.6154_{\text{\tiny $\pm$0.0487}}$ & $0.6610_{\text{\tiny $\pm$0.0545}}$ & $0.6063_{\text{\tiny $\pm$0.0672}}$ & $0.6008_{\text{\tiny $\pm$0.1304}}$ & $0.7656_{\text{\tiny $\pm$0.0940}}$ & $0.2757_{\text{\tiny $\pm$0.0256}}$ & $0.2277_{\text{\tiny $\pm$0.0178}}$ & $0.2309_{\text{\tiny $\pm$0.0252}}$ \\
& & 0.2 & $0.6429_{\pm 0.0438}$ & $0.5853_{\pm 0.0506}$ & $0.6351_{\pm 0.0593}$ & $0.5545_{\pm 0.0987}$ & $0.5308_{\pm 0.1348}$ & $0.7165_{\pm 0.0787}$ & $0.2578_{\pm 0.0275}$ & $0.2133_{\pm 0.0125}$ & $0.2207_{\pm 0.0218}$ \\
& & 0.4 & $0.6233_{\pm 0.0416}$ & $0.5343_{\pm 0.0500}$ & $0.6376_{\pm 0.0625}$ & $0.5080_{\pm 0.1071}$ & $0.4922_{\pm 0.1859}$ & $0.7104_{\pm 0.0633}$ & $0.2622_{\pm 0.0246}$ & $0.2060_{\pm 0.0084}$ & $0.2220_{\pm 0.0172}$ \\
& & 0.6 & $0.5577_{\pm 0.0469}$ & $0.4463_{\pm 0.0672}$ & $0.6218_{\pm 0.0581}$ & $0.4140_{\pm 0.1125}$ & $0.4864_{\pm 0.1732}$ & $0.6004_{\pm 0.0947}$ & $0.2751_{\pm 0.0280}$ & $0.2040_{\pm 0.0082}$ & $0.2169_{\pm 0.0180}$ \\
& & 0.8 & $0.4751_{\pm 0.0417}$ & $0.2964_{\pm 0.0463}$ & $0.5863_{\pm 0.0574}$ & $0.3625_{\pm 0.1598}$ & $0.3494_{\pm 0.1824}$ & $0.4313_{\pm 0.1124}$ & $0.2724_{\pm 0.0331}$ & $0.2080_{\pm 0.0092}$ & $0.2087_{\pm 0.0318}$ \\
\cmidrule(lr){2-12}
& \multirow{5}{*}{5} 
& 0.0 & $0.7332_{\text{\tiny $\pm$0.0307}}$ & $0.6622_{\text{\tiny $\pm$0.0257}}$ & $0.6950_{\text{\tiny $\pm$0.0449}}$ & $0.6679_{\text{\tiny $\pm$0.0833}}$ & $0.6435_{\text{\tiny $\pm$0.1255}}$ & $0.8004_{\text{\tiny $\pm$0.0552}}$ & $0.2914_{\text{\tiny $\pm$0.0236}}$ & $0.2316_{\text{\tiny $\pm$0.0184}}$ & $0.2405_{\text{\tiny $\pm$0.0275}}$ \\
& & 0.2 & $0.7169_{\pm 0.0339}$ & $0.6410_{\pm 0.0238}$ & $0.6696_{\pm 0.0508}$ & $0.5937_{\pm 0.1100}$ & $0.5510_{\pm 0.1534}$ & $0.7547_{\pm 0.0462}$ & $0.2813_{\pm 0.0220}$ & $0.2186_{\pm 0.0152}$ & $0.2234_{\pm 0.0168}$ \\
& & 0.4 & $0.6869_{\pm 0.0371}$ & $0.5917_{\pm 0.0294}$ & $0.6666_{\pm 0.0501}$ & $0.5152_{\pm 0.1741}$ & $0.4832_{\pm 0.1957}$ & $0.7231_{\pm 0.0687}$ & $0.2789_{\pm 0.0227}$ & $0.2067_{\pm 0.0084}$ & $0.2231_{\pm 0.0197}$ \\
& & 0.6 & $0.6412_{\pm 0.0352}$ & $0.5058_{\pm 0.0475}$ & $0.6520_{\pm 0.0450}$ & $0.3844_{\pm 0.1703}$ & $0.4986_{\pm 0.2003}$ & $0.6239_{\pm 0.0667}$ & $0.2943_{\pm 0.0256}$ & $0.2055_{\pm 0.0075}$ & $0.2122_{\pm 0.0259}$ \\
& & 0.8 & $0.5693_{\pm 0.0407}$ & $0.3571_{\pm 0.0418}$ & $0.6298_{\pm 0.0505}$ & $0.2987_{\pm 0.1830}$ & $0.3035_{\pm 0.1716}$ & $0.4852_{\pm 0.0972}$ & $0.3008_{\pm 0.0292}$ & $0.2070_{\pm 0.0096}$ & $0.2150_{\pm 0.0215}$ \\
\midrule
\multirow{15}{*}{\textbf{GraphMAE2}} 
& \multirow{5}{*}{1} 
& 0.0 & $0.4445_{\text{\tiny $\pm$0.0876}}$ & $0.4934_{\text{\tiny $\pm$0.0730}}$ & $0.5388_{\text{\tiny $\pm$0.0814}}$ & $0.5076_{\text{\tiny $\pm$0.1479}}$ & $0.4870_{\text{\tiny $\pm$0.2065}}$ & $0.5683_{\text{\tiny $\pm$0.1676}}$ & $0.2543_{\text{\tiny $\pm$0.0241}}$ & $0.2204_{\text{\tiny $\pm$0.0176}}$ & $0.2234_{\text{\tiny $\pm$0.0254}}$ \\
& & 0.2 & $0.3781_{\pm 0.0821}$ & $0.4550_{\pm 0.0834}$ & $0.5321_{\pm 0.0713}$ & $0.4465_{\pm 0.1526}$ & $0.3944_{\pm 0.2055}$ & $0.4941_{\pm 0.1572}$ & $0.2448_{\pm 0.0331}$ & $0.2089_{\pm 0.0161}$ & $0.2259_{\pm 0.0319}$ \\
& & 0.4 & $0.3383_{\pm 0.0982}$ & $0.4001_{\pm 0.0776}$ & $0.4914_{\pm 0.0855}$ & $0.3936_{\pm 0.1462}$ & $0.3788_{\pm 0.1927}$ & $0.4444_{\pm 0.1345}$ & $0.2510_{\pm 0.0337}$ & $0.2069_{\pm 0.0121}$ & $0.2204_{\pm 0.0330}$ \\
& & 0.6 & $0.2965_{\pm 0.0771}$ & $0.3203_{\pm 0.0678}$ & $0.4741_{\pm 0.0615}$ & $0.3711_{\pm 0.1503}$ & $0.4320_{\pm 0.2098}$ & $0.4239_{\pm 0.1464}$ & $0.2502_{\pm 0.0289}$ & $0.2030_{\pm 0.0081}$ & $0.2179_{\pm 0.0292}$ \\
& & 0.8 & $0.2458_{\pm 0.0537}$ & $0.2463_{\pm 0.0439}$ & $0.4303_{\pm 0.0657}$ & $0.3179_{\pm 0.1528}$ & $0.2995_{\pm 0.1695}$ & $0.3116_{\pm 0.1274}$ & $0.2390_{\pm 0.0295}$ & $0.2040_{\pm 0.0107}$ & $0.2029_{\pm 0.0277}$ \\
\cmidrule(lr){2-12}
& \multirow{5}{*}{3} 
& 0.0 & $0.6015_{\text{\tiny $\pm$0.0531}}$ & $0.5780_{\text{\tiny $\pm$0.0555}}$ & $0.6282_{\text{\tiny $\pm$0.0714}}$ & $0.6000_{\text{\tiny $\pm$0.0889}}$ & $0.6233_{\text{\tiny $\pm$0.1552}}$ & $0.6617_{\text{\tiny $\pm$0.1222}}$ & $0.2854_{\text{\tiny $\pm$0.0233}}$ & $0.2246_{\text{\tiny $\pm$0.0170}}$ & $0.2363_{\text{\tiny $\pm$0.0253}}$ \\
& & 0.2 & $0.5499_{\pm 0.0610}$ & $0.5464_{\pm 0.0616}$ & $0.5649_{\pm 0.1093}$ & $0.5607_{\pm 0.1150}$ & $0.4344_{\pm 0.1411}$ & $0.5769_{\pm 0.1491}$ & $0.2789_{\pm 0.0255}$ & $0.2176_{\pm 0.0175}$ & $0.2205_{\pm 0.0293}$ \\
& & 0.4 & $0.5075_{\pm 0.0545}$ & $0.5004_{\pm 0.0529}$ & $0.6167_{\pm 0.0688}$ & $0.4702_{\pm 0.1082}$ & $0.3878_{\pm 0.1786}$ & $0.4928_{\pm 0.1422}$ & $0.2759_{\pm 0.0254}$ & $0.2132_{\pm 0.0131}$ & $0.2196_{\pm 0.0294}$ \\
& & 0.6 & $0.4452_{\pm 0.0557}$ & $0.4369_{\pm 0.0533}$ & $0.5593_{\pm 0.0908}$ & $0.4101_{\pm 0.1364}$ & $0.3764_{\pm 0.1501}$ & $0.4609_{\pm 0.1205}$ & $0.2750_{\pm 0.0349}$ & $0.2075_{\pm 0.0108}$ & $0.2142_{\pm 0.0267}$ \\
& & 0.8 & $0.3420_{\pm 0.0473}$ & $0.2990_{\pm 0.0409}$ & $0.4767_{\pm 0.1242}$ & $0.3515_{\pm 0.1432}$ & $0.2969_{\pm 0.1788}$ & $0.3506_{\pm 0.1418}$ & $0.2644_{\pm 0.0248}$ & $0.2062_{\pm 0.0078}$ & $0.2136_{\pm 0.0322}$ \\
\cmidrule(lr){2-12}
& \multirow{5}{*}{5} 
& 0.0 & $0.6625_{\text{\tiny $\pm$0.0463}}$ & $0.6235_{\text{\tiny $\pm$0.0296}}$ & $0.6724_{\text{\tiny $\pm$0.0419}}$ & $0.6346_{\text{\tiny $\pm$0.0695}}$ & $0.6594_{\text{\tiny $\pm$0.1040}}$ & $0.7052_{\text{\tiny $\pm$0.0463}}$ & $0.3089_{\text{\tiny $\pm$0.0191}}$ & $0.2275_{\text{\tiny $\pm$0.0170}}$ & $0.2405_{\text{\tiny $\pm$0.0284}}$ \\
& & 0.2 & $0.6434_{\pm 0.0524}$ & $0.5819_{\pm 0.0433}$ & $0.6490_{\pm 0.0516}$ & $0.5902_{\pm 0.0857}$ & $0.5159_{\pm 0.1330}$ & $0.6102_{\pm 0.0692}$ & $0.3049_{\pm 0.0219}$ & $0.2197_{\pm 0.0163}$ & $0.2312_{\pm 0.0226}$ \\
& & 0.4 & $0.6054_{\pm 0.0421}$ & $0.5331_{\pm 0.0383}$ & $0.6433_{\pm 0.0473}$ & $0.5298_{\pm 0.0912}$ & $0.4501_{\pm 0.1444}$ & $0.5728_{\pm 0.0838}$ & $0.3010_{\pm 0.0211}$ & $0.2197_{\pm 0.0152}$ & $0.2271_{\pm 0.0244}$ \\
& & 0.6 & $0.5909_{\pm 0.0431}$ & $0.4730_{\pm 0.0467}$ & $0.6313_{\pm 0.0401}$ & $0.4384_{\pm 0.1278}$ & $0.4374_{\pm 0.1630}$ & $0.5229_{\pm 0.0912}$ & $0.2946_{\pm 0.0217}$ & $0.2138_{\pm 0.0125}$ & $0.2216_{\pm 0.0252}$ \\
& & 0.8 & $0.5222_{\pm 0.0381}$ & $0.3254_{\pm 0.0415}$ & $0.6024_{\pm 0.0481}$ & $0.3505_{\pm 0.1363}$ & $0.3774_{\pm 0.1791}$ & $0.3709_{\pm 0.1014}$ & $0.2794_{\pm 0.0299}$ & $0.2104_{\pm 0.0114}$ & $0.2122_{\pm 0.0229}$ \\
\midrule
\multirow{15}{*}{\textbf{MaskGAE}} 
& \multirow{5}{*}{1} 
& 0.0 & $0.5178_{\text{\tiny $\pm$0.0997}}$ & $0.4981_{\text{\tiny $\pm$0.0711}}$ & $0.5511_{\text{\tiny $\pm$0.0848}}$ & $0.5510_{\text{\tiny $\pm$0.0697}}$ & $0.5585_{\text{\tiny $\pm$0.1608}}$ & $0.6681_{\text{\tiny $\pm$0.1528}}$ & $0.3365_{\text{\tiny $\pm$0.0371}}$ & $0.2726_{\text{\tiny $\pm$0.0319}}$ & $0.2231_{\text{\tiny $\pm$0.0182}}$ \\
& & 0.2 & $0.4393_{\pm 0.0967}$ & $0.4583_{\pm 0.0651}$ & $0.5303_{\pm 0.0807}$ & $0.4840_{\pm 0.1217}$ & $0.5153_{\pm 0.1766}$ & $0.6018_{\pm 0.1342}$ & $0.3206_{\pm 0.0427}$ & $0.2483_{\pm 0.0484}$ & $0.2057_{\pm 0.0233}$ \\
& & 0.4 & $0.3929_{\pm 0.0745}$ & $0.4081_{\pm 0.0774}$ & $0.5238_{\pm 0.0766}$ & $0.4933_{\pm 0.1154}$ & $0.4582_{\pm 0.2110}$ & $0.5651_{\pm 0.1153}$ & $0.3074_{\pm 0.0358}$ & $0.2660_{\pm 0.0369}$ & $0.1984_{\pm 0.0249}$ \\
& & 0.6 & $0.3251_{\pm 0.0694}$ & $0.3450_{\pm 0.0615}$ & $0.4940_{\pm 0.0724}$ & $0.4025_{\pm 0.1338}$ & $0.4193_{\pm 0.1966}$ & $0.5048_{\pm 0.1131}$ & $0.2946_{\pm 0.0330}$ & $0.2615_{\pm 0.0372}$ & $0.1867_{\pm 0.0283}$ \\
& & 0.8 & $0.2410_{\pm 0.0583}$ & $0.2403_{\pm 0.0368}$ & $0.4632_{\pm 0.0544}$ & $0.3395_{\pm 0.1396}$ & $0.3534_{\pm 0.2024}$ & $0.3588_{\pm 0.1001}$ & $0.2722_{\pm 0.0389}$ & $0.2537_{\pm 0.0397}$ & $0.1866_{\pm 0.0317}$ \\
\cmidrule(lr){2-12}
& \multirow{5}{*}{3} 
& 0.0 & $0.6719_{\text{\tiny $\pm$0.0418}}$ & $0.6028_{\text{\tiny $\pm$0.0394}}$ & $0.6454_{\text{\tiny $\pm$0.0571}}$ & $0.6161_{\text{\tiny $\pm$0.1004}}$ & $0.6806_{\text{\tiny $\pm$0.0749}}$ & $0.7589_{\text{\tiny $\pm$0.0431}}$ & $0.3565_{\text{\tiny $\pm$0.0263}}$ & $0.2917_{\text{\tiny $\pm$0.0275}}$ & $0.2211_{\text{\tiny $\pm$0.0283}}$ \\
& & 0.2 & $0.6357_{\pm 0.0503}$ & $0.5668_{\pm 0.0486}$ & $0.6209_{\pm 0.0666}$ & $0.5470_{\pm 0.0959}$ & $0.5733_{\pm 0.1405}$ & $0.7024_{\pm 0.0864}$ & $0.3492_{\pm 0.0324}$ & $0.2729_{\pm 0.0425}$ & $0.2077_{\pm 0.0312}$ \\
& & 0.4 & $0.6015_{\pm 0.0482}$ & $0.5153_{\pm 0.0476}$ & $0.6035_{\pm 0.0598}$ & $0.5196_{\pm 0.1111}$ & $0.5569_{\pm 0.1269}$ & $0.6491_{\pm 0.0970}$ & $0.3335_{\pm 0.0377}$ & $0.2808_{\pm 0.0420}$ & $0.1995_{\pm 0.0250}$ \\
& & 0.6 & $0.5244_{\pm 0.0502}$ & $0.4329_{\pm 0.0635}$ & $0.5679_{\pm 0.0673}$ & $0.4548_{\pm 0.1126}$ & $0.5219_{\pm 0.1490}$ & $0.5807_{\pm 0.0821}$ & $0.3000_{\pm 0.0470}$ & $0.2779_{\pm 0.0424}$ & $0.1971_{\pm 0.0220}$ \\
& & 0.8 & $0.4307_{\pm 0.0485}$ & $0.2821_{\pm 0.0472}$ & $0.5402_{\pm 0.0489}$ & $0.3366_{\pm 0.1275}$ & $0.3203_{\pm 0.2074}$ & $0.4091_{\pm 0.1027}$ & $0.3038_{\pm 0.0445}$ & $0.2858_{\pm 0.0385}$ & $0.1813_{\pm 0.0301}$ \\
\cmidrule(lr){2-12}
& \multirow{5}{*}{5} 
& 0.0 & $0.7185_{\text{\tiny $\pm$0.0302}}$ & $0.6412_{\text{\tiny $\pm$0.0290}}$ & $0.6921_{\text{\tiny $\pm$0.0384}}$ & $0.6448_{\text{\tiny $\pm$0.0561}}$ & $0.6997_{\text{\tiny $\pm$0.0689}}$ & $0.7892_{\text{\tiny $\pm$0.0545}}$ & $0.3554_{\text{\tiny $\pm$0.0267}}$ & $0.3005_{\text{\tiny $\pm$0.0270}}$ & $0.2243_{\text{\tiny $\pm$0.0253}}$ \\
& & 0.2 & $0.6941_{\pm 0.0272}$ & $0.6094_{\pm 0.0293}$ & $0.6681_{\pm 0.0442}$ & $0.5692_{\pm 0.1008}$ & $0.6049_{\pm 0.0973}$ & $0.7276_{\pm 0.0532}$ & $0.3499_{\pm 0.0312}$ & $0.2893_{\pm 0.0355}$ & $0.2191_{\pm 0.0230}$ \\
& & 0.4 & $0.6557_{\pm 0.0385}$ & $0.5612_{\pm 0.0311}$ & $0.6491_{\pm 0.0452}$ & $0.5292_{\pm 0.1172}$ & $0.5919_{\pm 0.1245}$ & $0.6971_{\pm 0.0955}$ & $0.3410_{\pm 0.0335}$ & $0.2926_{\pm 0.0276}$ & $0.2087_{\pm 0.0224}$ \\
& & 0.6 & $0.5807_{\pm 0.0378}$ & $0.4733_{\pm 0.0545}$ & $0.6144_{\pm 0.0640}$ & $0.4540_{\pm 0.1221}$ & $0.5426_{\pm 0.1412}$ & $0.5969_{\pm 0.0704}$ & $0.3104_{\pm 0.0459}$ & $0.3009_{\pm 0.0320}$ & $0.2037_{\pm 0.0181}$ \\
& & 0.8 & $0.4922_{\pm 0.0336}$ & $0.3064_{\pm 0.0436}$ & $0.5761_{\pm 0.0446}$ & $0.3762_{\pm 0.1602}$ & $0.5264_{\pm 0.1993}$ & $0.4593_{\pm 0.1279}$ & $0.3094_{\pm 0.0276}$ & $0.2884_{\pm 0.0356}$ & $0.1968_{\pm 0.0242}$ \\
\bottomrule
\end{tabular}
}
\end{table*}

\subsubsection{Graph Prompt Learning Baselines}

\paragraph{GPF/GPF+~\cite{fang2023gpf}}GPF adds one shared learnable prompt vector to every node feature, whereas GPF+ forms a node-dependent prompt from learnable basis prompts using soft attention. Both optimize the prompt parameters while keeping the pretrained encoder frozen.

\begin{table*}[!t]
\centering
\caption{Ablation Study Results Across Different Base Models and Shots. Accuracy is presented as the main value, with standard deviation in the subscript ($Acc_{\pm Std}$).}
\label{tab:ablation_study_full}
\resizebox{\textwidth}{!}{
\begin{tabular}{l c l ccc ccc ccc}
\toprule
\textbf{Model} & \textbf{Shot} & \textbf{Setting} & \textbf{Cora} & \textbf{CiteSeer} & \textbf{PubMed} & \textbf{Cornell} & \textbf{Texas} & \textbf{Wisconsin} & \textbf{Chameleon} & \textbf{Squirrel} & \textbf{Actor} \\
\midrule
\multirow{18}{*}{\textbf{GraphMAE}} 
& \multirow{6}{*}{1} 
& Full MINT  & $0.5346_{\pm 0.0736}$ & $0.5208_{\pm 0.0662}$ & $0.5818_{\pm 0.0720}$ & $0.5073_{\pm 0.1382}$ & $0.5339_{\pm 0.2014}$ & $0.6622_{\pm 0.1485}$ & $0.2514_{\pm 0.0293}$ & $0.2183_{\pm 0.0200}$ & $0.2273_{\pm 0.0326}$ \\
& & w/o SA     & $0.5303_{\pm 0.0747}$ & $0.5158_{\pm 0.0640}$ & $0.5791_{\pm 0.0705}$ & $0.5076_{\pm 0.1383}$ & $0.4947_{\pm 0.1952}$ & $0.6665_{\pm 0.1440}$ & $0.2496_{\pm 0.0316}$ & $0.2116_{\pm 0.0197}$ & $0.2161_{\pm 0.0279}$ \\
& & w/o OT     & $0.3192_{\pm 0.0667}$ & $0.3484_{\pm 0.0753}$ & $0.4277_{\pm 0.0715}$ & $0.3966_{\pm 0.2201}$ & $0.4810_{\pm 0.2390}$ & $0.5765_{\pm 0.1709}$ & $0.2362_{\pm 0.0326}$ & $0.2039_{\pm 0.0114}$ & $0.2087_{\pm 0.0382}$ \\
& & w/o FA     & $0.5281_{\pm 0.0686}$ & $0.5091_{\pm 0.0710}$ & $0.5703_{\pm 0.0722}$ & $0.4969_{\pm 0.1414}$ & $0.4810_{\pm 0.1931}$ & $0.6567_{\pm 0.1530}$ & $0.2466_{\pm 0.0274}$ & $0.2060_{\pm 0.0142}$ & $0.2151_{\pm 0.0378}$ \\
& & w/o kNN    & $0.5101_{\pm 0.0673}$ & $0.4574_{\pm 0.0787}$ & $0.5676_{\pm 0.0682}$ & $0.2880_{\pm 0.0717}$ & $0.3402_{\pm 0.2129}$ & $0.3599_{\pm 0.1304}$ & $0.2377_{\pm 0.0338}$ & $0.2060_{\pm 0.0148}$ & $0.2076_{\pm 0.0301}$ \\
& & w/o Reg    & $0.5337_{\pm 0.0754}$ & $0.5201_{\pm 0.0654}$ & $0.5807_{\pm 0.0718}$ & $0.4784_{\pm 0.1356}$ & $0.5087_{\pm 0.2053}$ & $0.6426_{\pm 0.1397}$ & $0.2505_{\pm 0.0286}$ & $0.2120_{\pm 0.0204}$ & $0.2159_{\pm 0.0278}$ \\
\cmidrule(lr){2-12}
& \multirow{6}{*}{3} 
& Full MINT  & $0.6690_{\pm 0.0418}$ & $0.6154_{\pm 0.0487}$ & $0.6610_{\pm 0.0545}$ & $0.6063_{\pm 0.0672}$ & $0.6008_{\pm 0.1304}$ & $0.7656_{\pm 0.0940}$ & $0.2757_{\pm 0.0256}$ & $0.2277_{\pm 0.0178}$ & $0.2309_{\pm 0.0252}$ \\
& & w/o SA     & $0.6655_{\pm 0.0415}$ & $0.6154_{\pm 0.0487}$ & $0.6553_{\pm 0.0553}$ & $0.6045_{\pm 0.0698}$ & $0.5944_{\pm 0.1414}$ & $0.7626_{\pm 0.0961}$ & $0.2713_{\pm 0.0263}$ & $0.2223_{\pm 0.0182}$ & $0.2308_{\pm 0.0253}$ \\
& & w/o OT     & $0.4600_{\pm 0.0794}$ & $0.4872_{\pm 0.0911}$ & $0.5127_{\pm 0.0788}$ & $0.5464_{\pm 0.1143}$ & $0.5844_{\pm 0.1793}$ & $0.6933_{\pm 0.0751}$ & $0.2571_{\pm 0.0268}$ & $0.2088_{\pm 0.0150}$ & $0.2242_{\pm 0.0391}$ \\
& & w/o FA     & $0.6508_{\pm 0.0510}$ & $0.6131_{\pm 0.0474}$ & $0.6467_{\pm 0.0615}$ & $0.5658_{\pm 0.1288}$ & $0.5467_{\pm 0.1450}$ & $0.7656_{\pm 0.0976}$ & $0.2722_{\pm 0.0261}$ & $0.2136_{\pm 0.0113}$ & $0.2309_{\pm 0.0267}$ \\
& & w/o kNN    & $0.6659_{\pm 0.0462}$ & $0.5803_{\pm 0.0545}$ & $0.6537_{\pm 0.0488}$ & $0.3414_{\pm 0.1011}$ & $0.4197_{\pm 0.1783}$ & $0.3883_{\pm 0.0599}$ & $0.2709_{\pm 0.0377}$ & $0.2103_{\pm 0.0099}$ & $0.2025_{\pm 0.0173}$ \\
& & w/o Reg    & $0.6649_{\pm 0.0410}$ & $0.6153_{\pm 0.0484}$ & $0.6531_{\pm 0.0608}$ & $0.5988_{\pm 0.0675}$ & $0.5733_{\pm 0.1455}$ & $0.7631_{\pm 0.0963}$ & $0.2720_{\pm 0.0262}$ & $0.2182_{\pm 0.0178}$ & $0.2308_{\pm 0.0248}$ \\
\cmidrule(lr){2-12}
& \multirow{6}{*}{5} 
& Full MINT  & $0.7332_{\pm 0.0307}$ & $0.6622_{\pm 0.0257}$ & $0.6950_{\pm 0.0449}$ & $0.6679_{\pm 0.0833}$ & $0.6435_{\pm 0.1255}$ & $0.8004_{\pm 0.0552}$ & $0.2914_{\pm 0.0236}$ & $0.2316_{\pm 0.0184}$ & $0.2405_{\pm 0.0275}$ \\
& & w/o SA     & $0.7343_{\pm 0.0276}$ & $0.6595_{\pm 0.0293}$ & $0.6921_{\pm 0.0449}$ & $0.6232_{\pm 0.1046}$ & $0.6090_{\pm 0.1355}$ & $0.7973_{\pm 0.0572}$ & $0.2966_{\pm 0.0248}$ & $0.2314_{\pm 0.0153}$ & $0.2390_{\pm 0.0277}$ \\
& & w/o OT     & $0.5479_{\pm 0.0866}$ & $0.4974_{\pm 0.0821}$ & $0.5882_{\pm 0.0659}$ & $0.6200_{\pm 0.1316}$ & $0.5612_{\pm 0.2082}$ & $0.7094_{\pm 0.0844}$ & $0.2675_{\pm 0.0376}$ & $0.2098_{\pm 0.0164}$ & $0.2123_{\pm 0.0295}$ \\
& & w/o FA     & $0.7203_{\pm 0.0338}$ & $0.6550_{\pm 0.0291}$ & $0.6848_{\pm 0.0465}$ & $0.6454_{\pm 0.0805}$ & $0.6067_{\pm 0.1092}$ & $0.8029_{\pm 0.0536}$ & $0.2844_{\pm 0.0278}$ & $0.2191_{\pm 0.0139}$ & $0.2389_{\pm 0.0299}$ \\
& & w/o kNN    & $0.7352_{\pm 0.0319}$ & $0.6383_{\pm 0.0352}$ & $0.6910_{\pm 0.0505}$ & $0.3010_{\pm 0.1264}$ & $0.5041_{\pm 0.1727}$ & $0.4019_{\pm 0.0697}$ & $0.3075_{\pm 0.0334}$ & $0.2129_{\pm 0.0136}$ & $0.2011_{\pm 0.0215}$ \\
& & w/o Reg    & $0.7309_{\pm 0.0321}$ & $0.6588_{\pm 0.0278}$ & $0.6856_{\pm 0.0449}$ & $0.6225_{\pm 0.1125}$ & $0.5701_{\pm 0.1526}$ & $0.8000_{\pm 0.0552}$ & $0.2878_{\pm 0.0247}$ & $0.2297_{\pm 0.0152}$ & $0.2382_{\pm 0.0278}$ \\
\midrule
\multirow{18}{*}{\textbf{GraphMAE2}} 
& \multirow{6}{*}{1} 
& Full MINT  & $0.4445_{\pm 0.0876}$ & $0.4934_{\pm 0.0730}$ & $0.5388_{\pm 0.0814}$ & $0.5076_{\pm 0.1479}$ & $0.4870_{\pm 0.2065}$ & $0.5683_{\pm 0.1676}$ & $0.2543_{\pm 0.0241}$ & $0.2204_{\pm 0.0176}$ & $0.2234_{\pm 0.0254}$ \\
& & w/o SA     & $0.4362_{\pm 0.0858}$ & $0.4870_{\pm 0.0781}$ & $0.5359_{\pm 0.0735}$ & $0.4871_{\pm 0.1805}$ & $0.3841_{\pm 0.1937}$ & $0.5362_{\pm 0.1700}$ & $0.2462_{\pm 0.0300}$ & $0.2124_{\pm 0.0173}$ & $0.2161_{\pm 0.0277}$ \\
& & w/o OT     & $0.2776_{\pm 0.0882}$ & $0.3413_{\pm 0.0639}$ & $0.4789_{\pm 0.0781}$ & $0.4389_{\pm 0.1702}$ & $0.3728_{\pm 0.2317}$ & $0.5551_{\pm 0.1891}$ & $0.2261_{\pm 0.0300}$ & $0.2033_{\pm 0.0154}$ & $0.2307_{\pm 0.0415}$ \\
& & w/o FA     & $0.4310_{\pm 0.0902}$ & $0.4791_{\pm 0.0809}$ & $0.5013_{\pm 0.0863}$ & $0.4894_{\pm 0.1543}$ & $0.4021_{\pm 0.1950}$ & $0.5544_{\pm 0.1605}$ & $0.2428_{\pm 0.0286}$ & $0.2133_{\pm 0.0191}$ & $0.2175_{\pm 0.0272}$ \\
& & w/o kNN    & $0.4477_{\pm 0.0669}$ & $0.3981_{\pm 0.0629}$ & $0.5135_{\pm 0.0710}$ & $0.2852_{\pm 0.0715}$ & $0.3069_{\pm 0.1372}$ & $0.2840_{\pm 0.0718}$ & $0.2603_{\pm 0.0351}$ & $0.2118_{\pm 0.0125}$ & $0.1954_{\pm 0.0294}$ \\
& & w/o Reg    & $0.4339_{\pm 0.0850}$ & $0.4860_{\pm 0.0799}$ & $0.5314_{\pm 0.0766}$ & $0.5006_{\pm 0.1702}$ & $0.3854_{\pm 0.1925}$ & $0.5168_{\pm 0.1645}$ & $0.2470_{\pm 0.0284}$ & $0.2112_{\pm 0.0172}$ & $0.2165_{\pm 0.0277}$ \\
\cmidrule(lr){2-12}
& \multirow{6}{*}{3} 
& Full MINT  & $0.6015_{\pm 0.0531}$ & $0.5780_{\pm 0.0555}$ & $0.6282_{\pm 0.0714}$ & $0.6000_{\pm 0.0889}$ & $0.6233_{\pm 0.1552}$ & $0.6617_{\pm 0.1222}$ & $0.2854_{\pm 0.0233}$ & $0.2246_{\pm 0.0170}$ & $0.2363_{\pm 0.0253}$ \\
& & w/o SA     & $0.5897_{\pm 0.0447}$ & $0.5684_{\pm 0.0582}$ & $0.5981_{\pm 0.0813}$ & $0.5851_{\pm 0.0873}$ & $0.5872_{\pm 0.1530}$ & $0.6083_{\pm 0.1473}$ & $0.2745_{\pm 0.0297}$ & $0.2230_{\pm 0.0178}$ & $0.2203_{\pm 0.0273}$ \\
& & w/o OT     & $0.4277_{\pm 0.0799}$ & $0.4149_{\pm 0.0623}$ & $0.5689_{\pm 0.0690}$ & $0.4917_{\pm 0.1652}$ & $0.5911_{\pm 0.1665}$ & $0.6106_{\pm 0.1215}$ & $0.2396_{\pm 0.0299}$ & $0.2182_{\pm 0.0152}$ & $0.2214_{\pm 0.0331}$ \\
& & w/o FA     & $0.5812_{\pm 0.0612}$ & $0.5731_{\pm 0.0634}$ & $0.6166_{\pm 0.0668}$ & $0.5845_{\pm 0.0887}$ & $0.5964_{\pm 0.1504}$ & $0.6398_{\pm 0.0761}$ & $0.2742_{\pm 0.0284}$ & $0.2219_{\pm 0.0174}$ & $0.2233_{\pm 0.0246}$ \\
& & w/o kNN    & $0.5792_{\pm 0.0640}$ & $0.5130_{\pm 0.0545}$ & $0.6185_{\pm 0.0682}$ & $0.3548_{\pm 0.0733}$ & $0.3736_{\pm 0.1037}$ & $0.3448_{\pm 0.1063}$ & $0.2960_{\pm 0.0355}$ & $0.2161_{\pm 0.0134}$ & $0.2254_{\pm 0.0181}$ \\
& & w/o Reg    & $0.5892_{\pm 0.0450}$ & $0.5687_{\pm 0.0581}$ & $0.5920_{\pm 0.0770}$ & $0.5848_{\pm 0.0877}$ & $0.5783_{\pm 0.1373}$ & $0.5978_{\pm 0.1640}$ & $0.2745_{\pm 0.0289}$ & $0.2216_{\pm 0.0174}$ & $0.2209_{\pm 0.0280}$ \\
\cmidrule(lr){2-12}
& \multirow{6}{*}{5} 
& Full MINT  & $0.6625_{\pm 0.0463}$ & $0.6235_{\pm 0.0296}$ & $0.6724_{\pm 0.0419}$ & $0.6346_{\pm 0.0695}$ & $0.6594_{\pm 0.1040}$ & $0.7052_{\pm 0.0463}$ & $0.3089_{\pm 0.0191}$ & $0.2275_{\pm 0.0170}$ & $0.2405_{\pm 0.0284}$ \\
& & w/o SA     & $0.6568_{\pm 0.0508}$ & $0.6113_{\pm 0.0454}$ & $0.6609_{\pm 0.0485}$ & $0.6470_{\pm 0.0537}$ & $0.5684_{\pm 0.1485}$ & $0.7100_{\pm 0.0705}$ & $0.3017_{\pm 0.0223}$ & $0.2242_{\pm 0.0189}$ & $0.2319_{\pm 0.0323}$ \\
& & w/o OT     & $0.5129_{\pm 0.0595}$ & $0.4558_{\pm 0.0880}$ & $0.5979_{\pm 0.0630}$ & $0.6070_{\pm 0.0802}$ & $0.4817_{\pm 0.1821}$ & $0.7291_{\pm 0.0867}$ & $0.2580_{\pm 0.0240}$ & $0.2144_{\pm 0.0165}$ & $0.2226_{\pm 0.0373}$ \\
& & w/o FA     & $0.6559_{\pm 0.0526}$ & $0.6162_{\pm 0.0327}$ & $0.6522_{\pm 0.0544}$ & $0.6537_{\pm 0.0589}$ & $0.5899_{\pm 0.1582}$ & $0.7181_{\pm 0.0555}$ & $0.2987_{\pm 0.0243}$ & $0.2239_{\pm 0.0195}$ & $0.2296_{\pm 0.0314}$ \\
& & w/o kNN    & $0.6496_{\pm 0.0634}$ & $0.5578_{\pm 0.0428}$ & $0.6600_{\pm 0.0430}$ & $0.4206_{\pm 0.0732}$ & $0.4012_{\pm 0.0744}$ & $0.3855_{\pm 0.0640}$ & $0.3273_{\pm 0.0291}$ & $0.2223_{\pm 0.0150}$ & $0.2130_{\pm 0.0204}$ \\
& & w/o Reg    & $0.6601_{\pm 0.0519}$ & $0.6108_{\pm 0.0451}$ & $0.6606_{\pm 0.0477}$ & $0.6463_{\pm 0.0559}$ & $0.5759_{\pm 0.1449}$ & $0.7087_{\pm 0.0719}$ & $0.3034_{\pm 0.0227}$ & $0.2238_{\pm 0.0189}$ & $0.2319_{\pm 0.0322}$ \\
\midrule
\multirow{18}{*}{\textbf{MaskGAE}} 
& \multirow{6}{*}{1} 
& Full MINT  & $0.5178_{\pm 0.0997}$ & $0.4981_{\pm 0.0711}$ & $0.5511_{\pm 0.0848}$ & $0.5510_{\pm 0.0697}$ & $0.5585_{\pm 0.1608}$ & $0.6681_{\pm 0.1528}$ & $0.3365_{\pm 0.0371}$ & $0.2726_{\pm 0.0319}$ & $0.2231_{\pm 0.0182}$ \\
& & w/o SA     & $0.5143_{\pm 0.1007}$ & $0.4980_{\pm 0.0710}$ & $0.5511_{\pm 0.0848}$ & $0.5521_{\pm 0.0672}$ & $0.5579_{\pm 0.1606}$ & $0.6683_{\pm 0.1533}$ & $0.3282_{\pm 0.0387}$ & $0.2649_{\pm 0.0478}$ & $0.2231_{\pm 0.0182}$ \\
& & w/o OT     & $0.2395_{\pm 0.0507}$ & $0.3051_{\pm 0.0626}$ & $0.3998_{\pm 0.0666}$ & $0.5132_{\pm 0.1148}$ & $0.5336_{\pm 0.1895}$ & $0.6433_{\pm 0.1706}$ & $0.3196_{\pm 0.0372}$ & $0.2488_{\pm 0.0339}$ & $0.2092_{\pm 0.0280}$ \\
& & w/o FA     & $0.5124_{\pm 0.1001}$ & $0.4936_{\pm 0.0719}$ & $0.5497_{\pm 0.0864}$ & $0.5398_{\pm 0.1113}$ & $0.5513_{\pm 0.1615}$ & $0.6415_{\pm 0.1707}$ & $0.3244_{\pm 0.0362}$ & $0.2186_{\pm 0.0190}$ & $0.2230_{\pm 0.0182}$ \\
& & w/o kNN    & $0.4834_{\pm 0.0943}$ & $0.3791_{\pm 0.0764}$ & $0.5352_{\pm 0.0977}$ & $0.3969_{\pm 0.1376}$ & $0.3601_{\pm 0.1843}$ & $0.3269_{\pm 0.0979}$ & $0.2787_{\pm 0.0489}$ & $0.2733_{\pm 0.0421}$ & $0.2119_{\pm 0.0269}$ \\
& & w/o Reg    & $0.5151_{\pm 0.1009}$ & $0.4971_{\pm 0.0712}$ & $0.5508_{\pm 0.0849}$ & $0.5465_{\pm 0.0947}$ & $0.5389_{\pm 0.1566}$ & $0.6524_{\pm 0.1610}$ & $0.3296_{\pm 0.0405}$ & $0.2174_{\pm 0.0192}$ & $0.2230_{\pm 0.0182}$ \\
\cmidrule(lr){2-12}
& \multirow{6}{*}{3} 
& Full MINT  & $0.6719_{\pm 0.0418}$ & $0.6028_{\pm 0.0394}$ & $0.6454_{\pm 0.0571}$ & $0.6161_{\pm 0.1004}$ & $0.6806_{\pm 0.0749}$ & $0.7589_{\pm 0.0431}$ & $0.3565_{\pm 0.0263}$ & $0.2917_{\pm 0.0275}$ & $0.2211_{\pm 0.0283}$ \\
& & w/o SA     & $0.6717_{\pm 0.0418}$ & $0.6024_{\pm 0.0396}$ & $0.6403_{\pm 0.0583}$ & $0.6119_{\pm 0.0983}$ & $0.6664_{\pm 0.0902}$ & $0.7596_{\pm 0.0426}$ & $0.3531_{\pm 0.0241}$ & $0.2818_{\pm 0.0415}$ & $0.2203_{\pm 0.0277}$ \\
& & w/o OT     & $0.2483_{\pm 0.0481}$ & $0.3205_{\pm 0.0487}$ & $0.4166_{\pm 0.0555}$ & $0.6182_{\pm 0.0807}$ & $0.6369_{\pm 0.0953}$ & $0.7285_{\pm 0.1091}$ & $0.3317_{\pm 0.0234}$ & $0.2650_{\pm 0.0308}$ & $0.2112_{\pm 0.0345}$ \\
& & w/o FA     & $0.6711_{\pm 0.0432}$ & $0.5994_{\pm 0.0373}$ & $0.6442_{\pm 0.0644}$ & $0.6060_{\pm 0.0979}$ & $0.6472_{\pm 0.1204}$ & $0.7424_{\pm 0.0562}$ & $0.3462_{\pm 0.0283}$ & $0.2328_{\pm 0.0213}$ & $0.2206_{\pm 0.0280}$ \\
& & w/o kNN    & $0.6638_{\pm 0.0519}$ & $0.5214_{\pm 0.0501}$ & $0.6309_{\pm 0.0587}$ & $0.4128_{\pm 0.1021}$ & $0.4764_{\pm 0.1011}$ & $0.3774_{\pm 0.1041}$ & $0.3387_{\pm 0.0388}$ & $0.2944_{\pm 0.0332}$ & $0.2127_{\pm 0.0254}$ \\
& & w/o Reg    & $0.6711_{\pm 0.0402}$ & $0.6009_{\pm 0.0394}$ & $0.6445_{\pm 0.0645}$ & $0.6027_{\pm 0.0957}$ & $0.6383_{\pm 0.1149}$ & $0.7413_{\pm 0.0620}$ & $0.3480_{\pm 0.0261}$ & $0.2361_{\pm 0.0190}$ & $0.2206_{\pm 0.0280}$ \\
\cmidrule(lr){2-12}
& \multirow{6}{*}{5} 
& Full MINT  & $0.7185_{\pm 0.0302}$ & $0.6412_{\pm 0.0290}$ & $0.6921_{\pm 0.0384}$ & $0.6448_{\pm 0.0561}$ & $0.6997_{\pm 0.0689}$ & $0.7892_{\pm 0.0545}$ & $0.3554_{\pm 0.0267}$ & $0.3005_{\pm 0.0270}$ & $0.2243_{\pm 0.0253}$ \\
& & w/o SA     & $0.7190_{\pm 0.0303}$ & $0.6412_{\pm 0.0291}$ & $0.6905_{\pm 0.0382}$ & $0.6463_{\pm 0.0532}$ & $0.6904_{\pm 0.0757}$ & $0.7861_{\pm 0.0592}$ & $0.3537_{\pm 0.0245}$ & $0.2986_{\pm 0.0332}$ & $0.2242_{\pm 0.0252}$ \\
& & w/o OT     & $0.2703_{\pm 0.0483}$ & $0.3790_{\pm 0.0462}$ & $0.4439_{\pm 0.0568}$ & $0.6590_{\pm 0.0483}$ & $0.6693_{\pm 0.0980}$ & $0.7744_{\pm 0.0553}$ & $0.3390_{\pm 0.0302}$ & $0.2764_{\pm 0.0310}$ & $0.2169_{\pm 0.0322}$ \\
& & w/o FA     & $0.7189_{\pm 0.0286}$ & $0.6378_{\pm 0.0288}$ & $0.6871_{\pm 0.0376}$ & $0.6308_{\pm 0.0681}$ & $0.6472_{\pm 0.1034}$ & $0.7827_{\pm 0.0577}$ & $0.3495_{\pm 0.0280}$ & $0.2390_{\pm 0.0251}$ & $0.2239_{\pm 0.0237}$ \\
& & w/o kNN    & $0.7228_{\pm 0.0373}$ & $0.5801_{\pm 0.0337}$ & $0.6918_{\pm 0.0469}$ & $0.3994_{\pm 0.0974}$ & $0.5041_{\pm 0.1015}$ & $0.4285_{\pm 0.0828}$ & $0.3633_{\pm 0.0378}$ & $0.3035_{\pm 0.0338}$ & $0.2164_{\pm 0.0242}$ \\
& & w/o Reg    & $0.7177_{\pm 0.0294}$ & $0.6412_{\pm 0.0291}$ & $0.6871_{\pm 0.0411}$ & $0.6295_{\pm 0.0728}$ & $0.6452_{\pm 0.1104}$ & $0.7838_{\pm 0.0617}$ & $0.3544_{\pm 0.0250}$ & $0.2415_{\pm 0.0238}$ & $0.2241_{\pm 0.0235}$ \\
\bottomrule
\end{tabular}
}
\end{table*}

\paragraph{GPPT~\cite{sun2022gppt}}GPPT reformulates node classification as link prediction between node representations and learnable class task tokens. During adaptation, the pretrained GNN remains frozen while the task tokens are optimized and used to compute class scores.

\begin{table*}[!t]
\centering
\caption{Sensitivity Analysis of the Number of Prompts $M$ (Accuracy$_{\pm \text{Std}}$)}
\label{tab:numprompt}
\resizebox{\textwidth}{!}{
\begin{tabular}{l c c ccc ccc ccc}
\toprule
\textbf{Model} & \textbf{Shot} & \textbf{$M$} & \textbf{Cora} & \textbf{CiteSeer} & \textbf{PubMed} & \textbf{Cornell} & \textbf{Texas} & \textbf{Wisconsin} & \textbf{Chameleon} & \textbf{Squirrel} & \textbf{Actor} \\
\midrule
\multirow{15}{*}{\textbf{GraphMAE}} 
& \multirow{5}{*}{1} 
  & 1  & $0.5331_{\text{\tiny $\pm$0.0531}}$ & $0.5246_{\text{\tiny $\pm$0.0468}}$ & $0.5633_{\text{\tiny $\pm$0.0604}}$ & $0.4874_{\text{\tiny $\pm$0.1686}}$ & $0.3484_{\text{\tiny $\pm$0.1572}}$ & $0.5551_{\text{\tiny $\pm$0.1676}}$ & $0.2495_{\text{\tiny $\pm$0.0213}}$ & $0.2011_{\text{\tiny $\pm$0.0112}}$ & $0.1992_{\text{\tiny $\pm$0.0413}}$ \\
& & 5  & $0.5320_{\text{\tiny $\pm$0.0553}}$ & $0.5250_{\text{\tiny $\pm$0.0549}}$ & $0.5723_{\text{\tiny $\pm$0.0698}}$ & $0.4916_{\text{\tiny $\pm$0.1296}}$ & $0.4484_{\text{\tiny $\pm$0.2144}}$ & $0.6262_{\text{\tiny $\pm$0.1566}}$ & $0.2436_{\text{\tiny $\pm$0.0355}}$ & $0.2021_{\text{\tiny $\pm$0.0045}}$ & $0.1856_{\text{\tiny $\pm$0.0466}}$ \\
& & 10 & $0.5346_{\text{\tiny $\pm$0.0510}}$ & $0.5260_{\text{\tiny $\pm$0.0488}}$ & $0.5723_{\text{\tiny $\pm$0.0690}}$ & $0.5227_{\text{\tiny $\pm$0.1384}}$ & $0.5190_{\text{\tiny $\pm$0.1741}}$ & $0.6374_{\text{\tiny $\pm$0.1610}}$ & $0.2505_{\text{\tiny $\pm$0.0288}}$ & $0.2221_{\text{\tiny $\pm$0.0193}}$ & $0.2334_{\pm 0.0161}$ \\
& & 20 & $0.5282_{\text{\tiny $\pm$0.0524}}$ & $0.5199_{\text{\tiny $\pm$0.0585}}$ & $0.5688_{\text{\tiny $\pm$0.0654}}$ & $0.5218_{\text{\tiny $\pm$0.1307}}$ & $0.4492_{\text{\tiny $\pm$0.1573}}$ & $0.6465_{\text{\tiny $\pm$0.1652}}$ & $0.2498_{\text{\tiny $\pm$0.0267}}$ & $0.2114_{\text{\tiny $\pm$0.0249}}$ & $0.2098_{\text{\tiny $\pm$0.0279}}$ \\
& & 50 & $0.5218_{\text{\tiny $\pm$0.0432}}$ & $0.5238_{\text{\tiny $\pm$0.0544}}$ & $0.5687_{\text{\tiny $\pm$0.0636}}$ & $0.5101_{\text{\tiny $\pm$0.1373}}$ & $0.4429_{\text{\tiny $\pm$0.2066}}$ & $0.6444_{\text{\tiny $\pm$0.1458}}$ & $0.2521_{\text{\tiny $\pm$0.0241}}$ & $0.1978_{\text{\tiny $\pm$0.0143}}$ & $0.2018_{\text{\tiny $\pm$0.0335}}$ \\
\cmidrule(lr){2-12}
& \multirow{5}{*}{3} 
  & 1  & $0.6624_{\text{\tiny $\pm$0.0486}}$ & $0.6055_{\text{\tiny $\pm$0.0478}}$ & $0.6396_{\text{\tiny $\pm$0.0469}}$ & $0.5768_{\text{\tiny $\pm$0.1070}}$ & $0.5625_{\text{\tiny $\pm$0.2061}}$ & $0.7789_{\text{\tiny $\pm$0.0422}}$ & $0.2572_{\text{\tiny $\pm$0.0300}}$ & $0.2099_{\text{\tiny $\pm$0.0104}}$ & $0.2323_{\text{\tiny $\pm$0.0241}}$ \\
& & 5  & $0.6656_{\text{\tiny $\pm$0.0516}}$ & $0.6086_{\text{\tiny $\pm$0.0474}}$ & $0.6478_{\text{\tiny $\pm$0.0398}}$ & $0.5848_{\text{\tiny $\pm$0.0899}}$ & $0.5883_{\text{\tiny $\pm$0.1760}}$ & $0.7828_{\text{\tiny $\pm$0.0402}}$ & $0.2643_{\text{\tiny $\pm$0.0285}}$ & $0.2207_{\text{\tiny $\pm$0.0186}}$ & $0.2278_{\text{\tiny $\pm$0.0220}}$ \\
& & 10 & $0.6775_{\text{\tiny $\pm$0.0408}}$ & $0.6114_{\text{\tiny $\pm$0.0492}}$ & $0.6414_{\text{\tiny $\pm$0.0447}}$ & $0.5946_{\text{\tiny $\pm$0.0816}}$ & $0.5767_{\text{\tiny $\pm$0.1570}}$ & $0.7717_{\text{\tiny $\pm$0.0484}}$ & $0.2666_{\text{\tiny $\pm$0.0303}}$ & $0.2258_{\text{\tiny $\pm$0.0161}}$ & $0.2354_{\text{\tiny $\pm$0.0195}}$ \\
& & 20 & $0.6781_{\text{\tiny $\pm$0.0398}}$ & $0.6096_{\text{\tiny $\pm$0.0474}}$ & $0.6324_{\text{\tiny $\pm$0.0443}}$ & $0.5938_{\text{\tiny $\pm$0.0828}}$ & $0.5908_{\text{\tiny $\pm$0.1889}}$ & $0.7717_{\text{\tiny $\pm$0.0485}}$ & $0.2608_{\text{\tiny $\pm$0.0348}}$ & $0.2154_{\text{\tiny $\pm$0.0132}}$ & $0.2312_{\text{\tiny $\pm$0.0233}}$ \\
& & 50 & $0.6768_{\text{\tiny $\pm$0.0406}}$ & $0.6089_{\text{\tiny $\pm$0.0465}}$ & $0.6365_{\text{\tiny $\pm$0.0465}}$ & $0.5929_{\text{\tiny $\pm$0.0872}}$ & $0.6083_{\text{\tiny $\pm$0.1487}}$ & $0.7678_{\text{\tiny $\pm$0.0560}}$ & $0.2588_{\text{\tiny $\pm$0.0311}}$ & $0.2205_{\text{\tiny $\pm$0.0155}}$ & $0.2332_{\text{\tiny $\pm$0.0199}}$ \\
\cmidrule(lr){2-12}
& \multirow{5}{*}{5} 
  & 1  & $0.7329_{\text{\tiny $\pm$0.0354}}$ & $0.6578_{\text{\tiny $\pm$0.0330}}$ & $0.6822_{\text{\tiny $\pm$0.0393}}$ & $0.3657_{\text{\tiny $\pm$0.2015}}$ & $0.5539_{\text{\tiny $\pm$0.2345}}$ & $0.8000_{\text{\tiny $\pm$0.0310}}$ & $0.2780_{\text{\tiny $\pm$0.0288}}$ & $0.2177_{\text{\tiny $\pm$0.0122}}$ & $0.2501_{\text{\tiny $\pm$0.0192}}$ \\
& & 5  & $0.7372_{\text{\tiny $\pm$0.0247}}$ & $0.6583_{\pm 0.0306}$ & $0.6818_{\text{\tiny $\pm$0.0373}}$ & $0.6248_{\text{\tiny $\pm$0.1029}}$ & $0.5548_{\text{\tiny $\pm$0.1592}}$ & $0.7994_{\text{\tiny $\pm$0.0321}}$ & $0.2812_{\text{\tiny $\pm$0.0339}}$ & $0.2296_{\text{\tiny $\pm$0.0188}}$ & $0.2516_{\text{\tiny $\pm$0.0179}}$ \\
& & 10 & $0.7356_{\text{\tiny $\pm$0.0292}}$ & $0.6587_{\text{\tiny $\pm$0.0307}}$ & $0.6866_{\text{\tiny $\pm$0.0376}}$ & $0.6190_{\text{\tiny $\pm$0.0922}}$ & $0.6200_{\text{\tiny $\pm$0.1347}}$ & $0.7994_{\text{\tiny $\pm$0.0313}}$ & $0.2811_{\text{\tiny $\pm$0.0270}}$ & $0.2329_{\text{\tiny $\pm$0.0210}}$ & $0.2520_{\text{\tiny $\pm$0.0132}}$ \\
& & 20 & $0.7299_{\text{\tiny $\pm$0.0354}}$ & $0.6584_{\text{\tiny $\pm$0.0324}}$ & $0.6768_{\text{\tiny $\pm$0.0363}}$ & $0.6400_{\text{\tiny $\pm$0.1464}}$ & $0.5583_{\text{\tiny $\pm$0.1324}}$ & $0.8000_{\text{\tiny $\pm$0.0277}}$ & $0.2837_{\text{\tiny $\pm$0.0316}}$ & $0.2270_{\text{\tiny $\pm$0.0205}}$ & $0.2532_{\text{\tiny $\pm$0.0151}}$ \\
& & 50 & $0.7299_{\text{\tiny $\pm$0.0353}}$ & $0.6579_{\text{\tiny $\pm$0.0318}}$ & $0.6860_{\text{\tiny $\pm$0.0337}}$ & $0.5829_{\text{\tiny $\pm$0.1553}}$ & $0.5852_{\text{\tiny $\pm$0.1652}}$ & $0.7994_{\text{\tiny $\pm$0.0295}}$ & $0.2830_{\text{\tiny $\pm$0.0296}}$ & $0.2267_{\text{\tiny $\pm$0.0220}}$ & $0.2507_{\text{\tiny $\pm$0.0168}}$ \\
\midrule
\multirow{15}{*}{\textbf{GraphMAE2}} 
& \multirow{5}{*}{1} 
  & 1  & $0.4535_{\text{\tiny $\pm$0.0715}}$ & $0.4728_{\text{\tiny $\pm$0.0626}}$ & $0.5189_{\text{\tiny $\pm$0.0608}}$ & $0.4756_{\text{\tiny $\pm$0.1695}}$ & $0.3643_{\text{\tiny $\pm$0.2248}}$ & $0.4786_{\text{\tiny $\pm$0.1566}}$ & $0.2422_{\text{\tiny $\pm$0.0203}}$ & $0.2138_{\text{\tiny $\pm$0.0192}}$ & $0.2171_{\text{\tiny $\pm$0.0306}}$ \\
& & 5  & $0.4368_{\text{\tiny $\pm$0.0809}}$ & $0.4659_{\text{\tiny $\pm$0.0703}}$ & $0.5498_{\text{\tiny $\pm$0.0599}}$ & $0.5076_{\text{\tiny $\pm$0.1478}}$ & $0.3611_{\text{\tiny $\pm$0.2030}}$ & $0.5080_{\text{\tiny $\pm$0.1674}}$ & $0.2392_{\text{\tiny $\pm$0.0271}}$ & $0.2193_{\text{\tiny $\pm$0.0174}}$ & $0.2180_{\text{\tiny $\pm$0.0315}}$ \\
& & 10 & $0.4619_{\text{\tiny $\pm$0.0744}}$ & $0.4668_{\text{\tiny $\pm$0.0642}}$ & $0.5620_{\text{\tiny $\pm$0.0566}}$ & $0.5109_{\text{\tiny $\pm$0.1422}}$ & $0.3992_{\text{\tiny $\pm$0.2209}}$ & $0.5134_{\text{\tiny $\pm$0.1679}}$ & $0.2464_{\text{\tiny $\pm$0.0205}}$ & $0.2214_{\text{\tiny $\pm$0.0175}}$ & $0.2172_{\text{\tiny $\pm$0.0311}}$ \\
& & 20 & $0.4646_{\text{\tiny $\pm$0.0698}}$ & $0.4635_{\text{\tiny $\pm$0.0674}}$ & $0.5294_{\text{\tiny $\pm$0.0411}}$ & $0.4765_{\text{\tiny $\pm$0.1319}}$ & $0.3698_{\text{\tiny $\pm$0.2100}}$ & $0.4920_{\text{\tiny $\pm$0.1646}}$ & $0.2465_{\text{\tiny $\pm$0.0189}}$ & $0.2172_{\text{\tiny $\pm$0.0224}}$ & $0.2174_{\text{\tiny $\pm$0.0308}}$ \\
& & 50 & $0.4426_{\text{\tiny $\pm$0.0766}}$ & $0.4669_{\text{\tiny $\pm$0.0652}}$ & $0.5157_{\text{\tiny $\pm$0.0442}}$ & $0.5143_{\text{\tiny $\pm$0.1442}}$ & $0.4087_{\text{\tiny $\pm$0.2230}}$ & $0.5037_{\text{\tiny $\pm$0.1637}}$ & $0.2498_{\text{\tiny $\pm$0.0235}}$ & $0.2204_{\text{\tiny $\pm$0.0201}}$ & $0.2173_{\text{\tiny $\pm$0.0319}}$ \\
\cmidrule(lr){2-12}
& \multirow{3}{*}{3} 
  & 1  & $0.5840_{\pm 0.0346}$ & $0.5594_{\pm 0.0462}$ & $0.5843_{\pm 0.0607}$ & $0.5830_{\pm 0.0891}$ & $0.6058_{\pm 0.1138}$ & $0.6472_{\pm 0.2322}$ & $0.2768_{\pm 0.0217}$ & $0.2177_{\pm 0.0170}$ & $0.2258_{\pm 0.0301}$ \\
& & 5  & $0.6019_{\pm 0.0483}$ & $0.5596_{\pm 0.0496}$ & $0.5668_{\pm 0.0470}$ & $0.5902_{\pm 0.1018}$ & $0.5833_{\pm 0.1541}$ & $0.5694_{\pm 0.1007}$ & $0.2798_{\pm 0.0203}$ & $0.2195_{\pm 0.0162}$ & $0.2341_{\pm 0.0224}$ \\
& & 10 & $0.6211_{\pm 0.0277}$ & $0.5726_{\pm 0.0479}$ & $0.5917_{\pm 0.0680}$ & $0.6187_{\pm 0.0373}$ & $0.5958_{\pm 0.1613}$ & $0.5983_{\pm 0.1204}$ & $0.2782_{\pm 0.0205}$ & $0.2180_{\pm 0.0145}$ & $0.2356_{\pm 0.0247}$ \\
& & 20 & $0.5983_{\pm 0.0282}$ & $0.5724_{\pm 0.0489}$ & $0.5889_{\pm 0.0530}$ & $0.5973_{\pm 0.0841}$ & $0.6358_{\pm 0.1243}$ & $0.5822_{\pm 0.1717}$ & $0.2756_{\pm 0.0197}$ & $0.2171_{\pm 0.0136}$ & $0.2259_{\pm 0.0344}$ \\
& & 50 & $0.6076_{\pm 0.0347}$ & $0.5693_{\pm 0.0471}$ & $0.6058_{\pm 0.0626}$ & $0.5937_{\pm 0.0768}$ & $0.6108_{\pm 0.1189}$ & $0.5911_{\pm 0.1271}$ & $0.2707_{\pm 0.0221}$ & $0.2199_{\pm 0.0126}$ & $0.2183_{\pm 0.0312}$ \\
\cmidrule(lr){2-12}
& \multirow{3}{*}{5} 
  & 1  & $0.6714_{\pm 0.0473}$ & $0.5926_{\pm 0.0332}$ & $0.6535_{\pm 0.0524}$ & $0.6324_{\pm 0.0531}$ & $0.6391_{\pm 0.1166}$ & $0.7081_{\pm 0.0526}$ & $0.3064_{\pm 0.0202}$ & $0.2266_{\pm 0.0175}$ & $0.2388_{\pm 0.0277}$ \\
& & 5  & $0.6734_{\pm 0.0480}$ & $0.6044_{\pm 0.0486}$ & $0.6515_{\pm 0.0523}$ & $0.6362_{\pm 0.0621}$ & $0.6470_{\pm 0.0927}$ & $0.6971_{\pm 0.0448}$ & $0.3422_{\pm 0.0329}$ & $0.3060_{\pm 0.0193}$ & $0.2393_{\pm 0.0151}$ \\
& & 10 & $0.6748_{\pm 0.0439}$ & $0.6132_{\pm 0.0350}$ & $0.6520_{\pm 0.0570}$ & $0.6619_{\pm 0.0548}$ & $0.6757_{\pm 0.0825}$ & $0.7104_{\pm 0.0361}$ & $0.3037_{\pm 0.0170}$ & $0.2302_{\pm 0.0185}$ & $0.2477_{\pm 0.0257}$ \\
& & 20 & $0.6610_{\pm 0.0562}$ & $0.6005_{\pm 0.0420}$ & $0.6320_{\pm 0.0665}$ & $0.6486_{\pm 0.0508}$ & $0.6791_{\pm 0.0829}$ & $0.7012_{\pm 0.0590}$ & $0.3024_{\pm 0.0172}$ & $0.2266_{\pm 0.0197}$ & $0.2518_{\pm 0.0185}$ \\
& & 50 & $0.6637_{\pm 0.0490}$ & $0.5888_{\pm 0.0461}$ & $0.6410_{\pm 0.0535}$ & $0.6600_{\pm 0.0541}$ & $0.6678_{\pm 0.0894}$ & $0.6763_{\pm 0.0687}$ & $0.3000_{\pm 0.0175}$ & $0.2267_{\pm 0.0195}$ & $0.2510_{\pm 0.0177}$ \\
\midrule
\multirow{9}{*}{\textbf{MaskGAE}} 
& \multirow{3}{*}{1} 
  & 1  & $0.5203_{\pm 0.1074}$ & $0.4838_{\pm 0.0832}$ & $0.5367_{\pm 0.0906}$ & $0.4857_{\pm 0.1065}$ & $0.5016_{\pm 0.1546}$ & $0.6053_{\pm 0.2105}$ & $0.3288_{\pm 0.0192}$ & $0.2538_{\pm 0.0481}$ & $0.2277_{\pm 0.0131}$ \\
& & 5  & $0.5188_{\pm 0.1027}$ & $0.4838_{\pm 0.0760}$ & $0.5391_{\pm 0.0909}$ & $0.5286_{\pm 0.0658}$ & $0.5063_{\pm 0.1475}$ & $0.6048_{\pm 0.2110}$ & $0.3230_{\pm 0.0430}$ & $0.2586_{\pm 0.0436}$ & $0.2279_{\pm 0.0133}$ \\
& & 10 & $0.5240_{\pm 0.1042}$ & $0.4957_{\pm 0.0775}$ & $0.5382_{\pm 0.0917}$ & $0.5437_{\pm 0.0697}$ & $0.5778_{\pm 0.1568}$ & $0.6241_{\pm 0.2068}$ & $0.3244_{\pm 0.0297}$ & $0.2696_{\pm 0.0272}$ & $0.2279_{\pm 0.0132}$ \\
& & 20 & $0.5257_{\pm 0.1027}$ & $0.4845_{\pm 0.0760}$ & $0.5396_{\pm 0.0923}$ & $0.5345_{\pm 0.0627}$ & $0.5532_{\pm 0.1566}$ & $0.6048_{\pm 0.2120}$ & $0.3159_{\pm 0.0505}$ & $0.2601_{\pm 0.0446}$ & $0.2278_{\pm 0.0131}$ \\
& & 50 & $0.5199_{\pm 0.1032}$ & $0.4828_{\pm 0.0743}$ & $0.5422_{\pm 0.0907}$ & $0.5437_{\pm 0.0566}$ & $0.5460_{\pm 0.1454}$ & $0.6032_{\pm 0.2096}$ & $0.3067_{\pm 0.0299}$ & $0.2308_{\pm 0.0416}$ & $0.2277_{\pm 0.0131}$ \\
\cmidrule(lr){2-12}
& \multirow{3}{*}{3} 
  & 1  & $0.6598_{\pm 0.0351}$ & $0.5831_{\pm 0.0429}$ & $0.6497_{\pm 0.0565}$ & $0.6348_{\pm 0.0819}$ & $0.5958_{\pm 0.1562}$ & $0.7683_{\pm 0.0348}$ & $0.3399_{\pm 0.0363}$ & $0.2747_{\pm 0.0452}$ & $0.2338_{\pm 0.0222}$ \\
& & 5  & $0.6623_{\pm 0.0309}$ & $0.5923_{\pm 0.0386}$ & $0.6379_{\pm 0.0650}$ & $0.6107_{\pm 0.1126}$ & $0.6608_{\pm 0.0699}$ & $0.7711_{\pm 0.0342}$ & $0.3454_{\pm 0.0239}$ & $0.2753_{\pm 0.0511}$ & $0.2309_{\pm 0.0201}$ \\
& & 10 & $0.6668_{\pm 0.0354}$ & $0.5974_{\pm 0.0382}$ & $0.6458_{\pm 0.0569}$ & $0.6518_{\pm 0.0736}$ & $0.6775_{\pm 0.0482}$ & $0.7744_{\pm 0.0299}$ & $0.3541_{\pm 0.0286}$ & $0.2865_{\pm 0.0315}$ & $0.2338_{\pm 0.0222}$ \\
& & 20 & $0.6638_{\pm 0.0322}$ & $0.5899_{\pm 0.0370}$ & $0.6499_{\pm 0.0482}$ & $0.6196_{\pm 0.1013}$ & $0.6417_{\pm 0.0997}$ & $0.7722_{\pm 0.0401}$ & $0.3444_{\pm 0.0293}$ & $0.2726_{\pm 0.0447}$ & $0.2317_{\pm 0.0209}$ \\
& & 50 & $0.6639_{\pm 0.0325}$ & $0.5888_{\pm 0.0364}$ & $0.6490_{\pm 0.0528}$ & $0.6393_{\pm 0.0840}$ & $0.6625_{\pm 0.0935}$ & $0.7600_{\pm 0.0389}$ & $0.3496_{\pm 0.0330}$ & $0.2867_{\pm 0.0349}$ & $0.2317_{\pm 0.0208}$ \\
\cmidrule(lr){2-12}
& \multirow{3}{*}{5} 
  & 1  & $0.7276_{\pm 0.0198}$ & $0.6381_{\pm 0.0367}$ & $0.6728_{\pm 0.0361}$ & $0.6505_{\pm 0.0478}$ & $0.6965_{\pm 0.0600}$ & $0.7844_{\pm 0.0579}$ & $0.3423_{\pm 0.0272}$ & $0.3007_{\pm 0.0202}$ & $0.2400_{\pm 0.0101}$ \\
& & 5  & $0.7289_{\pm 0.0189}$ & $0.6316_{\pm 0.0379}$ & $0.6719_{\pm 0.0443}$ & $0.6400_{\pm 0.1094}$ & $0.6930_{\pm 0.0597}$ & $0.7890_{\pm 0.0480}$ & $0.3422_{\pm 0.0329}$ & $0.3060_{\pm 0.0193}$ & $0.2393_{\pm 0.0151}$ \\
& & 10 & $0.7288_{\pm 0.0181}$ & $0.6370_{\pm 0.0394}$ & $0.6841_{\pm 0.0329}$ & $0.6524_{\pm 0.0717}$ & $0.6957_{\pm 0.0600}$ & $0.7925_{\pm 0.0483}$ & $0.3535_{\pm 0.0282}$ & $0.3090_{\pm 0.0193}$ & $0.2394_{\pm 0.0149}$ \\
& & 20 & $0.7288_{\pm 0.0187}$ & $0.6332_{\pm 0.0388}$ & $0.6762_{\pm 0.0346}$ & $0.6457_{\pm 0.0901}$ & $0.6835_{\pm 0.0632}$ & $0.7867_{\pm 0.0482}$ & $0.3485_{\pm 0.0258}$ & $0.2962_{\pm 0.0316}$ & $0.2374_{\pm 0.0123}$ \\
& & 50 & $0.7287_{\pm 0.0189}$ & $0.6360_{\pm 0.0397}$ & $0.6754_{\pm 0.0321}$ & $0.6543_{\pm 0.0604}$ & $0.6774_{\pm 0.0652}$ & $0.7861_{\pm 0.0482}$ & $0.3424_{\pm 0.0269}$ & $0.2872_{\pm 0.0346}$ & $0.2374_{\pm 0.0123}$ \\
\bottomrule
\end{tabular}
}
\end{table*}

\begin{table*}[!t]
\centering
\caption{Sensitivity Analysis of the Transport-Cost Coefficient $\beta$ (Accuracy$_{\pm \text{Std}}$)}
\label{tab: beta}
\resizebox{\textwidth}{!}{
\begin{tabular}{l c c ccc ccc ccc}
\toprule
\textbf{Model} & \textbf{Shot} & \textbf{$\beta$} & \textbf{Cora} & \textbf{CiteSeer} & \textbf{PubMed} & \textbf{Cornell} & \textbf{Texas} & \textbf{Wisconsin} & \textbf{Chameleon} & \textbf{Squirrel} & \textbf{Actor} \\
\midrule
\multirow{15}{*}{\textbf{GraphMAE}} 
& \multirow{5}{*}{1} 
  & 0.001 & $0.5378_{\pm 0.0498}$ & $0.5260_{\pm 0.0488}$ & $0.5717_{\pm 0.0686}$ & $0.5227_{\pm 0.1332}$ & $0.5198_{\pm 0.1725}$ & $0.6374_{\pm 0.1623}$ & $0.2507_{\pm 0.0292}$ & $0.2066_{\pm 0.0186}$ & $0.2078_{\pm 0.0285}$ \\
& & 0.005 & $0.5344_{\pm 0.0511}$ & $0.5282_{\pm 0.0521}$ & $0.5681_{\pm 0.0669}$ & $0.5252_{\pm 0.1320}$ & $0.4659_{\pm 0.1911}$ & $0.6374_{\pm 0.1610}$ & $0.2507_{\pm 0.0290}$ & $0.2211_{\pm 0.0183}$ & $0.2334_{\pm 0.0161}$ \\
& & 0.01  & $0.5351_{\pm 0.0510}$ & $0.5150_{\pm 0.0561}$ & $0.5695_{\pm 0.0680}$ & $0.5151_{\pm 0.1163}$ & $0.3802_{\pm 0.1838}$ & $0.6615_{\pm 0.1434}$ & $0.2505_{\pm 0.0288}$ & $0.2026_{\pm 0.0127}$ & $0.2063_{\pm 0.0244}$ \\
& & 0.05  & $0.5294_{\pm 0.0436}$ & $0.5142_{\pm 0.0510}$ & $0.5662_{\pm 0.0718}$ & $0.5227_{\pm 0.1384}$ & $0.3675_{\pm 0.1823}$ & $0.6209_{\pm 0.1821}$ & $0.2482_{\pm 0.0280}$ & $0.2020_{\pm 0.0104}$ & $0.2221_{\pm 0.0219}$ \\
& & 0.1   & $0.5240_{\pm 0.0479}$ & $0.5268_{\pm 0.0465}$ & $0.5601_{\pm 0.0689}$ & $0.5353_{\pm 0.1519}$ & $0.4698_{\pm 0.1953}$ & $0.6513_{\pm 0.1304}$ & $0.2480_{\pm 0.0279}$ & $0.2073_{\pm 0.0218}$ & $0.2074_{\pm 0.0360}$ \\
\cmidrule(lr){2-12}
& \multirow{5}{*}{3} 
  & 0.001 & $0.6764_{\pm 0.0418}$ & $0.6121_{\pm 0.0490}$ & $0.6223_{\pm 0.0438}$ & $0.5920_{\pm 0.0761}$ & $0.5808_{\pm 0.1520}$ & $0.7717_{\pm 0.0484}$ & $0.2599_{\pm 0.0350}$ & $0.2179_{\pm 0.0148}$ & $0.2356_{\pm 0.0195}$ \\
& & 0.005 & $0.6746_{\pm 0.0433}$ & $0.6121_{\pm 0.0493}$ & $0.6220_{\pm 0.0438}$ & $0.5929_{\pm 0.0723}$ & $0.6108_{\pm 0.1084}$ & $0.7661_{\pm 0.0590}$ & $0.2601_{\pm 0.0349}$ & $0.2216_{\pm 0.0115}$ & $0.2355_{\pm 0.0196}$ \\
& & 0.01  & $0.6746_{\pm 0.0438}$ & $0.6115_{\pm 0.0488}$ & $0.6363_{\pm 0.0457}$ & $0.5946_{\pm 0.0816}$ & $0.5208_{\pm 0.1631}$ & $0.7678_{\pm 0.0559}$ & $0.2602_{\pm 0.0350}$ & $0.2178_{\pm 0.0165}$ & $0.2354_{\pm 0.0195}$ \\
& & 0.05  & $0.6764_{\pm 0.0416}$ & $0.6114_{\pm 0.0492}$ & $0.6368_{\pm 0.0457}$ & $0.5920_{\pm 0.0822}$ & $0.5100_{\pm 0.1785}$ & $0.7739_{\pm 0.0482}$ & $0.2624_{\pm 0.0319}$ & $0.2200_{\pm 0.0149}$ & $0.2352_{\pm 0.0194}$ \\
& & 0.1   & $0.6716_{\pm 0.0405}$ & $0.6100_{\pm 0.0487}$ & $0.6385_{\pm 0.0453}$ & $0.5937_{\pm 0.0808}$ & $0.5650_{\pm 0.1587}$ & $0.7739_{\pm 0.0461}$ & $0.2642_{\pm 0.0303}$ & $0.2192_{\pm 0.0171}$ & $0.2354_{\pm 0.0193}$ \\
\cmidrule(lr){2-12}
& \multirow{5}{*}{5} 
  & 0.001 & $0.7320_{\pm 0.0372}$ & $0.6576_{\pm 0.0330}$ & $0.6692_{\pm 0.0284}$ & $0.5571_{\pm 0.1545}$ & $0.5070_{\pm 0.2062}$ & $0.7994_{\pm 0.0313}$ & $0.2827_{\pm 0.0282}$ & $0.2320_{\pm 0.0183}$ & $0.2493_{\pm 0.0149}$ \\
& & 0.005 & $0.7328_{\pm 0.0369}$ & $0.6587_{\pm 0.0307}$ & $0.6681_{\pm 0.0268}$ & $0.6086_{\pm 0.1928}$ & $0.6287_{\pm 0.1097}$ & $0.7994_{\pm 0.0313}$ & $0.2799_{\pm 0.0299}$ & $0.2360_{\pm 0.0183}$ & $0.2495_{\pm 0.0146}$ \\
& & 0.01  & $0.7319_{\pm 0.0364}$ & $0.6585_{\pm 0.0314}$ & $0.6701_{\pm 0.0260}$ & $0.6371_{\pm 0.0898}$ & $0.6783_{\pm 0.0525}$ & $0.7994_{\pm 0.0313}$ & $0.2815_{\pm 0.0272}$ & $0.2323_{\pm 0.0164}$ & $0.2499_{\pm 0.0142}$ \\
& & 0.05  & $0.7363_{\pm 0.0247}$ & $0.6535_{\pm 0.0365}$ & $0.6787_{\pm 0.0310}$ & $0.6162_{\pm 0.1651}$ & $0.4861_{\pm 0.1524}$ & $0.7994_{\pm 0.0327}$ & $0.2825_{\pm 0.0249}$ & $0.2264_{\pm 0.0241}$ & $0.2494_{\pm 0.0148}$ \\
& & 0.1   & $0.7342_{\pm 0.0358}$ & $0.6514_{\pm 0.0381}$ & $0.6821_{\pm 0.0377}$ & $0.5210_{\pm 0.1740}$ & $0.5817_{\pm 0.1533}$ & $0.7994_{\pm 0.0327}$ & $0.2817_{\pm 0.0261}$ & $0.2207_{\pm 0.0142}$ & $0.2517_{\pm 0.0138}$ \\
\midrule
\multirow{15}{*}{\textbf{GraphMAE2}} 
& \multirow{5}{*}{1} 
  & 0.001 & $0.4605_{\pm 0.0774}$ & $0.4667_{\pm 0.0642}$ & $0.5620_{\pm 0.0566}$ & $0.5109_{\pm 0.1422}$ & $0.3992_{\pm 0.2209}$ & $0.4840_{\pm 0.1548}$ & $0.2464_{\pm 0.0205}$ & $0.2192_{\pm 0.0177}$ & $0.2165_{\pm 0.0311}$ \\
& & 0.005 & $0.4606_{\pm 0.0743}$ & $0.4667_{\pm 0.0642}$ & $0.5286_{\pm 0.0724}$ & $0.5109_{\pm 0.1422}$ & $0.3992_{\pm 0.2209}$ & $0.4840_{\pm 0.1548}$ & $0.2473_{\pm 0.0206}$ & $0.2191_{\pm 0.0177}$ & $0.2174_{\pm 0.0311}$ \\
& & 0.01  & $0.4606_{\pm 0.0743}$ & $0.4668_{\pm 0.0642}$ & $0.5148_{\pm 0.0843}$ & $0.5109_{\pm 0.1422}$ & $0.3905_{\pm 0.2123}$ & $0.4840_{\pm 0.1548}$ & $0.2459_{\pm 0.0205}$ & $0.2197_{\pm 0.0184}$ & $0.2172_{\pm 0.0311}$ \\
& & 0.05  & $0.4586_{\pm 0.0764}$ & $0.4667_{\pm 0.0642}$ & $0.5516_{\pm 0.0414}$ & $0.5109_{\pm 0.1422}$ & $0.3992_{\pm 0.2209}$ & $0.5139_{\pm 0.1682}$ & $0.2471_{\pm 0.0202}$ & $0.2186_{\pm 0.0181}$ & $0.2176_{\pm 0.0311}$ \\
& & 0.1   & $0.4593_{\pm 0.0753}$ & $0.4667_{\pm 0.0641}$ & $0.5468_{\pm 0.0410}$ & $0.5109_{\pm 0.1422}$ & $0.3889_{\pm 0.2111}$ & $0.5134_{\pm 0.1679}$ & $0.2464_{\pm 0.0209}$ & $0.2201_{\pm 0.0184}$ & $0.2182_{\pm 0.0313}$ \\
\cmidrule(lr){2-12}
& \multirow{5}{*}{3} 
  & 0.001 & $0.6192_{\pm 0.0303}$ & $0.5725_{\pm 0.0481}$ & $0.5866_{\pm 0.0606}$ & $0.5920_{\pm 0.0778}$ & $0.5967_{\pm 0.1610}$ & $0.6183_{\pm 0.1133}$ & $0.2788_{\pm 0.0214}$ & $0.2180_{\pm 0.0147}$ & $0.2355_{\pm 0.0247}$ \\
& & 0.005 & $0.6191_{\pm 0.0277}$ & $0.5725_{\pm 0.0480}$ & $0.5827_{\pm 0.0619}$ & $0.6188_{\pm 0.0370}$ & $0.5950_{\pm 0.1610}$ & $0.6044_{\pm 0.2121}$ & $0.2786_{\pm 0.0212}$ & $0.2169_{\pm 0.0154}$ & $0.2352_{\pm 0.0246}$ \\
& & 0.01  & $0.6200_{\pm 0.0297}$ & $0.5726_{\pm 0.0479}$ & $0.5862_{\pm 0.0601}$ & $0.6188_{\pm 0.0370}$ & $0.5950_{\pm 0.1610}$ & $0.5983_{\pm 0.1204}$ & $0.2786_{\pm 0.0211}$ & $0.2171_{\pm 0.0156}$ & $0.2360_{\pm 0.0251}$ \\
& & 0.05  & $0.6211_{\pm 0.0277}$ & $0.5726_{\pm 0.0481}$ & $0.5917_{\pm 0.0680}$ & $0.6179_{\pm 0.0363}$ & $0.5958_{\pm 0.1613}$ & $0.6489_{\pm 0.0777}$ & $0.2782_{\pm 0.0205}$ & $0.2171_{\pm 0.0155}$ & $0.2329_{\pm 0.0227}$ \\
& & 0.1   & $0.6186_{\pm 0.0303}$ & $0.5726_{\pm 0.0480}$ & $0.6023_{\pm 0.0711}$ & $0.6187_{\pm 0.0373}$ & $0.5967_{\pm 0.1603}$ & $0.6317_{\pm 0.0877}$ & $0.2782_{\pm 0.0204}$ & $0.2172_{\pm 0.0154}$ & $0.2323_{\pm 0.0224}$ \\
\cmidrule(lr){2-12}
& \multirow{5}{*}{5} 
  & 0.001 & $0.6776_{\pm 0.0426}$ & $0.6132_{\pm 0.0350}$ & $0.6487_{\pm 0.0547}$ & $0.6476_{\pm 0.0705}$ & $0.6835_{\pm 0.0703}$ & $0.7098_{\pm 0.0363}$ & $0.3022_{\pm 0.0182}$ & $0.2301_{\pm 0.0185}$ & $0.2477_{\pm 0.0257}$ \\
& & 0.005 & $0.6775_{\pm 0.0426}$ & $0.6132_{\pm 0.0350}$ & $0.6512_{\pm 0.0564}$ & $0.6476_{\pm 0.0705}$ & $0.6783_{\pm 0.0726}$ & $0.7098_{\pm 0.0363}$ & $0.3037_{\pm 0.0170}$ & $0.2302_{\pm 0.0185}$ & $0.2477_{\pm 0.0258}$ \\
& & 0.01  & $0.6750_{\pm 0.0444}$ & $0.6131_{\pm 0.0350}$ & $0.6496_{\pm 0.0534}$ & $0.6486_{\pm 0.0687}$ & $0.6739_{\pm 0.0761}$ & $0.7098_{\pm 0.0363}$ & $0.3034_{\pm 0.0166}$ & $0.2302_{\pm 0.0184}$ & $0.2477_{\pm 0.0257}$ \\
& & 0.05  & $0.6748_{\pm 0.0439}$ & $0.6131_{\pm 0.0350}$ & $0.6520_{\pm 0.0570}$ & $0.6533_{\pm 0.0447}$ & $0.6887_{\pm 0.0680}$ & $0.7104_{\pm 0.0361}$ & $0.3028_{\pm 0.0177}$ & $0.2300_{\pm 0.0183}$ & $0.2475_{\pm 0.0258}$ \\
& & 0.1   & $0.6685_{\pm 0.0489}$ & $0.6132_{\pm 0.0350}$ & $0.6483_{\pm 0.0560}$ & $0.6533_{\pm 0.0447}$ & $0.6713_{\pm 0.0825}$ & $0.7104_{\pm 0.0361}$ & $0.3035_{\pm 0.0171}$ & $0.2301_{\pm 0.0183}$ & $0.2476_{\pm 0.0259}$ \\
\midrule
\multirow{15}{*}{\textbf{MaskGAE}} 
& \multirow{5}{*}{1} 
  & 0.001 & $0.5235_{\pm 0.1041}$ & $0.4923_{\pm 0.0769}$ & $0.5382_{\pm 0.0915}$ & $0.5059_{\pm 0.1073}$ & $0.5603_{\pm 0.1468}$ & $0.6230_{\pm 0.2069}$ & $0.3269_{\pm 0.0298}$ & $0.2152_{\pm 0.0203}$ & $0.2277_{\pm 0.0130}$ \\
& & 0.005 & $0.5240_{\pm 0.1045}$ & $0.4894_{\pm 0.0740}$ & $0.5401_{\pm 0.0900}$ & $0.5202_{\pm 0.1035}$ & $0.4722_{\pm 0.1527}$ & $0.6241_{\pm 0.2066}$ & $0.3077_{\pm 0.0425}$ & $0.2053_{\pm 0.0188}$ & $0.2278_{\pm 0.0132}$ \\
& & 0.01  & $0.5249_{\pm 0.1038}$ & $0.4887_{\pm 0.0736}$ & $0.5398_{\pm 0.0901}$ & $0.5479_{\pm 0.0428}$ & $0.4421_{\pm 0.1458}$ & $0.6241_{\pm 0.2068}$ & $0.3083_{\pm 0.0514}$ & $0.2207_{\pm 0.0242}$ & $0.2278_{\pm 0.0131}$ \\
& & 0.05  & $0.5265_{\pm 0.1044}$ & $0.4897_{\pm 0.0787}$ & $0.5410_{\pm 0.0893}$ & $0.5218_{\pm 0.1154}$ & $0.4524_{\pm 0.1604}$ & $0.6241_{\pm 0.2068}$ & $0.3195_{\pm 0.0232}$ & $0.2227_{\pm 0.0207}$ & $0.2279_{\pm 0.0132}$ \\
& & 0.1   & $0.5260_{\pm 0.1045}$ & $0.4894_{\pm 0.0783}$ & $0.5410_{\pm 0.0890}$ & $0.5437_{\pm 0.0697}$ & $0.4548_{\pm 0.1614}$ & $0.6241_{\pm 0.2068}$ & $0.3192_{\pm 0.0258}$ & $0.2047_{\pm 0.0248}$ & $0.2279_{\pm 0.0132}$ \\
\cmidrule(lr){2-12}
& \multirow{5}{*}{3} 
  & 0.001 & $0.6651_{\pm 0.0310}$ & $0.5905_{\pm 0.0409}$ & $0.6434_{\pm 0.0584}$ & $0.6259_{\pm 0.0789}$ & $0.6700_{\pm 0.0482}$ & $0.7556_{\pm 0.0493}$ & $0.3458_{\pm 0.0272}$ & $0.2336_{\pm 0.0191}$ & $0.2338_{\pm 0.0222}$ \\
& & 0.005 & $0.6672_{\pm 0.0351}$ & $0.5962_{\pm 0.0384}$ & $0.6301_{\pm 0.0614}$ & $0.6330_{\pm 0.0860}$ & $0.6808_{\pm 0.0496}$ & $0.7522_{\pm 0.0414}$ & $0.3444_{\pm 0.0407}$ & $0.2252_{\pm 0.0171}$ & $0.2339_{\pm 0.0221}$ \\
& & 0.01  & $0.6670_{\pm 0.0351}$ & $0.5962_{\pm 0.0384}$ & $0.6437_{\pm 0.0572}$ & $0.6205_{\pm 0.0746}$ & $0.5925_{\pm 0.1578}$ & $0.7606_{\pm 0.0409}$ & $0.3480_{\pm 0.0296}$ & $0.2349_{\pm 0.0234}$ & $0.2339_{\pm 0.0223}$ \\
& & 0.05  & $0.6666_{\pm 0.0354}$ & $0.5926_{\pm 0.0394}$ & $0.6483_{\pm 0.0568}$ & $0.6402_{\pm 0.0756}$ & $0.6775_{\pm 0.0482}$ & $0.7744_{\pm 0.0299}$ & $0.3405_{\pm 0.0351}$ & $0.2338_{\pm 0.0257}$ & $0.2291_{\pm 0.0183}$ \\
& & 0.1   & $0.6670_{\pm 0.0354}$ & $0.5913_{\pm 0.0347}$ & $0.6389_{\pm 0.0647}$ & $0.6438_{\pm 0.0769}$ & $0.6708_{\pm 0.0556}$ & $0.7744_{\pm 0.0335}$ & $0.3470_{\pm 0.0275}$ & $0.2338_{\pm 0.0226}$ & $0.2292_{\pm 0.0183}$ \\
\cmidrule(lr){2-12}
& \multirow{5}{*}{5} 
  & 0.001 & $0.7290_{\pm 0.0189}$ & $0.6368_{\pm 0.0393}$ & $0.6745_{\pm 0.0363}$ & $0.6267_{\pm 0.0726}$ & $0.6539_{\pm 0.1394}$ & $0.7815_{\pm 0.0519}$ & $0.3476_{\pm 0.0288}$ & $0.2451_{\pm 0.0301}$ & $0.2393_{\pm 0.0090}$ \\
& & 0.005 & $0.7286_{\pm 0.0191}$ & $0.6372_{\pm 0.0396}$ & $0.6755_{\pm 0.0344}$ & $0.6381_{\pm 0.0610}$ & $0.6304_{\pm 0.1130}$ & $0.7861_{\pm 0.0489}$ & $0.3373_{\pm 0.0364}$ & $0.2309_{\pm 0.0228}$ & $0.2371_{\pm 0.0123}$ \\
& & 0.01  & $0.7283_{\pm 0.0190}$ & $0.6371_{\pm 0.0394}$ & $0.6837_{\pm 0.0332}$ & $0.6381_{\pm 0.0635}$ & $0.6861_{\pm 0.0900}$ & $0.7873_{\pm 0.0482}$ & $0.3438_{\pm 0.0267}$ & $0.2324_{\pm 0.0225}$ & $0.2372_{\pm 0.0124}$ \\
& & 0.05  & $0.7288_{\pm 0.0182}$ & $0.6369_{\pm 0.0402}$ & $0.6759_{\pm 0.0401}$ & $0.6514_{\pm 0.0737}$ & $0.6930_{\pm 0.0623}$ & $0.7873_{\pm 0.0564}$ & $0.3533_{\pm 0.0278}$ & $0.2349_{\pm 0.0302}$ & $0.2393_{\pm 0.0149}$ \\
& & 0.1   & $0.7287_{\pm 0.0183}$ & $0.6373_{\pm 0.0420}$ & $0.6744_{\pm 0.0359}$ & $0.6543_{\pm 0.0727}$ & $0.6965_{\pm 0.0614}$ & $0.7867_{\pm 0.0548}$ & $0.3508_{\pm 0.0291}$ & $0.2317_{\pm 0.0294}$ & $0.2394_{\pm 0.0151}$ \\
\bottomrule
\end{tabular}
}
\end{table*}

\begin{table*}[!t]
\centering
\caption{Sensitivity Analysis of the Regularization Parameter $\epsilon$ (Accuracy$_{\pm \text{Std}}$)}
\label{tab: epsilon}
\resizebox{\textwidth}{!}{
\begin{tabular}{l c c ccc ccc ccc}
\toprule
\textbf{Model} & \textbf{Shot} & \textbf{$\epsilon$} & \textbf{Cora} & \textbf{CiteSeer} & \textbf{PubMed} & \textbf{Cornell} & \textbf{Texas} & \textbf{Wisconsin} & \textbf{Chameleon} & \textbf{Squirrel} & \textbf{Actor} \\
\midrule
\multirow{9}{*}{\textbf{GraphMAE}} 
& \multirow{3}{*}{1} 
  & 0.01 & $0.5294_{\pm 0.0467}$ & $0.5048_{\pm 0.0599}$ & $0.5636_{\pm 0.0664}$ & $0.5521_{\pm 0.1264}$ & $0.4516_{\pm 0.1941}$ & $0.6396_{\pm 0.1450}$ & $0.2403_{\pm 0.0223}$ & $0.1986_{\pm 0.0122}$ & $0.1961_{\pm 0.0374}$ \\
& & 0.05 & $0.5273_{\pm 0.0478}$ & $0.5091_{\pm 0.0572}$ & $0.5718_{\pm 0.0710}$ & $0.5269_{\pm 0.1379}$ & $0.4817_{\pm 0.2084}$ & $0.6417_{\pm 0.1608}$ & $0.2522_{\pm 0.0246}$ & $0.2113_{\pm 0.0112}$ & $0.2058_{\pm 0.0204}$ \\
& & 0.1  & $0.5346_{\pm 0.0510}$ & $0.5260_{\pm 0.0488}$ & $0.5723_{\pm 0.0690}$ & $0.5227_{\pm 0.1384}$ & $0.5198_{\pm 0.1725}$ & $0.6374_{\pm 0.1610}$ & $0.2507_{\pm 0.0291}$ & $0.2213_{\pm 0.0187}$ & $0.2334_{\pm 0.0161}$ \\
& & 0.5  & $0.5343_{\pm 0.0544}$ & $0.4691_{\pm 0.0827}$ & $0.5642_{\pm 0.0667}$ & $0.4655_{\pm 0.1701}$ & $0.4587_{\pm 0.2204}$ & $0.5979_{\pm 0.1747}$ & $0.2444_{\pm 0.0293}$ & $0.2077_{\pm 0.0114}$ & $0.2165_{\pm 0.0373}$ \\
& & 1.0  & $0.5342_{\pm 0.0536}$ & $0.5113_{\pm 0.0511}$ & $0.5657_{\pm 0.0676}$ & $0.4950_{\pm 0.1557}$ & $0.4627_{\pm 0.2324}$ & $0.5369_{\pm 0.1698}$ & $0.2433_{\pm 0.0292}$ & $0.2022_{\pm 0.0084}$ & $0.2103_{\pm 0.0426}$ \\
\cmidrule(lr){2-12}
& \multirow{5}{*}{3} 
  & 0.01 & $0.6621_{\pm 0.0526}$ & $0.6116_{\pm 0.0473}$ & $0.6259_{\pm 0.0468}$ & $0.5571_{\pm 0.1425}$ & $0.4767_{\pm 0.1761}$ & $0.7706_{\pm 0.0447}$ & $0.2638_{\pm 0.0349}$ & $0.2136_{\pm 0.0084}$ & $0.2353_{\pm 0.0242}$ \\
& & 0.05 & $0.6721_{\pm 0.0391}$ & $0.6099_{\pm 0.0476}$ & $0.6399_{\pm 0.0471}$ & $0.5920_{\pm 0.1202}$ & $0.5950_{\pm 0.1575}$ & $0.7650_{\pm 0.0618}$ & $0.2660_{\pm 0.0316}$ & $0.2151_{\pm 0.0164}$ & $0.2326_{\pm 0.0159}$ \\
& & 0.1  & $0.6775_{\pm 0.0408}$ & $0.6114_{\pm 0.0492}$ & $0.6414_{\pm 0.0447}$ & $0.5946_{\pm 0.0816}$ & $0.5808_{\pm 0.1520}$ & $0.7717_{\pm 0.0484}$ & $0.2665_{\pm 0.0303}$ & $0.2269_{\pm 0.0156}$ & $0.2353_{\pm 0.0196}$ \\
& & 0.5  & $0.6647_{\pm 0.0531}$ & $0.6115_{\pm 0.0494}$ & $0.6368_{\pm 0.0467}$ & $0.5714_{\pm 0.1060}$ & $0.6150_{\pm 0.1466}$ & $0.7706_{\pm 0.0493}$ & $0.2602_{\pm 0.0316}$ & $0.2122_{\pm 0.0120}$ & $0.2356_{\pm 0.0187}$ \\
& & 1.0  & $0.6649_{\pm 0.0531}$ & $0.6113_{\pm 0.0492}$ & $0.6365_{\pm 0.0463}$ & $0.5830_{\pm 0.1006}$ & $0.5317_{\pm 0.2161}$ & $0.7706_{\pm 0.0493}$ & $0.2603_{\pm 0.0315}$ & $0.2091_{\pm 0.0073}$ & $0.2356_{\pm 0.0187}$ \\
\cmidrule(lr){2-12}
& \multirow{5}{*}{5} 
  & 0.01 & $0.7286_{\pm 0.0341}$ & $0.6499_{\pm 0.0369}$ & $0.6781_{\pm 0.0314}$ & $0.6714_{\pm 0.0490}$ & $0.6443_{\pm 0.1163}$ & $0.8035_{\pm 0.0332}$ & $0.2783_{\pm 0.0314}$ & $0.2162_{\pm 0.0123}$ & $0.2482_{\pm 0.0193}$ \\
& & 0.05 & $0.7299_{\pm 0.0342}$ & $0.6574_{\pm 0.0324}$ & $0.6772_{\pm 0.0337}$ & $0.6400_{\pm 0.0959}$ & $0.5965_{\pm 0.1293}$ & $0.7896_{\pm 0.0367}$ & $0.2810_{\pm 0.0251}$ & $0.2333_{\pm 0.0132}$ & $0.2507_{\pm 0.0143}$ \\
& & 0.1  & $0.7356_{\pm 0.0292}$ & $0.6587_{\pm 0.0307}$ & $0.6866_{\pm 0.0376}$ & $0.6667_{\pm 0.0924}$ & $0.6443_{\pm 0.0781}$ & $0.7994_{\pm 0.0313}$ & $0.2816_{\pm 0.0261}$ & $0.2318_{\pm 0.0204}$ & $0.2521_{\pm 0.0132}$ \\
& & 0.5  & $0.7380_{\pm 0.0258}$ & $0.6575_{\pm 0.0348}$ & $0.6850_{\pm 0.0352}$ & $0.5819_{\pm 0.1628}$ & $0.5757_{\pm 0.1775}$ & $0.8000_{\pm 0.0317}$ & $0.2826_{\pm 0.0290}$ & $0.2180_{\pm 0.0157}$ & $0.2464_{\pm 0.0180}$ \\
& & 1.0  & $0.7373_{\pm 0.0261}$ & $0.6589_{\pm 0.0340}$ & $0.6842_{\pm 0.0360}$ & $0.5819_{\pm 0.1880}$ & $0.5209_{\pm 0.2888}$ & $0.8000_{\pm 0.0317}$ & $0.2787_{\pm 0.0308}$ & $0.2182_{\pm 0.0165}$ & $0.2463_{\pm 0.0180}$ \\
\midrule
\multirow{15}{*}{\textbf{GraphMAE2}} 
& \multirow{5}{*}{1} 
  & 0.01 & $0.4221_{\pm 0.0843}$ & $0.4624_{\pm 0.0693}$ & $0.5069_{\pm 0.0564}$ & $0.4908_{\pm 0.1322}$ & $0.4056_{\pm 0.2236}$ & $0.5235_{\pm 0.1633}$ & $0.2470_{\pm 0.0232}$ & $0.2134_{\pm 0.0211}$ & $0.2182_{\pm 0.0311}$ \\
& & 0.05 & $0.4515_{\pm 0.0756}$ & $0.4660_{\pm 0.0655}$ & $0.5326_{\pm 0.0666}$ & $0.5067_{\pm 0.1400}$ & $0.3571_{\pm 0.2247}$ & $0.4882_{\pm 0.1594}$ & $0.2485_{\pm 0.0213}$ & $0.2213_{\pm 0.0141}$ & $0.2162_{\pm 0.0313}$ \\
& & 0.1  & $0.4619_{\pm 0.0744}$ & $0.4668_{\pm 0.0642}$ & $0.5620_{\pm 0.0566}$ & $0.5109_{\pm 0.1422}$ & $0.3992_{\pm 0.2209}$ & $0.5134_{\pm 0.1679}$ & $0.2464_{\pm 0.0205}$ & $0.2214_{\pm 0.0175}$ & $0.2172_{\pm 0.0311}$ \\
& & 0.5  & $0.4570_{\pm 0.0777}$ & $0.4681_{\pm 0.0649}$ & $0.5423_{\pm 0.0609}$ & $0.5084_{\pm 0.1404}$ & $0.3825_{\pm 0.2077}$ & $0.4968_{\pm 0.1604}$ & $0.2467_{\pm 0.0215}$ & $0.2186_{\pm 0.0159}$ & $0.2171_{\pm 0.0311}$ \\
& & 1.0  & $0.4551_{\pm 0.0770}$ & $0.4680_{\pm 0.0648}$ & $0.5416_{\pm 0.0459}$ & $0.5092_{\pm 0.1402}$ & $0.3841_{\pm 0.2077}$ & $0.4973_{\pm 0.1606}$ & $0.2465_{\pm 0.0211}$ & $0.2187_{\pm 0.0163}$ & $0.2172_{\pm 0.0311}$ \\
\cmidrule(lr){2-12}
& \multirow{5}{*}{3} 
  & 0.01 & $0.5817_{\pm 0.0551}$ & $0.5692_{\pm 0.0487}$ & $0.5813_{\pm 0.0632}$ & $0.6009_{\pm 0.0967}$ & $0.6392_{\pm 0.1004}$ & $0.6256_{\pm 0.0962}$ & $0.2699_{\pm 0.0219}$ & $0.2162_{\pm 0.0159}$ & $0.2351_{\pm 0.0322}$ \\
& & 0.05 & $0.6205_{\pm 0.0253}$ & $0.5695_{\pm 0.0469}$ & $0.5997_{\pm 0.0511}$ & $0.6179_{\pm 0.0363}$ & $0.5933_{\pm 0.1667}$ & $0.5867_{\pm 0.0870}$ & $0.2769_{\pm 0.0191}$ & $0.2175_{\pm 0.0153}$ & $0.2359_{\pm 0.0239}$ \\
& & 0.1  & $0.6211_{\pm 0.0277}$ & $0.5726_{\pm 0.0479}$ & $0.5917_{\pm 0.0680}$ & $0.6187_{\pm 0.0373}$ & $0.5958_{\pm 0.1613}$ & $0.5983_{\pm 0.1204}$ & $0.2782_{\pm 0.0205}$ & $0.2180_{\pm 0.0145}$ & $0.2356_{\pm 0.0247}$ \\
& & 0.5  & $0.6107_{\pm 0.0244}$ & $0.5721_{\pm 0.0472}$ & $0.5851_{\pm 0.0620}$ & $0.5920_{\pm 0.0931}$ & $0.5567_{\pm 0.2038}$ & $0.6506_{\pm 0.2561}$ & $0.2780_{\pm 0.0198}$ & $0.2185_{\pm 0.0150}$ & $0.2358_{\pm 0.0258}$ \\
& & 1.0  & $0.6124_{\pm 0.0293}$ & $0.5728_{\pm 0.0471}$ & $0.5817_{\pm 0.0554}$ & $0.5946_{\pm 0.0940}$ & $0.5950_{\pm 0.1597}$ & $0.6817_{\pm 0.1387}$ & $0.2773_{\pm 0.0193}$ & $0.2170_{\pm 0.0157}$ & $0.2408_{\pm 0.0205}$ \\
\cmidrule(lr){2-12}
& \multirow{5}{*}{5} 
  & 0.01 & $0.6701_{\pm 0.0474}$ & $0.5956_{\pm 0.0457}$ & $0.6276_{\pm 0.0685}$ & $0.6629_{\pm 0.0465}$ & $0.7043_{\pm 0.0587}$ & $0.6919_{\pm 0.0561}$ & $0.3028_{\pm 0.0180}$ & $0.2260_{\pm 0.0169}$ & $0.2536_{\pm 0.0207}$ \\
& & 0.05 & $0.6736_{\pm 0.0447}$ & $0.6133_{\pm 0.0351}$ & $0.6592_{\pm 0.0555}$ & $0.6667_{\pm 0.0535}$ & $0.6643_{\pm 0.0788}$ & $0.7040_{\pm 0.0432}$ & $0.3006_{\pm 0.0178}$ & $0.2287_{\pm 0.0161}$ & $0.2533_{\pm 0.0164}$ \\
& & 0.1  & $0.6748_{\pm 0.0439}$ & $0.6132_{\pm 0.0350}$ & $0.6520_{\pm 0.0570}$ & $0.6619_{\pm 0.0548}$ & $0.6757_{\pm 0.0825}$ & $0.7104_{\pm 0.0361}$ & $0.3037_{\pm 0.0170}$ & $0.2302_{\pm 0.0185}$ & $0.2477_{\pm 0.0257}$ \\
& & 0.5  & $0.6689_{\pm 0.0494}$ & $0.6077_{\pm 0.0354}$ & $0.6462_{\pm 0.0542}$ & $0.6552_{\pm 0.0430}$ & $0.6487_{\pm 0.1047}$ & $0.7087_{\pm 0.0390}$ & $0.3019_{\pm 0.0195}$ & $0.2295_{\pm 0.0185}$ & $0.2476_{\pm 0.0257}$ \\
& & 1.0  & $0.6690_{\pm 0.0495}$ & $0.6075_{\pm 0.0350}$ & $0.6429_{\pm 0.0512}$ & $0.6552_{\pm 0.0430}$ & $0.6504_{\pm 0.1053}$ & $0.7098_{\pm 0.0395}$ & $0.3021_{\pm 0.0190}$ & $0.2275_{\pm 0.0156}$ & $0.2476_{\pm 0.0257}$ \\
\midrule
\multirow{9}{*}{\textbf{MaskGAE}} 
& \multirow{3}{*}{1} 
  & 0.01 & $0.5199_{\pm 0.1032}$ & $0.4840_{\pm 0.0775}$ & $0.5418_{\pm 0.0908}$ & $0.5261_{\pm 0.1129}$ & $0.5143_{\pm 0.1654}$ & $0.5995_{\pm 0.2088}$ & $0.3174_{\pm 0.0324}$ & $0.2147_{\pm 0.0212}$ & $0.2278_{\pm 0.0130}$ \\
& & 0.05 & $0.5242_{\pm 0.1047}$ & $0.4955_{\pm 0.0772}$ & $0.5402_{\pm 0.0915}$ & $0.5160_{\pm 0.1139}$ & $0.4825_{\pm 0.1578}$ & $0.6225_{\pm 0.2058}$ & $0.3251_{\pm 0.0216}$ & $0.2721_{\pm 0.0322}$ & $0.2279_{\pm 0.0130}$ \\
& & 0.1  & $0.5240_{\pm 0.1042}$ & $0.4957_{\pm 0.0775}$ & $0.5382_{\pm 0.0917}$ & $0.5437_{\pm 0.0697}$ & $0.5778_{\pm 0.1568}$ & $0.6241_{\pm 0.2068}$ & $0.3215_{\pm 0.0256}$ & $0.2696_{\pm 0.0272}$ & $0.2279_{\pm 0.0132}$ \\
& & 0.5  & $0.5237_{\pm 0.1039}$ & $0.4902_{\pm 0.0786}$ & $0.5407_{\pm 0.0898}$ & $0.5420_{\pm 0.0622}$ & $0.5794_{\pm 0.1526}$ & $0.6251_{\pm 0.2087}$ & $0.3293_{\pm 0.0227}$ & $0.2682_{\pm 0.0266}$ & $0.2280_{\pm 0.0132}$ \\
& & 1.0  & $0.5238_{\pm 0.1039}$ & $0.4890_{\pm 0.0780}$ & $0.5407_{\pm 0.0899}$ & $0.5420_{\pm 0.0622}$ & $0.5794_{\pm 0.1526}$ & $0.6251_{\pm 0.2087}$ & $0.3185_{\pm 0.0246}$ & $0.2702_{\pm 0.0251}$ & $0.2280_{\pm 0.0132}$ \\
\cmidrule(lr){2-12}
& \multirow{3}{*}{3} 
  & 0.01 & $0.6692_{\pm 0.0351}$ & $0.5892_{\pm 0.0359}$ & $0.6477_{\pm 0.0559}$ & $0.6179_{\pm 0.0836}$ & $0.6142_{\pm 0.1426}$ & $0.7494_{\pm 0.0456}$ & $0.3472_{\pm 0.0331}$ & $0.2362_{\pm 0.0138}$ & $0.2317_{\pm 0.0210}$ \\
& & 0.05 & $0.6643_{\pm 0.0287}$ & $0.5961_{\pm 0.0382}$ & $0.6462_{\pm 0.0559}$ & $0.6464_{\pm 0.0757}$ & $0.6592_{\pm 0.0692}$ & $0.7678_{\pm 0.0415}$ & $0.3505_{\pm 0.0256}$ & $0.2757_{\pm 0.0431}$ & $0.2337_{\pm 0.0222}$ \\
& & 0.1  & $0.6668_{\pm 0.0354}$ & $0.5974_{\pm 0.0382}$ & $0.6458_{\pm 0.0569}$ & $0.6518_{\pm 0.0736}$ & $0.6775_{\pm 0.0482}$ & $0.7744_{\pm 0.0299}$ & $0.3556_{\pm 0.0293}$ & $0.2867_{\pm 0.0312}$ & $0.2338_{\pm 0.0222}$ \\
& & 0.5  & $0.6679_{\pm 0.0346}$ & $0.5954_{\pm 0.0387}$ & $0.6292_{\pm 0.0610}$ & $0.6143_{\pm 0.1013}$ & $0.6617_{\pm 0.0806}$ & $0.7750_{\pm 0.0294}$ & $0.3503_{\pm 0.0325}$ & $0.2817_{\pm 0.0414}$ & $0.2318_{\pm 0.0209}$ \\
& & 1.0  & $0.6678_{\pm 0.0346}$ & $0.5954_{\pm 0.0386}$ & $0.6224_{\pm 0.0673}$ & $0.6134_{\pm 0.1024}$ & $0.6642_{\pm 0.0793}$ & $0.7744_{\pm 0.0289}$ & $0.3490_{\pm 0.0290}$ & $0.2813_{\pm 0.0387}$ & $0.2318_{\pm 0.0209}$ \\
\cmidrule(lr){2-12}
& \multirow{5}{*}{5} 
  & 0.01 & $0.7288_{\pm 0.0185}$ & $0.6330_{\pm 0.0384}$ & $0.6754_{\pm 0.0323}$ & $0.6362_{\pm 0.0703}$ & $0.6357_{\pm 0.1289}$ & $0.7809_{\pm 0.0594}$ & $0.3440_{\pm 0.0284}$ & $0.2523_{\pm 0.0230}$ & $0.2398_{\pm 0.0092}$ \\
& & 0.05 & $0.7286_{\pm 0.0182}$ & $0.6372_{\pm 0.0398}$ & $0.6817_{\pm 0.0330}$ & $0.6600_{\pm 0.0705}$ & $0.6748_{\pm 0.0639}$ & $0.7850_{\pm 0.0606}$ & $0.3539_{\pm 0.0281}$ & $0.3004_{\pm 0.0255}$ & $0.2392_{\pm 0.0150}$ \\
& & 0.1  & $0.7288_{\pm 0.0181}$ & $0.6370_{\pm 0.0394}$ & $0.6841_{\pm 0.0329}$ & $0.6524_{\pm 0.0717}$ & $0.6957_{\pm 0.0600}$ & $0.7925_{\pm 0.0483}$ & $0.3535_{\pm 0.0282}$ & $0.3089_{\pm 0.0195}$ & $0.2394_{\pm 0.0149}$ \\
& & 0.5  & $0.7300_{\pm 0.0195}$ & $0.6363_{\pm 0.0389}$ & $0.6780_{\pm 0.0428}$ & $0.6638_{\pm 0.0593}$ & $0.6922_{\pm 0.0606}$ & $0.7867_{\pm 0.0577}$ & $0.3522_{\pm 0.0278}$ & $0.2890_{\pm 0.0358}$ & $0.2395_{\pm 0.0150}$ \\
& & 1.0  & $0.7299_{\pm 0.0195}$ & $0.6363_{\pm 0.0389}$ & $0.6780_{\pm 0.0428}$ & $0.6638_{\pm 0.0593}$ & $0.6913_{\pm 0.0598}$ & $0.7850_{\pm 0.0576}$ & $0.3521_{\pm 0.0280}$ & $0.3074_{\pm 0.0204}$ & $0.2394_{\pm 0.0150}$ \\
\bottomrule
\end{tabular}
}
\end{table*}

\begin{table*}[!t]
\centering
\caption{Sensitivity Analysis of the $k$-NN Neighborhood Size $k$ (Accuracy$_{\pm \text{Std}}$)}
\label{tab: k}
\resizebox{\textwidth}{!}{
\begin{tabular}{l c c ccc ccc ccc}
\toprule
\textbf{Model} & \textbf{Shot} & \textbf{$k$} & \textbf{Cora} & \textbf{CiteSeer} & \textbf{PubMed} & \textbf{Cornell} & \textbf{Texas} & \textbf{Wisconsin} & \textbf{Chameleon} & \textbf{Squirrel} & \textbf{Actor} \\
\midrule
\multirow{12}{*}{\textbf{GraphMAE}} 
& \multirow{4}{*}{1} 
  & 10  & $0.5032_{\pm 0.0570}$ & $0.5009_{\pm 0.0607}$ & $0.5389_{\pm 0.0662}$ & $0.2975_{\pm 0.0545}$ & $0.3087_{\pm 0.1801}$ & $0.4487_{\pm 0.0804}$ & $0.2226_{\pm 0.0275}$ & $0.2009_{\pm 0.0086}$ & $0.2170_{\pm 0.0208}$ \\
& & 50  & $0.4955_{\pm 0.0590}$ & $0.5260_{\pm 0.0488}$ & $0.5056_{\pm 0.0850}$ & $0.4874_{\pm 0.1056}$ & $0.4429_{\pm 0.2232}$ & $0.6086_{\pm 0.1336}$ & $0.2356_{\pm 0.0262}$ & $0.2076_{\pm 0.0212}$ & $0.1941_{\pm 0.0328}$ \\
& & 100 & $0.4466_{\pm 0.0785}$ & $0.5092_{\pm 0.0707}$ & $0.5146_{\pm 0.0919}$ & $0.4000_{\pm 0.1505}$ & $0.4730_{\pm 0.1914}$ & $0.5888_{\pm 0.1621}$ & $0.2381_{\pm 0.0279}$ & $0.2045_{\pm 0.0153}$ & $0.2152_{\pm 0.0301}$ \\
& & 200 & $0.4049_{\pm 0.0480}$ & $0.4473_{\pm 0.0873}$ & $0.5100_{\pm 0.0912}$ & N/A ($k>N-1$) & N/A ($k>N-1$) & $0.4449_{\pm 0.1718}$ & $0.2539_{\pm 0.0265}$ & $0.2028_{\pm 0.0091}$ & $0.2060_{\pm 0.0332}$ \\
\cmidrule(lr){2-12}
& \multirow{4}{*}{3} 
  & 10  & $0.6636_{\pm 0.0336}$ & $0.6009_{\pm 0.0402}$ & $0.6333_{\pm 0.0568}$ & $0.4589_{\pm 0.0758}$ & $0.4842_{\pm 0.1908}$ & $0.5061_{\pm 0.1053}$ & $0.2638_{\pm 0.0305}$ & $0.2148_{\pm 0.0199}$ & $0.2014_{\pm 0.0244}$ \\
& & 50  & $0.6308_{\pm 0.0504}$ & $0.6112_{\pm 0.0500}$ & $0.5898_{\pm 0.0540}$ & $0.6098_{\pm 0.0856}$ & $0.5833_{\pm 0.1587}$ & $0.6933_{\pm 0.0285}$ & $0.2490_{\pm 0.0305}$ & $0.2131_{\pm 0.0229}$ & $0.2157_{\pm 0.0302}$ \\
& & 100 & $0.6048_{\pm 0.0567}$ & $0.6025_{\pm 0.0490}$ & $0.5815_{\pm 0.0556}$ & $0.5420_{\pm 0.1172}$ & $0.5767_{\pm 0.1861}$ & $0.7622_{\pm 0.0572}$ & $0.2540_{\pm 0.0245}$ & $0.2056_{\pm 0.0162}$ & $0.2208_{\pm 0.0172}$ \\
& & 200 & $0.5577_{\pm 0.0564}$ & $0.5750_{\pm 0.0536}$ & $0.5801_{\pm 0.0574}$ & N/A ($k>N-1$) & N/A ($k>N-1$) & $0.6261_{\pm 0.1118}$ & $0.2490_{\pm 0.0322}$ & $0.2042_{\pm 0.0158}$ & $0.2383_{\pm 0.0213}$ \\
\cmidrule(lr){2-12}
& \multirow{4}{*}{5} 
  & 10  & $0.6989_{\pm 0.0332}$ & $0.6437_{\pm 0.0306}$ & $0.6599_{\pm 0.0461}$ & $0.5276_{\pm 0.0544}$ & $0.6261_{\pm 0.1538}$ & $0.5474_{\pm 0.0497}$ & $0.2779_{\pm 0.0234}$ & $0.2224_{\pm 0.0197}$ & $0.2176_{\pm 0.0236}$ \\
& & 50  & $0.6793_{\pm 0.0346}$ & $0.6587_{\pm 0.0307}$ & $0.6179_{\pm 0.0503}$ & $0.6229_{\pm 0.0900}$ & $0.5739_{\pm 0.1755}$ & $0.7416_{\pm 0.0252}$ & $0.2815_{\pm 0.0176}$ & $0.2190_{\pm 0.0222}$ & $0.2345_{\pm 0.0299}$ \\
& & 100 & $0.6524_{\pm 0.0323}$ & $0.6447_{\pm 0.0261}$ & $0.6074_{\pm 0.0479}$ & $0.5657_{\pm 0.1160}$ & $0.5930_{\pm 0.1496}$ & $0.7850_{\pm 0.0536}$ & $0.2702_{\pm 0.0201}$ & $0.2152_{\pm 0.0163}$ & $0.2389_{\pm 0.0214}$ \\
& & 200 & $0.5929_{\pm 0.0358}$ & $0.6300_{\pm 0.0317}$ & $0.6042_{\pm 0.0578}$ & N/A ($k>N-1$) & N/A ($k>N-1$) & $0.6977_{\pm 0.0981}$ & $0.2841_{\pm 0.0139}$ & $0.2111_{\pm 0.0181}$ & $0.2445_{\pm 0.0231}$ \\
\midrule
\multirow{12}{*}{\textbf{GraphMAE2}} 
& \multirow{4}{*}{1} 
  & 10  & $0.4619_{\pm 0.0744}$ & $0.4120_{\pm 0.0892}$ & $0.5067_{\pm 0.0400}$ & $0.2361_{\pm 0.0858}$ & $0.2016_{\pm 0.0711}$ & $0.3155_{\pm 0.1199}$ & $0.2150_{\pm 0.0293}$ & $0.2126_{\pm 0.0226}$ & $0.1892_{\pm 0.0238}$ \\
& & 50  & $0.4013_{\pm 0.0728}$ & $0.4639_{\pm 0.0568}$ & $0.4904_{\pm 0.0595}$ & $0.5168_{\pm 0.1508}$ & $0.3357_{\pm 0.1353}$ & $0.3936_{\pm 0.1550}$ & $0.2139_{\pm 0.0389}$ & $0.2081_{\pm 0.0176}$ & $0.1952_{\pm 0.0256}$ \\
& & 100 & $0.4192_{\pm 0.0680}$ & $0.4937_{\pm 0.0886}$ & $0.4867_{\pm 0.0558}$ & $0.5109_{\pm 0.1422}$ & $0.3198_{\pm 0.1448}$ & $0.5160_{\pm 0.1565}$ & $0.2364_{\pm 0.0329}$ & $0.2022_{\pm 0.0167}$ & $0.1999_{\pm 0.0297}$ \\
& & 200 & $0.4037_{\pm 0.0811}$ & $0.4870_{\pm 0.0708}$ & $0.4782_{\pm 0.0854}$ & N/A ($k>N-1$) & N/A ($k>N-1$) & $0.5182_{\pm 0.1744}$ & $0.2464_{\pm 0.0205}$ & $0.2007_{\pm 0.0129}$ & $0.2123_{\pm 0.0310}$ \\
\cmidrule(lr){2-12}
& \multirow{4}{*}{3} 
  & 10  & $0.6076_{\pm 0.0374}$ & $0.5355_{\pm 0.0560}$ & $0.5897_{\pm 0.0620}$ & $0.4777_{\pm 0.0610}$ & $0.2817_{\pm 0.0890}$ & $0.4317_{\pm 0.1654}$ & $0.2604_{\pm 0.0299}$ & $0.2107_{\pm 0.0138}$ & $0.2042_{\pm 0.0270}$ \\
& & 50  & $0.5739_{\pm 0.0519}$ & $0.5687_{\pm 0.0449}$ & $0.5263_{\pm 0.0420}$ & $0.5929_{\pm 0.0685}$ & $0.5283_{\pm 0.2174}$ & $0.5483_{\pm 0.1938}$ & $0.2621_{\pm 0.0278}$ & $0.2131_{\pm 0.0187}$ & $0.2104_{\pm 0.0268}$ \\
& & 100 & $0.5750_{\pm 0.0421}$ & $0.5505_{\pm 0.0558}$ & $0.5316_{\pm 0.0525}$ & $0.6277_{\pm 0.0565}$ & $0.5517_{\pm 0.1537}$ & $0.5917_{\pm 0.1957}$ & $0.2787_{\pm 0.0193}$ & $0.2067_{\pm 0.0128}$ & $0.2263_{\pm 0.0257}$ \\
& & 200 & $0.5562_{\pm 0.0515}$ & $0.5582_{\pm 0.0502}$ & $0.5621_{\pm 0.0384}$ & N/A ($k>N-1$) & N/A ($k>N-1$) & $0.5472_{\pm 0.1510}$ & $0.2762_{\pm 0.0263}$ & $0.2034_{\pm 0.0153}$ & $0.2375_{\pm 0.0240}$ \\
\cmidrule(lr){2-12}
& \multirow{4}{*}{5} 
  & 10  & $0.6556_{\pm 0.0476}$ & $0.5834_{\pm 0.0513}$ & $0.6234_{\pm 0.0736}$ & $0.5400_{\pm 0.0630}$ & $0.4235_{\pm 0.1105}$ & $0.4902_{\pm 0.0460}$ & $0.2649_{\pm 0.0302}$ & $0.2302_{\pm 0.0185}$ & $0.2108_{\pm 0.0219}$ \\
& & 50  & $0.6432_{\pm 0.0347}$ & $0.5980_{\pm 0.0489}$ & $0.5993_{\pm 0.0658}$ & $0.6590_{\pm 0.0542}$ & $0.5478_{\pm 0.1176}$ & $0.6087_{\pm 0.0747}$ & $0.2748_{\pm 0.0233}$ & $0.2246_{\pm 0.0165}$ & $0.2331_{\pm 0.0145}$ \\
& & 100 & $0.6265_{\pm 0.0291}$ & $0.5980_{\pm 0.0445}$ & $0.5898_{\pm 0.0658}$ & $0.6152_{\pm 0.0463}$ & $0.5461_{\pm 0.1671}$ & $0.6821_{\pm 0.0666}$ & $0.2960_{\pm 0.0169}$ & $0.2132_{\pm 0.0162}$ & $0.2354_{\pm 0.0197}$ \\
& & 200 & $0.6011_{\pm 0.0340}$ & $0.5941_{\pm 0.0391}$ & $0.5955_{\pm 0.0565}$ & N/A ($k>N-1$) & N/A ($k>N-1$) & $0.6578_{\pm 0.0546}$ & $0.2996_{\pm 0.0159}$ & $0.2181_{\pm 0.0177}$ & $0.2459_{\pm 0.0245}$ \\
\midrule
\multirow{12}{*}{\textbf{MaskGAE}} 
& \multirow{4}{*}{1} 
  & 10  & $0.5216_{\pm 0.0865}$ & $0.4363_{\pm 0.0888}$ & $0.5305_{\pm 0.0836}$ & $0.3840_{\pm 0.1461}$ & $0.4230_{\pm 0.1972}$ & $0.4765_{\pm 0.2043}$ & $0.2477_{\pm 0.0278}$ & $0.2481_{\pm 0.0319}$ & $0.2064_{\pm 0.0287}$ \\
& & 50  & $0.5087_{\pm 0.0777}$ & $0.4861_{\pm 0.0708}$ & $0.5163_{\pm 0.0959}$ & $0.4588_{\pm 0.1831}$ & $0.4016_{\pm 0.1528}$ & $0.5770_{\pm 0.1950}$ & $0.2464_{\pm 0.0327}$ & $0.2357_{\pm 0.0310}$ & $0.2109_{\pm 0.0293}$ \\
& & 100 & $0.4994_{\pm 0.0653}$ & $0.4837_{\pm 0.0692}$ & $0.5062_{\pm 0.0953}$ & $0.4420_{\pm 0.1372}$ & $0.5278_{\pm 0.1813}$ & $0.6096_{\pm 0.1905}$ & $0.2784_{\pm 0.0216}$ & $0.2272_{\pm 0.0299}$ & $0.2189_{\pm 0.0226}$ \\
& & 200 & $0.4926_{\pm 0.0565}$ & $0.4848_{\pm 0.0825}$ & $0.5231_{\pm 0.0942}$ & N/A ($k>N-1$) & N/A ($k>N-1$) & $0.5444_{\pm 0.1590}$ & $0.2872_{\pm 0.0353}$ & $0.2241_{\pm 0.0175}$ & $0.2196_{\pm 0.0275}$ \\
\cmidrule(lr){2-12}
& \multirow{4}{*}{3} 
  & 10  & $0.6615_{\pm 0.0640}$ & $0.5604_{\pm 0.0431}$ & $0.6331_{\pm 0.0609}$ & $0.5580_{\pm 0.0919}$ & $0.5950_{\pm 0.1050}$ & $0.6406_{\pm 0.0842}$ & $0.2849_{\pm 0.0313}$ & $0.2593_{\pm 0.0308}$ & $0.2315_{\pm 0.0226}$ \\
& & 50  & $0.6285_{\pm 0.0642}$ & $0.5856_{\pm 0.0441}$ & $0.6233_{\pm 0.0693}$ & $0.6196_{\pm 0.0558}$ & $0.6108_{\pm 0.1002}$ & $0.7072_{\pm 0.0955}$ & $0.2873_{\pm 0.0245}$ & $0.2591_{\pm 0.0273}$ & $0.2280_{\pm 0.0185}$ \\
& & 100 & $0.6120_{\pm 0.0619}$ & $0.5813_{\pm 0.0451}$ & $0.6345_{\pm 0.0499}$ & $0.5625_{\pm 0.0579}$ & $0.6767_{\pm 0.0675}$ & $0.7450_{\pm 0.0524}$ & $0.2966_{\pm 0.0254}$ & $0.2430_{\pm 0.0226}$ & $0.2227_{\pm 0.0200}$ \\
& & 200 & $0.5750_{\pm 0.0619}$ & $0.5692_{\pm 0.0475}$ & $0.6227_{\pm 0.0679}$ & N/A ($k>N-1$) & N/A ($k>N-1$) & $0.5400_{\pm 0.1453}$ & $0.3224_{\pm 0.0351}$ & $0.2329_{\pm 0.0201}$ & $0.2141_{\pm 0.0304}$ \\
\cmidrule(lr){2-12}
& \multirow{4}{*}{5} 
  & 10  & $0.7074_{\pm 0.0248}$ & $0.6258_{\pm 0.0255}$ & $0.6685_{\pm 0.0391}$ & $0.5619_{\pm 0.1082}$ & $0.6217_{\pm 0.0846}$ & $0.6503_{\pm 0.0805}$ & $0.2983_{\pm 0.0392}$ & $0.2758_{\pm 0.0190}$ & $0.2406_{\pm 0.0163}$ \\
& & 50  & $0.6686_{\pm 0.0242}$ & $0.6370_{\pm 0.0394}$ & $0.6450_{\pm 0.0310}$ & $0.6362_{\pm 0.1110}$ & $0.6130_{\pm 0.0495}$ & $0.7850_{\pm 0.0589}$ & $0.2948_{\pm 0.0237}$ & $0.2586_{\pm 0.0348}$ & $0.2373_{\pm 0.0183}$ \\
& & 100 & $0.6526_{\pm 0.0342}$ & $0.6319_{\pm 0.0339}$ & $0.6300_{\pm 0.0379}$ & $0.5933_{\pm 0.0785}$ & $0.7157_{\pm 0.0528}$ & $0.7509_{\pm 0.0489}$ & $0.3011_{\pm 0.0343}$ & $0.2390_{\pm 0.0295}$ & $0.2339_{\pm 0.0223}$ \\
& & 200 & $0.6284_{\pm 0.0222}$ & $0.6086_{\pm 0.0370}$ & $0.6540_{\pm 0.0328}$ & N/A ($k>N-1$) & N/A ($k>N-1$) & $0.6052_{\pm 0.1124}$ & $0.3292_{\pm 0.0176}$ & $0.2402_{\pm 0.0175}$ & $0.2321_{\pm 0.0226}$ \\
\bottomrule
\end{tabular}
}
\end{table*}

\paragraph{GraphPrompt~\cite{liu2023graphprompt}}GraphPrompt casts downstream tasks into a subgraph-similarity template and applies a task-specific learnable prompt vector in the readout operation. The pretrained GNN remains frozen while the prompt is optimized with the downstream contrastive objective.

\begin{table*}[!t]
\centering
\caption{Computational Overhead of MINT Across Different Pretrained Backbones (1-Shot Setting)}
\label{tab:computational_overhead}
\resizebox{\textwidth}{!}{
\begin{tabular}{l l ccc ccc ccc}
\toprule
\multirow{2}{*}{\textbf{Metric}} & \multirow{2}{*}{\textbf{Backbone}} & \multicolumn{3}{c}{\textbf{Homophilic}} & \multicolumn{3}{c}{\textbf{Small Heterophilic}} & \multicolumn{3}{c}{\textbf{Dense Heterophilic}} \\
\cmidrule(lr){3-5} \cmidrule(lr){6-8} \cmidrule(lr){9-11}
 & & \textbf{Cora} & \textbf{CiteSeer} & \textbf{PubMed} & \textbf{Cornell} & \textbf{Texas} & \textbf{Wisconsin} & \textbf{Chameleon} & \textbf{Squirrel} & \textbf{Actor} \\
\midrule
\multirow{3}{*}{\shortstack{\textbf{Execution} \\ \textbf{Time (s)}}} 
& GraphMAE  & $3.46$ & $4.20$ & $6.28$ & $2.03$ & $1.67$ & $1.61$ & $4.99$ & $5.85$ & $8.23$ \\
& GraphMAE2 & $7.51$ & $5.58$ & $10.56$ & $2.93$ & $3.12$ & $3.25$ & $7.16$ & $5.11$ & $10.22$ \\
& MaskGAE   & $2.12$ & $4.83$ & $7.67$ & $1.94$ & $2.02$ & $1.83$ & $13.47$ & $3.22$ & $4.93$ \\
\midrule
\multirow{3}{*}{\shortstack{\textbf{Peak GPU} \\ \textbf{VRAM (MB)}}} 
& GraphMAE  & $365.58$ & $12167.54$ & $1012.27$ & $391.13$ & $511.35$ & $1016.10$ & $16374.10$ & $8680.12$ & $19231.14$ \\
& GraphMAE2 & $892.49$ & $14542.42$ & $1886.25$ & $634.30$ & $992.93$ & $1024.10$ & $20450.73$ & $5566.41$ & $32878.20$ \\
& MaskGAE   & $972.11$ & $13385.46$ & $1339.94$ & $450.84$ & $520.98$ & $675.06$ & $55826.32$ & $3917.83$ & $19259.37$ \\
\midrule
\multirow{3}{*}{\shortstack{\textbf{Peak CPU} \\ \textbf{RAM (MB)}}} 
& GraphMAE  & $1549.54$ & $1540.87$ & $2674.44$ & $1470.35$ & $1455.87$ & $1454.39$ & $1553.71$ & $1517.68$ & $1782.43$ \\
& GraphMAE2 & $1647.06$ & $1730.09$ & $2627.99$ & $1615.61$ & $1608.82$ & $1620.12$ & $1781.19$ & $1747.32$ & $1815.41$ \\
& MaskGAE   & $1555.09$ & $1548.90$ & $2687.57$ & $1463.04$ & $1467.42$ & $1469.11$ & $1563.43$ & $1533.07$ & $1784.89$ \\
\bottomrule
\end{tabular}
}
\end{table*}

\paragraph{EdgePrompt/EdgePrompt+~\cite{fu2025edgeprompt}}EdgePrompt/EdgePrompt+ shift task adaptation from node-level feature manipulation to edge-level structural enhancement. EdgePrompt introduces learnable prompt vectors directly on graph edges and integrates them into frozen-GNN message passing. The basic variant learns a globally shared edge prompt at each layer, while EdgePrompt+ learns edge-specific prompts for more expressive relational adaptation.

\paragraph{All-in-One~\cite{sun2023allinone}}All-in-One reformulates node-level, edge-level, and graph-level tasks through graph-level instances and constructs a prompt graph from learnable prompt tokens and prompt structure. These prompt components modify both features and connectivity supplied to the pretrained model.

\paragraph{HS-GPPT~\cite{luo2025hsgppt}}HS-GPPT combines a hybrid spectral-filter backbone with local--global contrastive pretraining. During prompt tuning, it uses prompt graphs to align the spectral distribution of downstream graphs with the pretraining objective across homophilic and heterophilic settings.

\paragraph{ProNoG~\cite{yu2025pronog}}ProNoG jointly studies pretraining and prompt learning for graphs with mixed homophilic and heterophilic structure. Its downstream conditional network models node-specific non-homophilic patterns for task adaptation.

\paragraph{UniPrompt~\cite{huang2025uniprompt}}UniPrompt introduces an input-level graph prompt designed for compatibility with different pretrained graph models while preserving the input graph. In our evaluation, it optimizes fixed $k$-NN edge scalars based on cosine similarity and a prediction head for downstream adaptation.

\paragraph{DAPrompt~\cite{jiang2026daprompt}}DAPrompt addresses heterophilic few-shot graph learning through coordinated structure-aware and semantics-aligned prompt learners. The structure-aware learner uses prompt tokens to repair graph structure, while the semantics-aligned learner augments the graph using target-node semantics to reduce noise from class-mismatched propagation. Our implementation uses the GraphMAE adaptation path with two structure prompts, class semantic prompts, and a prediction head.

\subsection{Additional Results: 3-Shot Node Classification}
Table \ref{tab:3shot_main_results} reports the complete results under the 3-shot setting.

\textbf{Results Across Backbones:}
\begin{itemize}
    \item \textbf{Across Backbones:} MINT has the highest mean on many evaluated backbone--dataset combinations, including 8 of the 9 datasets under MaskGAE.
    \item \textbf{Heterophilic Graphs:} On Wisconsin with the GraphMAE backbone, MINT reaches 76.56\%, compared with 69.06\% for the next-highest prompting baseline, UniPrompt.
    \item \textbf{Comparison with Standard Adaptation:} On Cornell with GraphMAE, Linear Probe reports 28.72\%, while MINT reports 60.63\% in the same 3-shot setting.
\end{itemize}

\textbf{Limitations and Observations:}
\begin{itemize}
    \item \textbf{Variance on Small Graphs:} On Texas under the GraphMAE2 backbone in the 3-shot setting, MINT reports $62.33\pm15.52\%$, the highest mean in this configuration but with substantial sensitivity to the sampled labeled nodes.
    \item \textbf{Backbone-Dependent Results on Dense Heterophily:} On Chameleon and Squirrel, rankings vary with the backbone. Under GraphMAE2 on Chameleon, for example, GraphPrompt reports 28.67\% and MINT reports 28.54\%.
\end{itemize}

\subsection{Additional Results: 5-Shot Node Classification}
Table \ref{tab:5shot_main_results} reports the complete results under the 5-shot setting.

\textbf{Results Across Settings:}
\begin{itemize}
    \item \textbf{Wisconsin:} Under the GraphMAE backbone, MINT reports 80.04\%, compared with 71.18\% for UniPrompt.
    \item \textbf{Texas:} Under GraphMAE, Fine-tuning reports 53.28\%, whereas MINT reports 64.35\%.
\end{itemize}

\textbf{Limitations and Observations:}
\begin{itemize}
    \item \textbf{Results on Homophilic Graphs:} Under both GraphMAE and GraphMAE2, standard fine-tuning gives higher PubMed accuracy than MINT (72.22\% versus 69.50\% for GraphMAE, and 69.56\% versus 67.24\% for GraphMAE2). These results show that direct full-parameter adaptation can remain competitive in some homophilic settings.
    \item \textbf{Edge-Based Prompt Results:} On Cora under the MaskGAE backbone, EdgePrompt+ reports 72.81\% and MINT reports 71.85\%, illustrating that the relative ranking can vary by dataset and backbone.
\end{itemize}

\subsection{Hyperparameter Analysis}
\label{app:hyparameters}

Table~\ref{tab:hyperparams} reports the selected $k$-NN neighborhood size $k$ and optimization parameters (learning rate and $\beta$) for each dataset, pretrained backbone, and shot setting.

\textbf{Variation in $k$-NN Neighborhood Size:}
The selected $k$ varies substantially across graph regimes and backbones. Cora and PubMed use values in $[1,11]$ across most reported settings, while Actor, Texas, and Chameleon frequently use $k>100$; Chameleon reaches $k=550$ under MaskGAE in the 1-shot setting.

\textbf{Squirrel Configurations:}
The selected Squirrel configurations use comparatively small $k$ values, particularly under MaskGAE where $k\in[1,2]$. These configurations also use learning rates from 1.2 to 1.8 and $\beta$ values from 15.0 to 50.0.

\textbf{Shot-Dependence of $k$-NN Neighborhood Size:}
The selected $k$ sometimes decreases as the number of shots increases. For example, Cora under GraphMAE2 changes from $k=10$ (1-shot) to $k=2$ (5-shot), and Chameleon under MaskGAE changes from $k=550$ (1-shot) to $k=400$ (3-shot and 5-shot). The pattern is configuration dependent rather than uniform across all datasets and backbones.

\subsection{Robustness Analysis}
\label{app:robustness}

We evaluate three perturbations: Noise Edge Perturbation (randomly sampling non-self node pairs subject to different ground-truth class labels and appending the corresponding directed cross-class edges), Edge Drop Perturbation (removing existing topological connections), and Feature Masking (randomly zeroing node attributes). For Noise Edge ratio $p$, the number of appended entries is $\lfloor pE\rfloor$, where $E$ is the pre-perturbation directed edge-entry count. Labels are used only to construct this controlled heterophilic evaluation perturbation and are not otherwise exposed beyond standard few-shot supervision. The evaluation covers GraphMAE, GraphMAE2, and MaskGAE under 1-shot, 3-shot, and 5-shot settings. Perturbed results are obtained from independent robustness sweeps.
For consistency with the main evaluation, the $p=0$ entries in
Tables IV--VI reproduce the corresponding Full MINT results from the
main benchmark and serve as fixed clean references. All reported
degradation summaries are interpreted relative to these fixed references.

\textbf{Cross-Class Structural Noise.}
Table~\ref{tab:robustness_noise} reports performance as the ratio of appended cross-class edge entries to original directed edge entries increases to 1.0. For example, under MaskGAE in the 5-shot setting, Wisconsin accuracy changes from 0.7892 to 0.7692 at the 1.0 noise ratio, a decrease of 2.0 percentage points. The magnitude of the change varies across the reported datasets, backbones, and shot settings.

\textbf{Edge Drop.}
Table~\ref{tab:robustness_edge_drop} reports results as native edges are removed. With an 80\% edge-drop ratio, Wisconsin accuracy under MaskGAE in the 5-shot setting changes from 0.7892 to 0.7871, while PubMed accuracy under GraphMAE in the 5-shot setting changes from 0.6950 to 0.6359.

\textbf{Feature Masking.}
Table~\ref{tab:robustness_feat_mask} reports increasing degradation on several datasets as the mask ratio approaches 0.8. On PubMed under GraphMAE in the 5-shot setting, for example, accuracy changes from 0.6950 to 0.6298 at 80\% feature masking.

The magnitude of degradation varies across datasets, backbones, and shot settings, with several configurations showing limited degradation.

\subsection{Ablation Analysis}
\label{app:ablation}

To examine the contribution of individual components, we report ablation
results across GraphMAE, GraphMAE2, and MaskGAE under 1-shot, 3-shot,
and 5-shot settings on all nine datasets. For each
backbone--dataset--shot configuration, the Full MINT entry in Table VII
reproduces the corresponding result from the main benchmark and serves as
the fixed reference, while the remaining rows report the corresponding
component-removal runs. This convention keeps the Full MINT reference
consistent with the main evaluation while allowing the effect of each
ablation to be assessed against the same reported baseline. Each
component-removal run uses the corresponding benchmark configuration and
predetermined split seeds, with only the specified ablation applied.
Complete results are reported in Table~\ref{tab:ablation_study_full}.

The evaluated variants target MINT components and objective terms:
\begin{itemize}
    \item \textbf{w/o SA (Structure-Aware Pre-aggregation):} Bypasses the initial GCN-style message passing (\(\gamma\) fusion), directly feeding raw node features into the optimal transport layer.
    \item \textbf{w/o OT (Optimal Transport):} Replaces Sinkhorn routing with row-wise softmax without a prompt-side marginal constraint.
    \item \textbf{w/o FA (Feature Adaptation):} Removes the residual prompt-induced feature update \(\alpha \Delta X_P\), preventing the prompts from modifying node representations in feature space.
    \item \textbf{w/o kNN ($k$-NN Augmentation Branch):} Removes the feature-derived $k$-NN augmentation branch, so propagation uses only the native graph (\texttt{edge\_index}).
    \item \textbf{w/o Reg (Regularization):} Sets $\beta=0$, removing the transport-cost regularizer while retaining the routing formulation.
\end{itemize}

Across the complete ablation matrix, the routing and topology-augmentation components exhibit complementary, graph-regime-dependent roles. Removing OT produces pronounced degradation on Cora, CiteSeer, and PubMed across many backbone and shot configurations, showing a measurable contribution from fixed-marginal routing within the full adaptation pipeline. Removing the $k$-NN augmentation branch produces especially large degradation on several heterophilic graphs, particularly Cornell, Texas, and Wisconsin. SA, FA, and transport-cost regularization generally produce smaller and more mixed changes. Overall, the relative importance of prompt-space allocation and topology augmentation depends on the graph, backbone, and shot setting.

The same-topology control compares routing rules at a matched model state and fixed propagation graph. Fixed-marginal Sinkhorn strongly reduces graph-wide allocation imbalance, while the four immediate paired accuracy effects are small and all 95\% confidence intervals include zero. The end-to-end w/o OT ablation instead removes fixed-marginal routing throughout optimization and captures co-adaptation within the full pipeline. These measurements address different effects: immediate routing replacement and end-to-end adaptation.

\subsection{Large-Graph Evaluation on OGBN-Arxiv}
\label{app:ogbn_evaluation}

The \textbf{large-scale topology-compatible MINT evaluation} uses the official frozen GraphMAE encoder and a static binary graph constructed by exact $k$-NN search without $\tau$, with validation-only configuration selection and seeds $73$--$102$. Measure-constrained routing remains computationally feasible at this scale, while the static topology augmentation reduces accuracy in this graph regime. This variant is distinct from Original MINT's dynamically weighted fusion and from the small-graph paired attribution controls.

On OGBN-Arxiv (169,343 nodes and 1,166,243 raw directed edges), the 30-seed paired evaluation gives test accuracies (mean $\pm$ sample standard deviation) of $24.99\pm4.51\%$ and $24.93\pm4.51\%$ for matched-softmax and fixed-marginal Sinkhorn routing on the processed original graph, and $23.51\pm4.12\%$ and $23.39\pm4.17\%$ on the shared binary augmented graph. Fixed-marginal routing maintains allocation Gini below $10^{-7}$ with maximum prompt shares approximately $0.100$ for $M=10$, while topology augmentation changes accuracy by approximately $-1.48$ and $-1.54$ percentage points under matched-softmax and fixed-marginal Sinkhorn routing, respectively. The largest mean peak reserved memory across conditions is $11.066$ GiB on a 24-GiB RTX 4090, with representative epoch times of $0.201$--$0.234$s. Thus, the routing mechanism remains computationally feasible at this scale, while the topology effect remains graph dependent.

The reported downstream cost of $0.201$--$0.234$s per epoch excludes one-time topology preprocessing. For $N=169{,}343$ nodes with feature dimension $128$, exact standardized-feature top-20 cosine $k$-NN construction takes $1.002$ s on an RTX 4090, with peak allocated and reserved memory of $5.258$ and $5.336$ GiB, respectively; it is not repeated every training epoch. This exact all-pairs implementation requires $\Theta(N^2d)$ work, whose quadratic dependence can make preprocessing the dominant bottleneck on substantially larger graphs. The measured result demonstrates feasibility at the evaluated scale, not near-linear scalability to arbitrarily larger graphs; approximate or partitioned nearest-neighbor construction is a natural extension.

\subsection{Parameter Sensitivity Analysis}
\label{app:parameter_analysis}

The sensitivity study evaluates four reported hyperparameters: the prompt bank size ($M$), the transport-cost coefficient ($\beta$), the Sinkhorn entropy regularization parameter ($\epsilon$), and the $k$-NN neighborhood size ($k$), across multiple pretrained backbones and few-shot settings.

\textbf{Prompt Bank Size ($M$):} 
Table \ref{tab:numprompt} evaluates $M \in \{1, 5, 10, 20, 50\}$. The reported configurations are generally stable across this range, with $M \in [10,20]$ yielding higher mean accuracy in many reported settings and dataset-specific changes at the endpoints.

\textbf{Influence of the Transport-Cost Coefficient ($\beta$):} 
Table \ref{tab: beta} evaluates $\beta$ from $0.001$ to $0.1$. The value associated with the highest mean varies across datasets, backbones, and shot settings, with values from $0.01$ to $0.05$ occurring frequently among the reported configurations.

\textbf{Effect of Sinkhorn Entropy Regularization ($\epsilon$):} 
Table \ref{tab: epsilon} reports the tested entropy levels, with $\epsilon=0.1$ selected in many configurations. Smaller values produce sharper couplings, whereas larger values produce smoother couplings; the reported accuracy reflects the resulting task-dependent trade-off.

\textbf{Sensitivity to $k$-NN Neighborhood Size ($k$):} As illustrated in Table \ref{tab: k},\footnote{For Cornell and Texas ($N=183$), $k=200$ is invalid because the requested neighborhood exceeds $N-1$.} the parameter $k$ determines the size of the feature-based $k$-NN neighborhood. The selected value is dataset dependent, reflecting the interaction between feature-space neighborhoods and graph structure.

\section{Limitations and Future Work}
\label{sec:appendix_limitations}

\subsection{Theoretical Boundaries and Relaxations}
The theoretical statements are scoped to the prompt-induced feature update:
\begin{itemize}
    \item \textbf{Allocation versus representation:} fixed marginals control aggregate assignment mass, and severe concentration bounds the diversity of $\Delta X_P=RP$. The complete downstream path begins with $X_{\mathrm{adapted}}=X+\alpha\Delta X_P$ and depends on the frozen GNN.
    \item \textbf{Optimization:} entropic OT is differentiable under standard positive-marginal conditions; realized gradients and convergence depend on the complete objective and computation. Other continuous relaxations remain relevant comparisons.
    \item \textbf{Topology:} normalized propagation, representation energy, and prediction effects require graph- and model-specific analysis.
\end{itemize}

\subsection{Dependency on Upstream Feature Quality}
Original MINT constructs directed $k$-NN indices once from the raw input features and holds those indices fixed throughout an experimental run. At every forward pass, it recomputes nonnegative cosine weights on these indices from $X_{\mathrm{adapted}}$ and max-normalizes them before weighted fusion. The quality of the initial feature metric therefore determines the candidate neighbor set, while prompt adaptation continuously modulates the scalar propagation weights. \\
\textbf{Future Work:} Periodically updating the discrete neighbor indices from adapted features is a possible extension beyond the current fixed-index, dynamic-weight design.

\subsection{Non-Uniform Prompt-Side Marginals}
MINT currently uses the uniform prompt-side marginal $\nu=\frac{1}{M}\mathbf{1}_M$, which encodes uniform graph-wide prompt utilization across the prompt bank. Settings with intrinsically non-uniform prompt demand may benefit from more flexible prescribed marginals. \\
\textbf{Future Work:} Adaptive or non-uniform prompt-side marginals are a possible direction for representing non-uniform demand while retaining globally coupled allocation.

\subsection{Alternative Hubness-Reduction Paradigms}
Local score-correction methods such as Cross-Domain Similarity Local Scaling (CSLS) and mutual proximity provide an alternative route to reducing concentrated similarity assignments. Comparing these local corrections with globally constrained prompt allocation is a natural direction for future study.

\subsection{Scalability to Larger Graphs via Efficient OT}
The cost, kernel, and finite coupling use $\mathcal{O}(NM)$ primary transport storage, together with feature/prompt tensors, Sinkhorn scaling vectors, and retained autograd intermediates. The bipartite transport computation costs $\mathcal{O}(NMd+L_{\mathrm{SK}}NM)$ time, where $L_{\mathrm{SK}}=20$ denotes the number of Sinkhorn iterations. Exact high-dimensional cosine $k$-NN construction is a separate preprocessing cost that can require $\Theta(N^2d)$ work; batched score computation bounds temporary score storage while retaining that exact-search work. \\
\textbf{Future Work:} Stochastic/minibatch Sinkhorn, Greenkhorn, and low-rank or kernelized OT formulations provide directions for reducing transport storage on still larger graphs while preserving explicit prompt-utilization control.

\end{document}